\documentclass{article}

\usepackage{microtype}
\usepackage{graphicx}
\usepackage{subfigure}
\usepackage{booktabs} 

\usepackage{hyperref}

\usepackage[accepted]{icml2019}

\usepackage{amsmath}
\usepackage{amsthm}
\usepackage{amssymb}
\usepackage[utf8]{inputenc}
\usepackage{enumitem}
\usepackage[export]{adjustbox}
\allowdisplaybreaks
\renewcommand{\algorithmicrequire}{\textbf{Input:}}
\renewcommand{\algorithmicensure}{\textbf{Output:}}

\usepackage{forloop}
\newcounter{ct}

\newcommand{\doubleQuote}[1]{\lq\lq{#1}\rq\rq}

\newcommand{\pih}{\widehat{\pi}}

\newcommand{\lah}{\widehat{\lambda}}

\newcommand{\Gh}{\widehat{G}}
\newcommand{\Ch}{\widehat{C}}
\newcommand{\Lh}{\widehat{L}}
\newcommand{\Lmax}{\mathrm{L}_{\mathrm{max}}}
\newcommand{\Lmin}{\mathrm{L}_{\mathrm{min}}}
\newcommand{\Lhmax}{\widehat{\mathrm{L}}_{\mathrm{max}}}
\newcommand{\Lhmin}{\widehat{\mathrm{L}}_{\mathrm{min}}}

\newcommand{\X}{\mathrm{X}}
\newcommand{\A}{\mathrm{A}}
\newcommand{\D}{\mathrm{D}}
\newcommand{\E}{\mathbb{E}}
\newcommand{\R}{\mathbb{R}}
\newcommand{\F}{\mathrm{F}}

\newcommand{\T}{\mathbb{T}}
\newcommand{\Tpi}{\mathbb{T}^\pi}

\newcommand{\pdim}{\textnormal{\texttt{dim}}}

\usepackage{color}

\DeclareMathOperator*{\argmin}{arg\,min}
\DeclareMathOperator*{\argmax}{arg\,max}
\DeclareMathOperator*{\arginf}{arg\,inf}

\newcommand{\abs}[1]{\lvert#1\rvert}
\newcommand{\norm}[1]{\left\|#1\right\|}

\newcommand*\xbar[1]{%
   \hbox{%
     \vbox{%
       \hrule height 0.5pt 
       \kern0.5ex
       \hbox{%
         \kern-0.1em
         \ensuremath{#1}%
         \kern-0.1em
       }%
     }%
   }%
} 
\graphicspath{{figures/}}

\theoremstyle{plain}
\newtheorem{thm}{Theorem}[section]
\newtheorem{lem}[thm]{Lemma}
\newtheorem{prop}[thm]{Proposition}

\newtheorem{assumption}{Assumption}

\newtheorem{rem}[thm]{Remark}

\theoremstyle{definition}
\newtheorem{defn}{Definition}[section]

\newtheorem{exmp}{Example}[section]

\theoremstyle{remark}

\makeatletter
\newcommand*\rel@kern[1]{\kern#1\dimexpr\macc@kerna}
\newcommand*\widebar[1]{%
  \begingroup
  \def\mathaccent##1##2{%
    \rel@kern{0.8}%
    \overline{\rel@kern{-0.8}\macc@nucleus\rel@kern{0.2}}%
    \rel@kern{-0.2}%
  }%
  \macc@depth\@ne
  \let\math@bgroup\@empty \let\math@egroup\macc@set@skewchar
  \mathsurround\z@ \frozen@everymath{\mathgroup\macc@group\relax}%
  \macc@set@skewchar\relax
  \let\mathaccentV\macc@nested@a
  \macc@nested@a\relax111{#1}%
  \endgroup
}
\makeatother

\makeatletter
\g@addto@macro\normalsize{%
  \setlength\abovedisplayskip{3pt}
  \setlength\belowdisplayskip{3pt}
  \setlength\abovedisplayshortskip{3pt}
  \setlength\belowdisplayshortskip{3pt}
}
\makeatother

\usepackage{etoolbox}
\makeatletter
\patchcmd\@combinedblfloats{\box\@outputbox}{\unvbox\@outputbox}{}{%
  \errmessage{\noexpand\@combinedblfloats could not be patched}%
}%
\makeatother

\icmltitlerunning{Batch Policy Learning under Constraints}

\begin{document}

\twocolumn[
\icmltitle{Batch Policy Learning under Constraints}




\begin{icmlauthorlist}
\icmlauthor{Hoang M. Le}{caltech}
\icmlauthor{Cameron Voloshin}{caltech}
\icmlauthor{Yisong Yue}{caltech}

\end{icmlauthorlist}

\icmlaffiliation{caltech}{California Institute of Technology, Pasadena, CA}

\icmlcorrespondingauthor{Hoang M. Le}{hmle@caltech.edu}

\icmlkeywords{Batch Learning, Reinforcement Learning, Learning with Constraints, Machine Learning Reductions}

\vskip 0.3in
]



\printAffiliationsAndNotice{}  

\begin{abstract}
When learning policies for real-world domains, two important questions arise: (i) how to efficiently use pre-collected off-policy, non-optimal behavior data; and (ii) how to mediate among different competing objectives and constraints. We thus study the problem of batch policy learning under multiple constraints, and offer a systematic solution. We first propose a flexible meta-algorithm that admits any batch reinforcement learning and online learning procedure as subroutines. We then present a specific algorithmic instantiation and provide performance guarantees for the main objective and all constraints. To certify constraint satisfaction, we propose a new and simple method for off-policy policy evaluation (OPE) and derive PAC-style bounds. Our algorithm achieves strong empirical results in different domains, including in a challenging problem of simulated car driving subject to multiple constraints such as lane keeping and smooth driving. We also show experimentally that our OPE method outperforms other popular OPE techniques on a standalone basis, especially in a high-dimensional setting.
\end{abstract}

\section{Introduction}
We study the problem of policy learning under multiple constraints. Contemporary approaches to learning sequential decision making policies have largely focused on optimizing some cost objective that is easily reducible to a scalar value function. However, in many real-world domains, choosing the right cost function to optimize is often not a straightforward task. Frequently, the agent designer faces multiple competing objectives. For instance, consider the aspirational task of designing autonomous vehicle controllers: one may care about minimizing the travel time while making sure the driving behavior is safe, comfortable, or fuel efficient.  Indeed, many such real-world applications require the primary objective function be augmented with an appropriate set of constraints \cite{altman1999constrained}. 

Contemporary policy learning research has largely focused on either online reinforcement learning (RL) with a focus on exploration, or imitation learning (IL) with a focus on learning from expert demonstrations.  However, many real-world settings already contain large amounts of pre-collected data generated by existing policies (e.g., existing driving behavior, power grid control policies, etc.).
We thus study the complementary question: 
\textit{can we leverage this abundant source of (non-optimal) behavior data in order to learn sequential decision making policies with provable guarantees on constraint satisfaction}? 
We thus propose and study the problem of batch policy learning under multiple constraints.  Historically, batch RL is regarded as a subfield of approximate dynamic programming (ADP) \cite{lange2012batch}, where a set of transitions sampled from the existing system is given and fixed.  
From an interaction perspective, one can view many online RL methods (e.g., DDPG \cite{lillicrap2016continuous}) as running a growing batch RL subroutine per round of online RL.
In that sense, batch policy learning is complementary to any exploration scheme. 
To the best of our knowledge, the study of constrained policy learning in the batch setting is novel.

We present an algorithmic framework for learning sequential decision making policies from off-policy data. We employ multiple learning reductions to online and supervised learning, and present an analysis that relates performance in the reduced procedures to the overall performance with respect to both the primary objective and constraint satisfaction. 

Constrained optimization is a well studied problem in supervised machine learning and optimization. In fact, our approach is inspired by the work of \citet{agarwal2018reductions} in the context of fair classification. In contrast to supervised learning for classification, batch policy learning for sequential decision making introduces multiple additional challenges. First, setting aside the constraints, batch policy learning itself presents a layer of difficulty, and the analysis is significantly more complicated. 
Second, verifying whether the constraints are satisfied is no longer as straightforward as passing the training data through the learned classifier. In sequential decision making, certifying constraint satisfaction amounts to an off-policy policy evaluation problem, which is a challenging problem and the subject of active research. In this paper, we develop a systematic approach to address these challenges, provide a careful error analysis, and experimentally validate our proposed algorithms. In summary, our contributions are:
\begin{itemize}[topsep=0pt]
\setlength\itemsep{0em}
    \item We formulate the problem of batch policy learning under multiple constraints, and present the first approach of its kind to solve this problem. The definition of constraints is general and can subsume many objectives. Our meta-algorithm utilizes multi-level learning reductions, and we show how to instantiate it using various batch RL and online learning subroutines. We show that guarantees from the subroutines provably lift to provide end-to-end guarantees on the original constrained batch policy learning problem.
    \item Leveraging techniques from batch RL as a subroutine, we provide a refined theoretical analysis for general non-linear function approximation that improves upon the previously known sample complexity bound \cite{munos2008finite} from $O(n^4)$ to $O(n^2)$.
    \item To evaluate and verify constraint satisfaction, we propose a simple new technique for off-policy policy evaluation, which is used as a subroutine in our main algorithm. We show that it is competitive to other off-policy policy evaluation methods.
    \item We validate our algorithm and analysis with two experimental settings. First, a simple navigation domain where we consider safety constraint. Second, we consider a high-dimensional racing car domain with smooth driving and lane keeping constraints.
\end{itemize}
\section{Problem Formulation}
\label{sec:problem}
We first introduce notation.
Let $\X \subset \mathbb{R}^d$ be a bounded and closed $d$-dimensional state space. Let $\A$ be a finite action space.
Let $c:\X\times\A\mapsto [0,\widebar{C}]$ be the primary objective cost function that is bounded by $\widebar{C}$.
Let there be $m$ constraint cost functions,  $g_i:\X\times\A\mapsto[0,\widebar{G}]$, each bounded by $\widebar{G}$.
To simplify the notation, we view the set of constraints as a vector function $g:\X\times\A\mapsto[0,\widebar{G}]^m$ where $g(x,a)$ is the column vector of individual $g_i(x,a)$. 
Let $p(\cdot|x,a)$ denote the (unknown) transition/dynamics model that maps state/action pairs to a distribution over the next state.
Let $\gamma\in(0,1)$ denote the discount factor. Let $\chi$ be the initial states distribution.

We consider the discounted infinite horizon setting. 
An MDP is defined using the tuple $\left( \X,\A,c,g,p,\gamma, \chi \right)$. A policy $\pi\in\Pi$ maps states to actions, i.e., $\pi(x)\in\A$. The value function $C^\pi:\X\mapsto\R$ corresponding to the primary cost function $c$ is defined in the usual way: $C^\pi(x) = \E\left[ \sum_{t=0}^\infty \gamma^t c(x_t,a_t) \enskip | \enskip x_0=x\right]$,
over the randomness of the policy $\pi$ and transition dynamics $p$. We similarly define the vector-value function for the  constraint costs $G^\pi:\X\mapsto\R^m$ as $G^\pi(x) = \E\left[ \sum_{t=0}^\infty \gamma^t g(x_t,a_t)\smallskip | \smallskip x_0=x \right]$. Define $C(\pi)$ and $G(\pi)$ as the expectation of $C^\pi(x)$ and $G^\pi(x)$, respectively, over the distribution $\chi$ of initial states. 

 \subsection{Batch Policy Learning under Constraints}
In batch policy learning, we have a pre-collected dataset, $\D = \{ (x_i,a_i,x_i^\prime,c(x_i,a_i), g_{1:m}(x_i,a_i) \}_{i=1}^n$, generated from (a set of) historical behavioral policies denoted jointly by $\pi_\D$. The goal of batch policy learning under constraints is to learn a policy $\pi\in\Pi$ from $\D$ that minimizes the primary objective cost while satisfying $m$ different constraints:
 \begin{align}
 \tag{OPT}\label{eqn:main_problem}
 \begin{split}
     &\min_{\pi\in\Pi} \quad C(\pi) \\
     &\text{s.t. } \quad G(\pi) \leq \tau 
\end{split}
 \end{align}
 where $G(\cdot) = \left[g_1(\cdot),\ldots,g_m(\cdot) \right]^\top$ and $\tau\in\R^m$ is a vector of known constants. We assume that \eqref{eqn:main_problem} is feasible. However, the dataset $\D$ might be generated from multiple policies that violate the constraints. 

 \subsection{Examples of Policy Learning with Constraints}
 \textbf{Counterfactual \& Safe Policy Learning.} 
 In conventional online RL, the agent needs to \doubleQuote{re-learn} from scratch when the cost function is modified. Our framework enables counterfactual policy learning assuming the ability to compute the new cost objective from the same historical data. A simple example is \emph{safe} policy learning \cite{garcia2015comprehensive}. Define safety cost $G(x,a) = \phi(x,a,c)$ as a new function of existing cost $c$ and features associated with current state-action pair. The goal here is to counterfactually avoid undesirable behaviors observed from historical data. We experimentally study this safety problem in Section \ref{sec:experiment}. 

 Other examples from the literature that belong to this safety perspective include planning under chance constraints \cite{ono2015chance,blackmore2011chance}. The constraint here is $G(\pi) = \E[\mathbb{I}(x\in\X_{error})] = \mathrm{P}(x\in\X_{error})\leq\tau$. 

 \textbf{Multi-objective Batch Learning.} Traditional policy learning (RL or IL) presupposes that the agent optimizes a single cost function. In reality, we may want to satisfy multiple objectives that are not easily reducible to a scalar objective function. 
 One example is learning fast driving policies under multiple behavioral constraints such as smooth driving and lane keeping consistency (see Section \ref{sec:experiment}). 
 

 \subsection{Equivalence between Constraint Satisfaction and Regularization}
 \label{sec:regularization_view}
 Our constrained policy learning framework subsumes several existing regularized policy learning settings. Regularization typically encodes prior knowledge, and has been used extensively in the RL and IL literature to improve learning performance. Many instances of regularized policy learning can be naturally cast into \eqref{eqn:main_problem}:
 \begin{itemize}[noitemsep,topsep=0pt]
 	\item \emph{Entropy regularized RL} \cite{haarnoja2017reinforcement, ziebart2010modeling} is equivalent to  $G(\pi) = \mathbb{H}(\pi)$, where $\mathbb{H}(\pi)$ measures policy entropy.
 	\item \emph{Smooth imitation learning} \cite{le2016smooth} is equivalent to  $G(\pi)=\min_{h\in\mathrm{H}}\Delta(h,\pi)$, where $H$ is a class of provably smooth policies and $\Delta$ is a divergence metric.
 	\item \emph{Regularizing RL with expert demonstration} \cite{hester2018deep} is equivalent to $G(\pi) = \E[\ell(\pi(x),\pi^*(x))]$, where $\pi^*$ is the expert policy.
 	\item \emph{Conservative policy improvement} \cite{levine2014learning,schulman2015trust, achiam2017constrained} is equivalent to $G(\pi) = D_{KL}(\pi,\pi_k)$, where $\pi_k$ is some \doubleQuote{well-behaving} policy.
 \end{itemize}
 We provide a detailed equivalence derivation of the above examples in Appendix \ref{sec:app-regularization}.  Of course, some problems are more naturally described using the regularization perspective, while others using constraint satisfaction.

 More generally, one can establish the equivalence between regularized and constrained policy learning via a simple appeal to Lagrangian duality as shown in Proposition \ref{prop:equivalence} below.
  This Lagrangian duality also has a game-theoretic interpretation (Section 5.4 of \citet{boyd2004convex}), which serves as an inspiration for developing our approach.
 \begin{prop} \label{prop:equivalence} Let $\Pi$ be a convex set of policies. Let $C:\Pi\mapsto\R, C:\Pi\mapsto\R^K$ be value functions. Consider the two policy optimization tasks:
 \begin{align*}
     \texttt{Regularization:}&\qquad\min_{\pi\in\Pi} \quad C(\pi) + \lambda^\top G(\pi)  \\
     \texttt{Constraint:}&\qquad\min_{\pi\in\Pi} \quad C(\pi) \quad\text{s.t. } G(\pi)\leq \tau \nonumber
 \end{align*}
 Assume that the Slater's condition is satisfied in the \texttt{Constraint} problem (i.e., $\exists \pi$ s.t. $G(\pi)<\tau$). Assume also that the constraint cannot be removed without changing the optimal solution, i.e., $\inf_{\pi\in\Pi} C(\pi) < \inf_{\pi\in\Pi:G(\pi)\leq\tau} C(\pi)$. Then $\forall$ $\lambda>0$, $\exists$ $\tau$, and vice versa, such that \texttt{Regularization} and \texttt{Constraint} share the same optimal solutions. (Proof in Appendix \ref{sec:app-regularization}.)
 \end{prop}

\section{Proposed Approach}
To make use of strong duality, we first \emph{convexify} the policy class $\Pi$ by allowing stochastic combinations of policies, which effectively expands $\Pi$ into its convex hull $\texttt{Conv}(\Pi)$. Formally, $\texttt{Conv}(\Pi)$ contains \emph{randomized policies},\footnote{This places no restrictions on the individual policies. Individual policies can be arbitrarily non-convex. Convexifiying a policy class amounts to allowing ensembles of learned policies.} which we denote $\pi = \sum_{t=1}^T \alpha_t\pi_t$ for $\pi_t\in\Pi$ and $\sum_{t=1}^T \alpha_t = 1$. Executing a mixed $\pi$ consists of first sampling \emph{one} policy $\pi_t$ from $\pi_{1:T}$ according to distribution $\alpha_{1:T}$, and then executing $\pi_t$. 
Note that we still have $\E[\pi] = \sum_{t=1}^T\alpha_t\E[\pi_t]$ for any first-moment statistic of interest (e.g., state distribution, expected cost).
It is easy to see that the augmented version of \eqref{eqn:main_problem} over $\texttt{Conv}(\Pi)$
has a solution at least as good as the original \eqref{eqn:main_problem}. As such, to lighten the notation, we will equate $\Pi$ with its convex hull for the rest of the paper.
\subsection{Meta-Algorithm}
 The Lagrangian of \eqref{eqn:main_problem} is $L(\pi,\lambda) = C(\pi)+\lambda^\top(G(\pi)-\tau)$
 for $\lambda\in\R^m_{+}$. Clearly \eqref{eqn:main_problem} is equivalent to the min-max problem: $\min\limits_{\pi\in\Pi}\max\limits_{\lambda\in\R^k_{+}} L(\pi,\lambda)$.
 We assume  \eqref{eqn:main_problem} is feasible and that Slater's condition holds (otherwise, we can simply increase the constraint $\tau$ by a tiny amount). Slater's condition and policy class convexification ensure that strong duality holds \cite{boyd2004convex}, and \eqref{eqn:main_problem} is also equivalent to the max-min problem:$\max\limits_{\lambda\in\R^k_{+}}\min\limits_{\pi\in\Pi} L(\pi,\lambda)$.
 
 Since $L(\pi,\lambda)$ is linear in both $\lambda$ and $\pi$, strong duality is also a consequence of von Neumann's celebrated convex-concave minimax theorem for zero-sum games \cite{von2007theory}. From a game-thoeretic perspective, the problem becomes finding the equilibrium of a two-player game between the $\pi-$player and the $\lambda-$player \cite{freund1999adaptive}. In this repeated game, the $\pi-$player minimizes  $L(\pi,\lambda)$ given the current $\lambda$, and the $\lambda-$player maximizes it given the current (mixture over) $\pi$.

 \begin{algorithm}[t]
	\begin{small}
	\caption{ Meta-algo for Batch Constrained Learning} 
	\label{algo:meta}
	\begin{algorithmic}[1]
	    \FOR{each round $t$}
		\STATE $\pi_t\leftarrow \texttt{Best-response}(\lambda_t)$\\
		\STATE $\widehat{\pi}_t \leftarrow \frac{1}{t} \sum_{t^\prime = 1}^{t} \pi_{t^\prime}$, $\widehat{\lambda}_t \leftarrow \frac{1}{t}\sum_{t^\prime = 1}^t \lambda_{t^\prime}$ \\
		\STATE $\Lmax = \max_{\lambda}L(\widehat{\pi}_t,\lambda)$\\
		\STATE $\Lmin = L(\texttt{Best-response}(\widehat{\lambda}_t),\widehat{\lambda}_t)$\\
		\IF{$\Lmax-\Lmin \leq \omega$}
		\STATE Return $\widehat{\pi}_t$\\
		\ENDIF
	    \STATE $\lambda_{t+1}\leftarrow \texttt{Online-algorithm}(\pi_1,\ldots,\pi_{t-1},\pi_t)$
		\ENDFOR
	\end{algorithmic}
	\end{small}
\end{algorithm}

 We first present a meta-algorithm (Algorithm \ref{algo:meta}) that uses any no-regret online learning algorithm (for $\lambda$) and batch policy optimization (for $\pi$). 
	At each iteration, given $\lambda_t$, the $\pi$-player runs \texttt{Best-response} to get the best response:
\begin{align*}
    \texttt{Best-response}&(\lambda_t) = \argmin_{\pi\in\Pi} L(\pi,\lambda_t)\\
    &=\argmin_{\pi\in\Pi}C(\pi)+\lambda_t^\top (G(\pi)-\tau).
\end{align*}
This is equivalent to a standard batch reinforcement learning problem where we learn a policy that is optimal with respect to $c+\lambda_t^\top g$. The corresponding mixed strategy is the uniform distribution over all previous $\pi_t$. 
	In response to the $\pi-$player, the $\lambda-$player employs \texttt{Online-algorithm}, which can be \emph{any} no-regret  algorithm that satisfies:
    $$\sum_t L(\pi_t,\lambda_t) \geq \max_{\lambda} \sum_t L(\pi_t,\lambda) - o(T)$$
	Finally, the algorithm terminates when the estimated primal-dual gap is below a threshold $\omega$ (Lines 7-8).

Leaving aside (for the moment) issues of generalization, Algorithm \ref{algo:meta} is guaranteed to converge assuming: (i)  \texttt{Best-response} gives the best single policy in the class, and (ii) $\Lmax$ and $\Lmin$ can  be evaluated exactly. 
\begin{prop} 
\label{prop:convergence}
Assuming (i) and (ii) above, Algorithm \ref{algo:meta} is guaranteed to stop and the convergence depends on the regret of \texttt{Online-algorithm}. (Proof in Appendix \ref{sec:app-convergence}.)
\end{prop}


\begin{algorithm}[t]
	\begin{small}
	\caption{ Constrained Batch Policy Learning} 
	\label{algo:main_algo}
	\begin{algorithmic}[1]
	\REQUIRE Dataset $\D = \{ x_i,a_i,x_i^\prime,c_i,g_i\}_{i=1}^n \sim\pi_{\D}$. Online algorithm parameters: $\ell_1$ norm bound $B$, learning rate $\eta$
	\STATE Initialize $\lambda_1 = (\frac{B}{m+1},\ldots,\frac{B}{m+1})\in\R^{m+1}$
	    \FOR{each round $t$}
		\STATE Learn $\pi_t\leftarrow \texttt{FQI}(c+\lambda_t^\top g)$ \label{line:fqi} \hfill \small{// \textit{FQI with cost $c+\lambda_t^\top g$}}\\
		\STATE Evaluate $\widehat{C}(\pi_t)\leftarrow\texttt{FQE}(\pi_t,c)$ \hfill \small{// \textit{Algo \ref{algo:fqe} with $\pi_t$, cost $c$}}\\
		\STATE Evaluate $\widehat{G}(\pi_t)\leftarrow\texttt{FQE}(\pi_t,g)$ \hfill \small{// \textit{Algo \ref{algo:fqe} with $\pi_t$}, cost $g$}\\
		\STATE $\widehat{\pi}_t \leftarrow \frac{1}{t} \sum_{t^\prime = 1}^{t} \pi_{t^\prime}$ \\
		\STATE $\widehat{C}(\widehat{\pi}_t) \leftarrow \frac{1}{t} \sum_{t^\prime = 1}^{t} \widehat{C}(\pi_{t^\prime})$, $\widehat{G}(\widehat{\pi}_t) \leftarrow \frac{1}{t} \sum_{t^\prime = 1}^{t} \widehat{G}(\pi_{t^\prime})$  \\
		\STATE $\widehat{\lambda}_t \leftarrow \frac{1}{t}\sum_{t^\prime = 1}^t \lambda_{t^\prime}$ \\
		\STATE Learn $\widetilde{\pi}\leftarrow\texttt{FQI}(c+\widehat{\lambda}_t^\top g)$ \hfill \small{// \textit{FQI with cost $c+\widehat{\lambda}_t^\top g$}} \\
		\STATE Evaluate $\widehat{C}(\widetilde{\pi})\leftarrow\texttt{FQE}(\widetilde{\pi},c), \widehat{G}(\widetilde{\pi})\leftarrow\texttt{FQE}(\widetilde{\pi},g)$ \\
		\STATE $\Lhmax= \max\limits_{\lambda, \norm{\lambda}_1=B} \left( \widehat{C}(\widehat{\pi}_t) + \lambda^\top \left[ (\widehat{G}(\widehat{\pi}_t)-\tau)^\top, 0 \right]^\top\right)$
		\STATE $\Lhmin = \widehat{C}(\widetilde{\pi}) + \widehat{\lambda}_t^\top\left[ (\widehat{G}(\widetilde{\pi})-\tau)^\top,0 \right]^\top$
		\IF{$\Lhmax - \Lhmin \leq \omega$}
		\STATE Return $\widehat{\pi}_t$\\
		\ENDIF
		\STATE Set $z_t= \left[ (\widehat{G}(\pi_t)-\tau)^\top,0\right]^\top \in\R^{m+1}$\\
		\STATE $\lambda_{t+1}[i] = B\frac{\lambda_t[i] e^{-\eta z_t[i]}}{\sum_j \lambda_t[j] e^{-\eta z_t[j]}}\forall i$ \hfill \small{// \textit{$\lambda[i]$ the $i^{th}$ coordinate}}
		\ENDFOR
	\end{algorithmic}
	\end{small}
\end{algorithm}

\subsection{Our Main Algorithm}
We now focus on a specific instantiation of Algorithm \ref{algo:meta}. Algorithm \ref{algo:main_algo} is our main algorithm in this paper.  

\textbf{Policy Learning.} We instantiate $\texttt{Best-response}$  with Fitted Q Iteration (FQI), a model-free off-policy learning approach \cite{ernst2005tree}. FQI relies on a series of reductions to supervised learning. The key idea is to approximate the true action-value function $Q^*$ by a sequence $\{Q_k\in\F\}_{k=0}^K$, where $\F$ is a chosen function class.  

In Lines 3 \& 9, $\texttt{FQI}(c+\lambda^\top g)$  is defined as follows.
With $Q_0$ randomly initialized, for each $k=1,\ldots,K$, we form a new training dataset $\widetilde{\D}_k = \{(x_i,a_i),y_i \}_{i=1}^n$ where: 
$$ \forall i: \enskip y_i = (c_i+\lambda^\top g_i)+\gamma\min_a Q_{k-1}(x_i^\prime,a),$$ 
and $(x_i,a_i,x_i^\prime,c_i,g_i)\sim\D$ (original dataset). A supervised regression procedure is called to solve for:
$$Q_k = \argmin\limits_{f\in\F} \frac{1}{n}\sum_{i=1}^n (f(x_i,a_i)-y_i )^2.$$
The policy then: $\pi_K = \argmin_{a} Q_K(\cdot,a)$. 

FQI has been shown to work well with several empirical domains: spoken dialogue systems \cite{pietquin2011sample}, physical robotic soccer \cite{riedmiller2009reinforcement}, and cart-pole swing-up \cite{riedmiller2005neural}. Another possible model-free subroutine is Least-Squares Policy Iteration (LSPI) \cite{lagoudakis2003least}. One can also consider model-based alternatives \cite{ormoneit2002kernel}. 

\textbf{Off-policy Policy Evaluation.} A crucial difference between constrained policy learning and existing work on constrained supervised learning is the technical challenge of evaluating the objective and constraints. First, estimating $\Lh(\pi,\lambda)$ (Lines 11-12) requires estimating $\Ch(\pi)$ and $\Gh(\pi)$. Second, any gradient-based approach to \texttt{Online-learning} requires passing in $\Gh(\pi)-\tau$ as part of gradient estimate (line 15). This problem is known as the off-policy policy evaluation (OPE) problem: we need to evaluate $\Ch(\pi)$ and $\Gh(\pi)$ having only access to data $\D\sim\pi_\D$

There are three main contemporary approaches to OPE: (i) importance weighting (IS) \cite{precup2000eligibility,precup2001off}, which is unbiased but often has high-variance; (ii) regression-based direct methods (DM), which are typically model based \cite{thomas2016data},\footnote{I.e., using regression to learn the reward function and transition dynamics model, before solving the estimated MDP.} and can be biased but have much lower variance than IS; and (iii)  doubly-robust techniques \cite{jiang2016doubly,dudik2011doubly}, which combine IS and DM.

We propose a new and simple model-free technique using function approximation, called Fitted Q Evaluation (FQE). FQE is based on an iterative reductions scheme similar to FQI, but for the problem of off-policy evaluation. Algorithm \ref{algo:fqe} lays out the steps. The key difference with FQI is that the $min$ operator is replaced by $Q_{k-1}(x_i^\prime, \pi(x_i^\prime))$ (Line 3 of Algorithm \ref{algo:fqe}). Each $x_i^\prime$ comes from the original $\D$. Since we know $\pi(x_i^\prime)$, each $\widetilde{\D}_k$ is well-defined. 
Note that FQE can be plugged-in as a direct method if one wishes to augment the policy evaluation with a doubly-robust technique. 

\begin{algorithm}[t]
	\begin{small}
	\caption{ Fitted Off-Policy Evaluation with Function Approximation: $\texttt{FQE}(\pi,c)$} 
	\label{algo:fqe}
	\begin{algorithmic}[1]
		\REQUIRE Dataset $\D = \{ x_i,a_i,x_i^\prime,c_i\}_{i=1}^n \sim\pi_{\D}$. Function class $\F$. Policy $\pi$ to be evaluated
        \STATE Initialize $Q_0 \in \F$ randomly
		\FOR{$k = 1,2,\ldots,K$}
		\STATE Compute target $y_i = c_i+\gamma Q_{k-1}(x_i^\prime,\pi(x_i^\prime)) \enskip \forall i$ \\
		\STATE Build training set $\widetilde{\D}_k = \{(x_i,a_i),y_i \}_{i=1}^n$
		\STATE Solve a supervised learning problem: \\
		\quad $Q_k = \argmin\limits_{f\in\F} \frac{1}{n}\sum_{i=1}^n (f(x_i,a_i)-y_i )^2$
		\ENDFOR
		\ENSURE $\widehat{C}^\pi(x) = Q_K(x,\pi(x))\quad\forall x$ 
	\end{algorithmic}
	\end{small}
\end{algorithm}

\textbf{Online Learning Subroutine.} 
 As $L(\pi_t,\lambda)$ is linear in $\lambda$, many online convex optimization approaches can be used for $\texttt{Online-algorithm}$. Perhaps the simpliest choice is Online Gradient Descent (OGD) \cite{zinkevich2003online}. We include an instantiation using OGD in Appendix \ref{sec:app-algorithm}. 

For our main Algorithm \ref{algo:main_algo}, similar to \cite{agarwal2018reductions}, we use Exponentiated Gradient (EG) \cite{kivinen1997exponentiated}, which has a regret bound of $O(\sqrt{\log(m)T})$ instead of $O(\sqrt{mT})$ as in OGD. One can view EG as a variant of Online Mirror Descent \cite{nemirovsky1983problem} with a softmax link function, or of Follow-the-Regularized-Leader with entropy regularization \cite{shalev2012online}.  Gradient-based algorithms generally require bounded $\lambda$. We thus force $\norm{\lambda}_1\leq B$ using hyperparameter $B$. Solving \eqref{eqn:main_problem} exactly requires $B=\infty$. We will analyze Algorithm \ref{algo:main_algo} with respect to finite $B$. With some abuse of notation, we augment $\lambda$ into a $(m+1)-$dimensional vector by appending $B-\norm{\lambda}_1$,
and augment the constraint cost vector $g$  by appending $0$ (Lines 11, 12 \& 15 of Algorithm \ref{algo:main_algo}).\footnote{The $(m+1)^{th}$ coordinate of $g$ is thus always satisfied. This augmentation is only necessary when executing EG.}


\section{Theoretical Analysis}
\subsection{Convergence Guarantee}
The convergence rate of Algorithm \ref{algo:main_algo} depends on the radius $B$ of the dual variables $\lambda$, the maximal constraint value $\widebar{G}$, and the number of constraints $m$. In particular, we can show $O(\frac{B^2}{\omega^2})$ convergence for primal-dual gap $\omega$.
\begin{thm}[Convergence of Algorithm \ref{algo:main_algo}]\label{thm:convergence_main} After $T$ iterations, the empirical duality gap is bounded by 
\begin{equation*}
\Lhmax - \Lhmin \leq 2\frac{B\log(m+1)}{\eta T}+2\eta B\widebar{G}^2
\end{equation*}
Consequently, to achieve the primal-dual gap of $\omega$, setting $\eta = \frac{\omega}{4\widebar{G}^2B}$ will ensure that Algorithm \ref{algo:main_algo} converges after $\frac{16B^2\widebar{G}^2\log(m+1)}{\omega^2}$ iterations. (Proof in Appendix \ref{sec:app-convergence}.)
\end{thm}
Convergence analysis of our main Algorithm \ref{algo:main_algo} is an extension of the proof to Proposition \ref{prop:convergence}, leveraging the no-regret property of the EG procedure \cite{shalev2012online}. 

\subsection{Generalization Guarantee of FQE and FQI}
In this section, we provide sample complexity analysis for FQE and FQI as \emph{standalone} procedures for off-policy evaluation and off-policy learning. 
We use the notion of pseudo-dimension as capacity measure of non-linear function class $\F$ \cite{friedman2001elements}. Pseudo-dimension $\texttt{dim}_\F$, which naturally extends VC dimension into the regression setting, is defined as the VC dimension of the function class induced by the sub-level set of functions of $\F$: $\texttt{dim}_\F = \texttt{VC-dim}(\{(x,y)\mapsto \text{sign}(f(x)-y) : f\in\F\})$. 
Pseudo-dimension is finite for a large class of function approximators. For example, \citet{bartlett2017nearly} bounded the pseudo-dimension of piece-wise linear deep neural networks (e.g., with ReLU activations) as $O(WL\log W)$, where $W$ is the number of weights, and $L$ is the number of layers. 

Both FQI and FQE rely on reductions to supervised learning to update the value functions. In both cases, the learned policy and evaluation policy induces a different state-action distribution compared to the data generating distribution $\mu$. We use the notion of concentration coefficient for the worst case, proposed by \cite{munos2003error}, to measure the degree of distribution shift. The following standard assumption from analysis of related ADP algorithms limits the severity of distribution shift over future time steps:
\begin{assumption}[Concentration coefficient of future state-action distribution]\cite{munos2003error,munos2007performance,munos2008finite,antos2008fitted,antos2008learning,lazaric2010finite,lazaric2012finite,farahmand2009regularized,maillard2010finite}
\label{assume:concentrability_main}\\
Let $P^\pi$ be the operator acting on $f:\X\times\A\mapsto\R$ s.t. $(P^\pi f)(x,a) = \int_\X f(x^\prime, \pi(x^\prime)) p(dx^\prime|x,a)$. Given data generating distribution $\mu$, initial state distribution $\chi$, for $m \geq 0$ and an arbitrary sequence of stationary policies $\{ \pi_m\}_{m\geq 1}$ define the concentration coeffient:
\begin{equation*}
\beta_{\mu}(m) = \sup_{\pi_1,\ldots,\pi_m} \left\| \frac{d(\chi P^{\pi_1}P^{\pi_2}\ldots P^{\pi_m})}{d\mu} \right\|_\infty
\end{equation*}
\\
We assume $\beta_{\mu} = (1-\gamma)^2\sum\limits_{m\geq 1} m\gamma^{m-1} \beta_{\mu}(m) < \infty$.
\end{assumption}
This assumption is valid for a fairly large class of MDPs \cite{munos2007performance}. For instance $\beta_\mu$ is finite for any finite MDP, or any infinite state-space MDP with bounded transition density.\footnote{This assumption ensures sufficient data diversity, even when the executing policy is deterministic. An example of how learning can fail without this assumption is based on the ``combination lock'' MDP \cite{koenig1996effect}. 
In this deterministic MDP example, $\beta_\mu$ can grow arbitrarily large, and we need an exponential number of samples for both FQE and FQI. See Appendix \ref{sec:appendix_preliminaries}.} Having a finite concentration coefficient is equivalent the top-Lyapunov exponent $\Gamma\leq 0$ \cite{bougerol1992strict}, which means the underlying stochastic system is stable. We show below a simple sufficient condition for Assumption \ref{assume:concentrability_main} (albeit stronger than necessary). 
\begin{exmp} Consider an MDP such that for any non-stationary distribution $\rho$, the marginals over states satisfy $\frac{\rho_x(x)}{\mu_x(x)}\leq L$ (i.e., transition dynamics are sufficiently stochastic), and $\exists M:\ \forall x,a:\ \mu(a|x)>\frac{1}{M}$ (i.e., the behavior policy is sufficiently exploratory). Then $\beta_{\mu} \leq LM$.
\end{exmp}

Recall that for a given policy $\pi$, the Bellman (evaluation) operator is defined as $(\Tpi Q)(x,a) = r(x,a) + \gamma\int_{\X} Q(x^\prime,\pi(x^\prime))p(dx^\prime|x,a)$. In general $\T^\pi f$ may not belong to $\F$ for $f\in\F$. For FQE (and FQI), the main operation in the algorithm is to iteratively project $\T^\pi Q_{k-1}$ back to $\F$ via $Q_k = \argmin_{f\in\F}\norm{f-\T^\pi Q_{k-1}}$. The performance of both FQE and FQI thus depend on how well the function class $\F$ approximates the Bellman operator. We measure the ability of function class $\F$ to approximate the Bellman evaluation operator via the worst-case Bellman error:
\begin{defn}[inherent Bellman evaluation error] Given a function class $\F$ and policy $\pi$, the \emph{inherent Bellman evaluation error} of $\F$ is defined as $d_\F^\pi = \sup_{g\in\F}\inf_{f\in\F}\norm{f-\T^\pi g}_{\pi}$
where $\norm{\cdot}_{\pi}$ is the $\ell_2$ norm weighted by the state-action distribution induced by $\pi$.
\end{defn}
We are now ready to state the generalization bound for FQE:
\begin{thm}[Generalization error of FQE]\label{thm:fqe_main} 
Under Assumption \ref{assume:concentrability_main}, for $\epsilon>0$ \& $\delta\in(0,1)$, after $K$ iterations of Fitted Q Evaluation (Algorithm \ref{algo:fqe}), for $n=O\big(\frac{\widebar{C}^4}{\epsilon^2}( \log\frac{K}{\delta}+\textnormal{\texttt{dim}}_{\F}\log\frac{\widebar{C}^2}{\epsilon^2}+\log \textnormal{\texttt{dim}}_{\F})\big)$, we have with probability $1-\delta$:
\begin{small}
$$\big\lvert C(\pi) - \widehat{C}(\pi)\big\rvert \leq \frac{\gamma^{1/2}}{(1-\gamma)^{3/2}} \big( \sqrt{\beta_{\mu}}\left(2d_\F^\pi+\epsilon\right) + \frac{2\gamma^{K/2}\widebar{C}}{(1-\gamma)^{1/2}}\big).$$
\end{small}
\end{thm}
This result shows a dependency on $\epsilon$ of $\widetilde{O}(\frac{1}{\epsilon^2})$, compared to $\widetilde{O}(\frac{1}{\epsilon^4})$ from other related ADP algorithms \cite{munos2008finite,antos2008learning}. The price that we pay is a multiplicative constant 2 in front of the inherent error $d_\F^\pi$. The error from second term on RHS decays exponentially with iterations $K$. The proof is in Appendix \ref{sec:appendix_proof_fqe}.

We can show an analogous generalization bound for FQI.
While FQI has been widely used, to the best of our knowledge, a complete analysis of FQI for non-linear function approximation has not been previously reported.\footnote{FQI for continuous action space from \cite{antos2008fitted} is a variant of fitted policy iteration and not the version of FQI under consideration. The appendix of \cite{lazaric2011transfer} contains a proof of FQI but for linear functions.} 
\begin{defn}[inherent Bellman optimality error]\cite{munos2008finite} Recall that the Bellman optimality operator is defined as $(\T Q)(x,a) = r(x,a) + \gamma\int_{\X}\min_{a^\prime\in\A} Q(x^\prime,a^\prime)p(dx^\prime|x,a)$. Given a function class $\F$, the \emph{inherent Bellman error} is defined as $d_\F = \sup_{g\in\F}\inf_{f\in\F}\norm{f-\T g}_{\mu}$, where $\norm{\cdot}_\mu$ is the $\ell_2$ norm weighted by $\mu$, the state-action distribution induced by $\pi_\D$.
\end{defn}
\begin{thm}[Generalization error of FQI]\label{thm:fqi_main} 
Under Assumption \ref{assume:concentrability_main}, for $\epsilon>0$ \& $\delta\in(0,1)$, after $K$ iterations of Fitted Q Iteration, for $n=O\big(\frac{\widebar{C}^4}{\epsilon^2}( \log\frac{K}{\delta}+\pdim_{\F}\log\frac{\widebar{C}^2}{\epsilon^2}+\log \pdim_{\F})\big)$, we have with probability $1-\delta$:
$$\big\lvert C^*-C(\pi_K)\big\rvert \leq \frac{2\gamma}{(1-\gamma)^3}\big( \sqrt{\beta_{\mu}}\left(2d_\F+\epsilon\right) + 2\gamma^{K/2} \widebar{C}\big)$$
where $\pi_K$ is the policy acting greedy with respect to the returned function $Q_K$.  (Proof in Appendix \ref{sec:proof-fqi}.)
\end{thm}
\subsection{End-to-End Generalization Guarantee}
We are ultimately interested in the test-time performance and constraint satisfaction of the returned policy from Algorithm \ref{algo:main_algo}. We now connect the previous analyses from Theorems \ref{thm:convergence_main}, \ref{thm:fqe_main} \& \ref{thm:fqi_main} into an end-to-end error analysis.

Since Algorithm \ref{algo:main_algo} uses FQI and FQE as subroutines, the inherent Bellman error terms $d_\F$ and $d_\F^\pi$ will enter our overall performance bound. Estimating the inherent Bellman error caused by function approximation is not possible in general (chapter 11 of \citet{sutton2018reinforcement}). Fortunately, a sufficiently expressive $\F$ can generally make $d_\F$ and $d_\F^\pi$ to arbitrarily small. 
To simplify our end-to-end analysis, we assume $d_\F=0$ and $d_\F^\pi=0$, i.e.,  the function class $\F$ is closed under applying the Bellman operator.\footnote{A similar assumption was made in \citet{cheng2019accelerating} on near-realizability of learning the model dynamics.}
\begin{assumption}
\label{assume:realizability_main}
We consider function classes $\F$ sufficiently rich so that $\forall f:\ \T f\in\F$ \& $\T^\pi f\in\F$ for the policies $\pi$ returned by Algorithm \ref{algo:main_algo}.
\end{assumption}
With Assumptions \ref{assume:concentrability_main} \& \ref{assume:realizability_main}, we have the following error bound:
\begin{thm}[Generalization guarantee of Algorithm \ref{algo:main_algo}] \label{thm:end_to_end_main}
Let $\pi^*$ be the optimal policy to \eqref{eqn:main_problem}. Denote $\widebar{V}=\widebar{C}+B\widebar{G}$. Let $K$ be the number of iterations of FQE and FQI. Let $\pih$ be the policy returned by Algorithm \ref{algo:main_algo}, with termination threshold $\omega$. For $\epsilon>0$ \& $\delta\in(0,1)$, when $n = O\big(\frac{\widebar{V}^4}{\epsilon^2}( \log\frac{K(m+1)}{\delta}+\textnormal{\texttt{dim}}_{\F}\log\frac{\widebar{V}^2}{\epsilon^2}+\log\textnormal{\texttt{dim}}_{\F})\big)$, we have with probability at least $1-\delta$:
$$C(\pih)\leq C(\pi^*) + \omega + \frac{(4+B)\gamma}{(1-\gamma)^3}\big( \sqrt{\beta_{\mu}}\epsilon + 2\gamma^{K/2} \widebar{V}\big),$$
and
$$G(\pih)\leq \tau + 2\frac{\widebar{V}+\omega}{B} +\frac{\gamma^{1/2}}{(1-\gamma)^{3/2}} \big( \sqrt{\beta_{\mu}}\epsilon + \frac{2\gamma^{K/2}\widebar{V}}{(1-\gamma)^{1/2}}\big).$$
\end{thm}
The proof is in Appendix \ref{sec:app-generalization}. This result guarantees that, upon termination of Algorithm \ref{algo:main_algo}, the true performance on the main objective can be arbitrarily close to that of the optimal policy. At the same time, each constraint will be approximately satisfied with high probability, assuming sufficiently large  $B$ \& $K$, and sufficiently small $\epsilon$.  
\makeatletter
\newcommand{\customlabel}[3]{%
   \protected@write \@auxout {}{\string \newlabel {#1}{{\ref{#2}~(#3)}{\thepage}{Subfigure #2 (#3)\relax}{}{}} }%
   \protected@write \@auxout {}{\string \newlabel {sub@#1}{{#3}{\thepage}{Subfigure #2 (#3)\relax}{}{}} }%
}
\makeatother

\section{Empirical Analysis}
\label{sec:experiment}
We perform experiments on two different domains: a grid-world domain (from OpenAI's FrozenLake) under a safety constraint, and a challenging high-dimensional car racing domain (from OpenAI's CarRacing) under multiple behavior constraints. We seek to answer the following questions in our experiments: (i) whether the empirical convergence behavior of Algorithm \ref{algo:main_algo} is consistent with our theory; and (ii) how the returned policy performs with respect to the main objective and constraint satisfaction. Appendix \ref{sec:app-experiment} includes a more detailed discussion of our experiments. 


\subsection{Frozen Lake.}
\textbf{Environment \& Data Collection.} The environment is an 8x8 grid. The agent has 4 actions N,S,E,W at each state. The main goal is to navigate from a starting position to the goal. Each episode terminates when the agent reaches the goal or falls into a hole. The main cost function is defined as $c=-1$ if goal is reached, otherwise $c=0$ everywhere. We simulate a non-optimal data gathering policy $\pi_\D$ by adding random sub-optimal actions to the shortest path policy from any given state to goal. We run $\pi_D$ for 5000 trajectories to collect the behavior dataset $\D$ (with constraint cost measurement specified below).

\textbf{Counterfactual Safety Constraint.} We augment the main objective $c$ with safety constraint cost defined as $g=1$ if the agent steps into a hole, and $g=0$ otherwise. We set the constraint threshold $\tau=0.1$, roughly $75\%$ of the accumulated constraint cost of behavior policy $\pi_\D$. The threshold can be interpreted as a counterfactually acceptable probability that we allow the learned policy to fail. 

\textbf{Results.} The empirical primal dual gap $\Lhmax - \Lhmin$ in Figure~\ref{fig:lake_primal_dual_curves} quickly decreases toward the optimal gap of zero. The convergence is fast and monotonic, supporting the predicted behavior from our theory. The test-time performance in Figure~\ref{fig:lake_values} shows the safety constraint is always satisfied, while the main objective cost also smoothly converges to the optimal value achieved by an online RL baseline (DQN) trained without regard to the constraint. The returned policy significantly outperformed the data gathering policy $\pi_\D$ on both main and safety cost.

\begin{figure*}[t]
    \centering     
\hfill%
    \customlabel{fig:lake_primal_dual_curves}{fig:lake}{left}%
        \includegraphics[width=1.8in,valign=t]{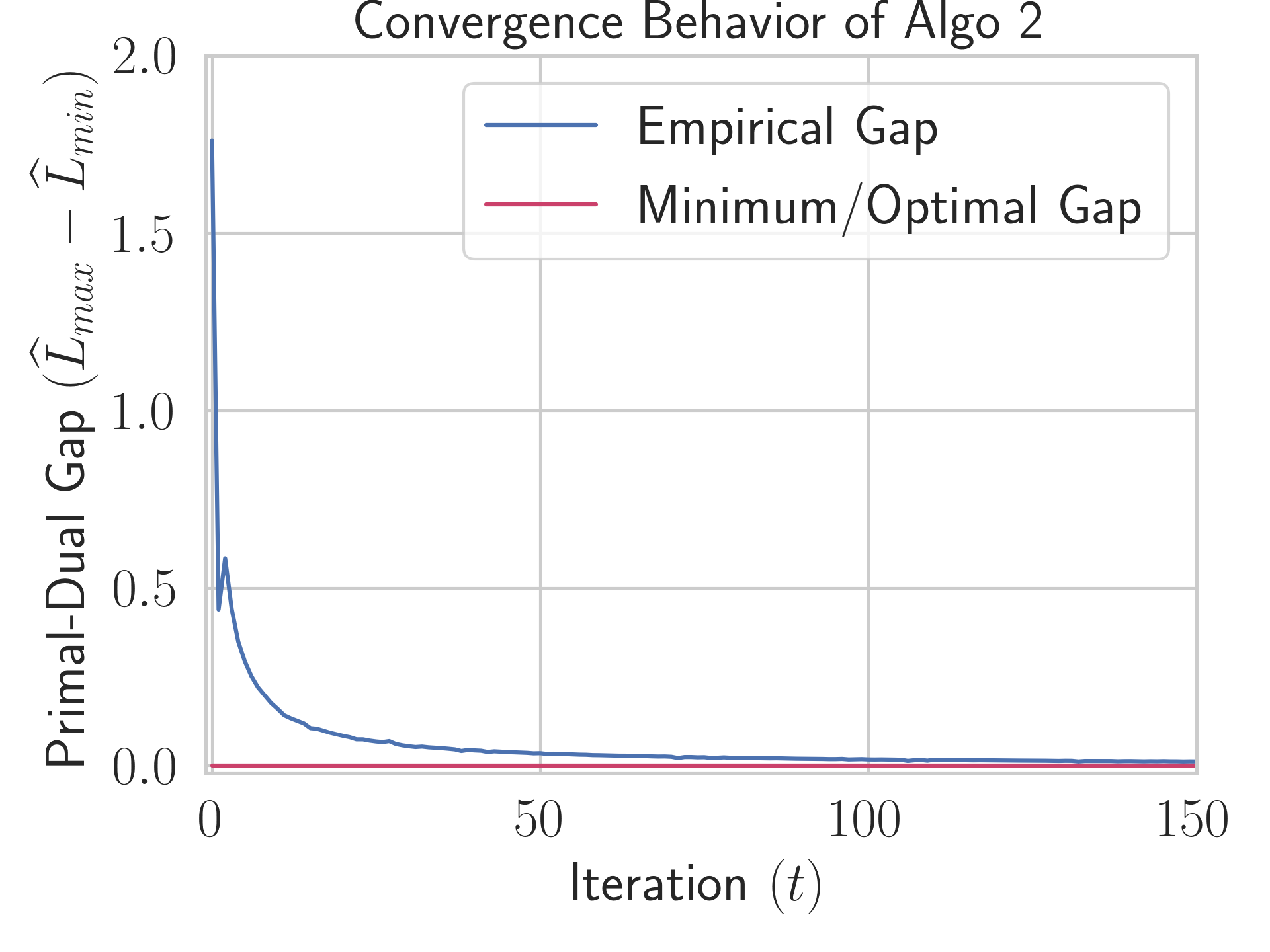}
\hfill%
    \customlabel{fig:lake_values}{fig:lake}{middle}%
        \includegraphics[width=1.8in,valign=t]{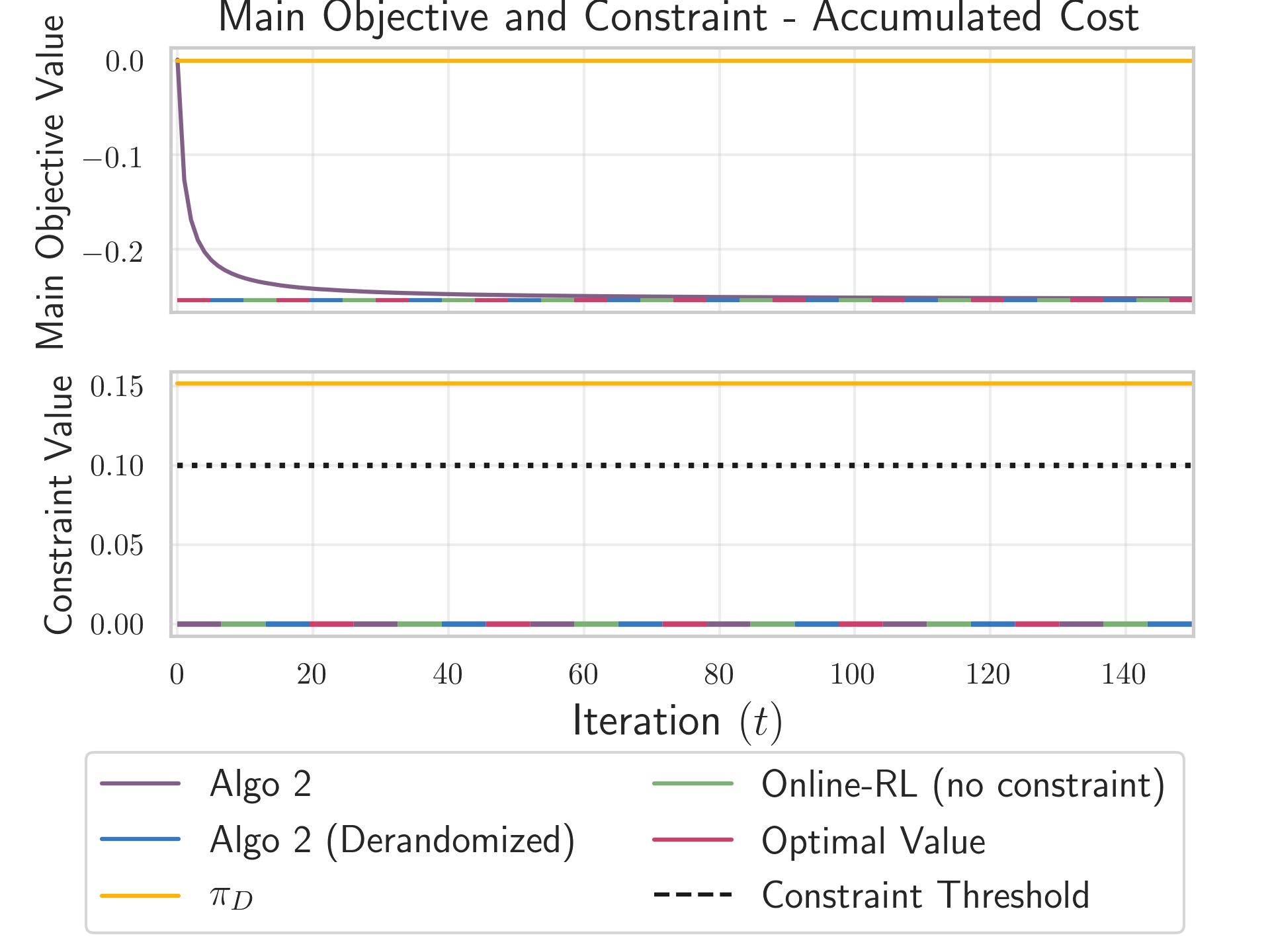}
\hfill%
    \customlabel{fig:lake_fqe}{fig:lake}{right}%
        \includegraphics[width=1.8in,valign=t]{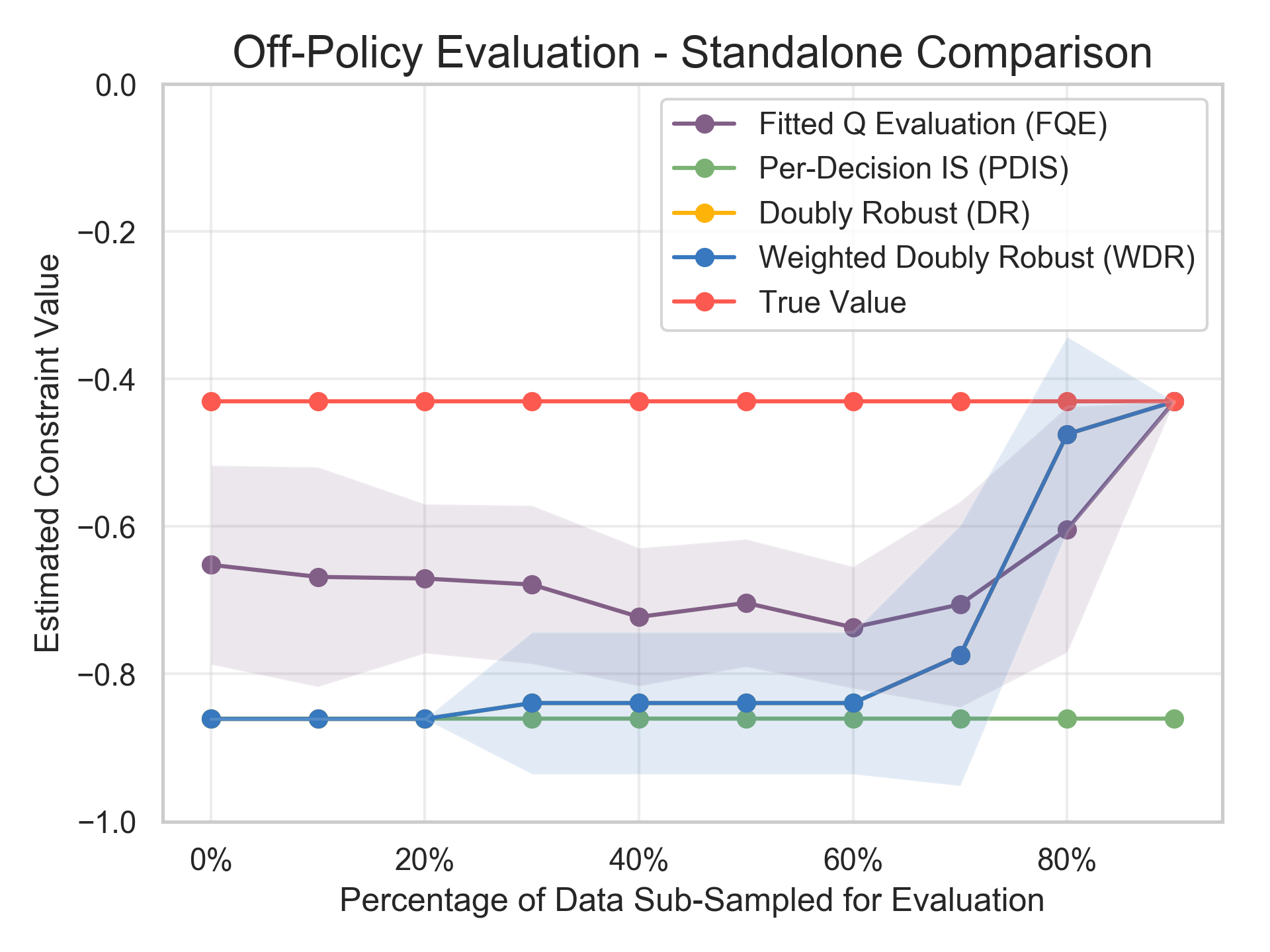}
        \vspace{-0.1in}
\caption{\emph{FrozenLake Results}.
             \emph{(Left)} Empirical duality gap of algorithm \ref{algo:main_algo} vs. optimal gap.
             \emph{(Middle)} Comparison of returned policy and others w.r.t. (top) main objective value and (bottom) safety constraint value.
             \emph{(Right)} FQE vs. other OPE methods on a standalone basis.
             }
\label{fig:lake}
\end{figure*}


\begin{figure*}[t]
    \centering     
    \customlabel{fig:car_main_values}{fig:car}{left \& middle}%
        \includegraphics[width=3.6in,valign = c]{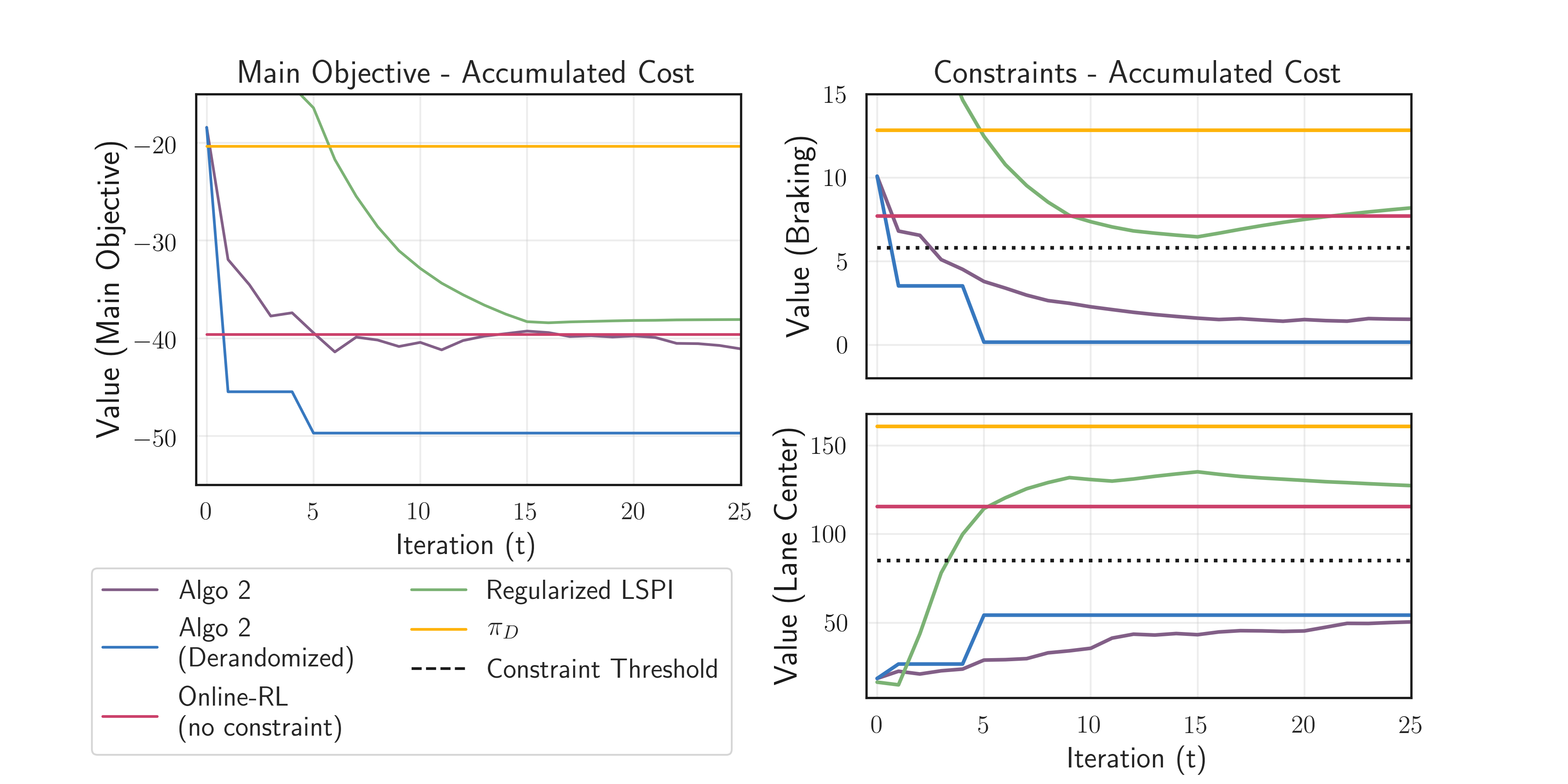}
    \customlabel{fig:car_fqe}{fig:car}{right}%
        \includegraphics[width=1.8in,valign=c]{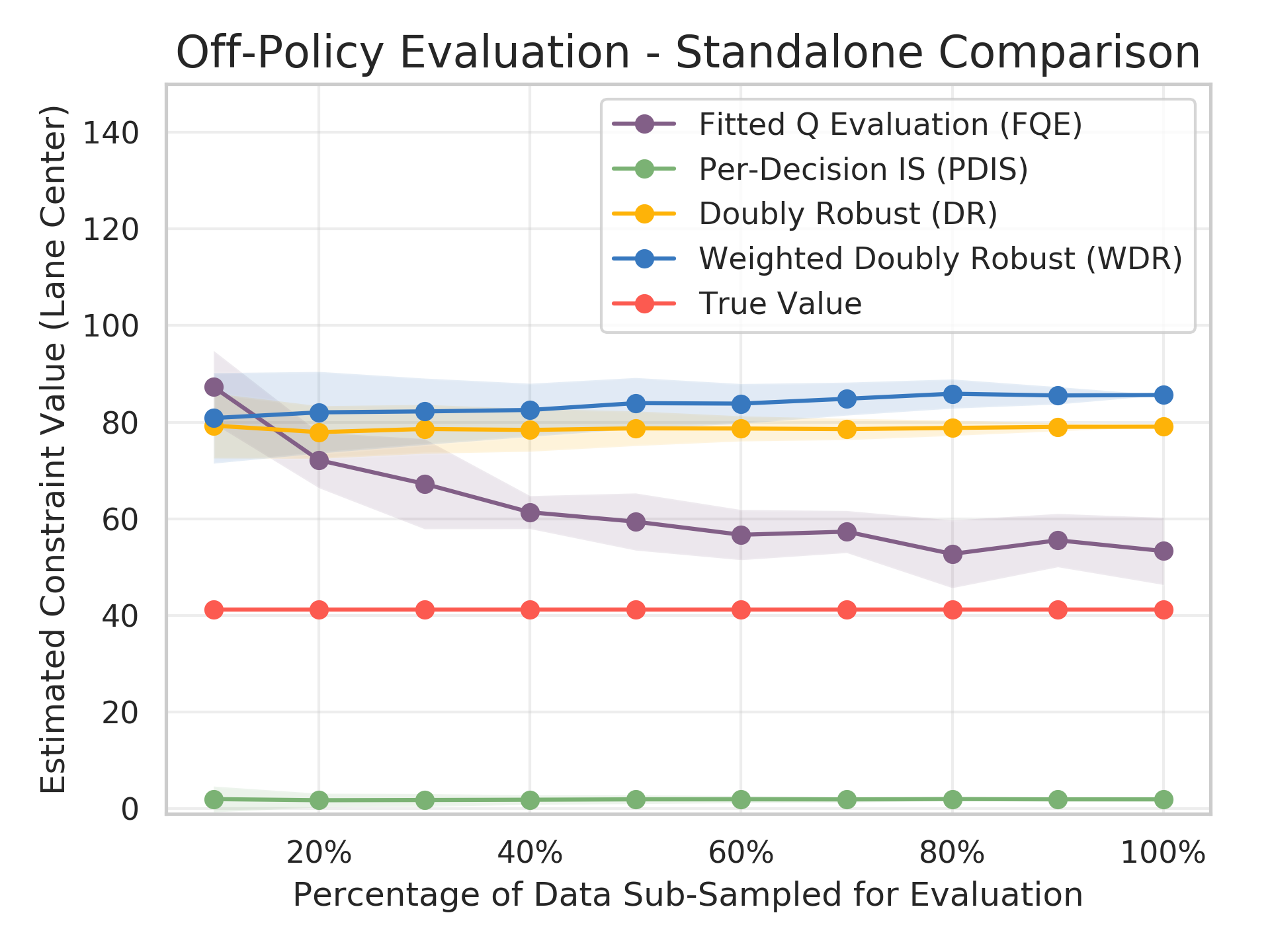}
        \vspace{-0.1in}
\caption{\emph{CarRacing Results}.
             \emph{(Left)} \& \emph{(Middle)} (Lower is better) Comparing our algorithm, regularized LSPI, online RL w/o constraints, behavior policy $\pi_D$ w.r.t. main cost objectives and two constraints.
             \emph{(Right)} FQE vs. other OPE methods on a standalone basis.
             }
\label{fig:car}
\end{figure*}

\subsection{Car Racing.}
\textbf{Environment \& Data Collection.} The car racing environment, seen in Figure~\ref{fig:car_screenshot}, is a high-dimensional domain where the state is a $96\times96\times3$ tensor of raw pixels. The action space $\A= \{\text{steering}\times\text{gas}\times\text{brake} \}$ takes 12 discretized values. 
The goal in this episodic task is to traverse over $95\%$ of the track, measured by a given number of \doubleQuote{tiles} as a proxy for distance coverage. The agent receives a reward (negative cost) for each unique tile crossed and no reward if the agent is off-track. A small positive cost applies at every time step, with maximum horizon of 1000 for each episode. With these costs given by the environment, one can train online RL agent using DDQN \cite{van2016deep}. We collect $\approx 1500$ trajectories from DDQN's randomization, resulting in data set $\D$ with $\approx 94000$ transition tuples. 


\textbf{Fast Driving under Behavioral Constraints.} We study the problem of minimizing environment cost while subject to two behavioral constraints: smooth driving and lane centering. The first constraint $G_0$ approximates smooth driving by $g_0(x,a) = 1$ if $a$ contains braking action, and $0$ otherwise. The second constraint cost $g_1$ measures the distance between the agent and center of lane at each time step. This is a highly challenging setup since three objectives and constraints are in direct conflict with one another, e.g., fast driving encourages the agent to cut corners and apply frequent brakes to make turns. Outside of this work, we are not aware of previous work in policy learning with 2 or more constraints in high-dimensional settings. 

\textbf{Baseline and Procedure.} As a na{\"i}ve baseline, DDQN achieves low cost, but exhibits \doubleQuote{non-smooth} driving behavior (see our supplementary videos). We set the threshold for each constraint to 75\% of the DDQN benchmark. We also compare against regularized batch RL algorithms \cite{farahmand2009regularized}, specifically regularized LSPI. We instantiate our subroutines, FQE and FQI, with multi-layered CNNs. We augment LSPI's linear policy with non-linear features derived from a well-performing FQI model. 

\textbf{Results.} The returned mixture policy from our algorithm achieves low main objective cost, comparable with online RL policy trained without regard to constraints. After several initial iterations violating the braking constraint, the returned policy - corresponding to the appropriate $\lambda$ trade-off - satisties both constraints, while improving the main objective. The improvement over data gathering policy is significant for both constraints and main objective. 

Regularized policy learning is an alternative approach to \eqref{eqn:main_problem} (section \ref{sec:problem}). We provide the regularized LSPI baseline the same set of $\lambda$ found by our algorithm for one-shot regularized learning (Figures \ref{fig:car_main_values}). While regularized LSPI obtains good performance for the main objective, it does not achieve acceptable constraint satisfaction. By default, regularized policy learning requires parameter tuning heuristics. In principle, one can perform a grid-search over a range of parameters to find the right combination - we include such an example for both regularized LSPI and FQI in Appendix \ref{sec:app-experiment}. However, since our objective and constraints are in conflict, main objective and constraint satisfaction of policies returned by one-shot regularized learning are sensitive to step changes in $\lambda$. In constrast, our approach is systematic, and is able to avoid the curse-of-dimensionality of brute-force search that comes with multiple constraints. 

In practice, one may wish to deterministically extract a single policy from the returned mixture for execution. A de-randomized policy can be obtained naturally from our algorithm by selecting the best policy from the existing FQE's estimates of individual \texttt{Best-response} policies. 

\subsection{Off-Policy Evaluation} 
The off-policy evaluation by FQE is critical for updating policies in our algorithm, and is ultimately responsible for certifying constraint satisfaction. While other OPE methods can also be used in place of FQE, we find that the estimates from popular methods are not sufficiently accurate in a high-dimensional setting. As a standalone comparison, we select an individual policy and compare FQE against PDIS \cite{precup2000eligibility}, DR \cite{jiang2016doubly} and WDR \cite{thomas2016data} with respect to the constraint cost evaluation. To compare both accuracy and data-efficiency, for each domain we randomly sample different subsets of dataset $\D$ (from 10\% to 100\% transitions, 30 trials each). Figure~\ref{fig:lake_fqe} and ~\ref{fig:car_fqe} illustrate the difference in quality. In the FrozenLake domain, FQE performs competitively with the top baseline method (DR and WDR), converging to the true value estimate as the data subsample grows close to 100\%. In the high-dimensional car domain, FQE signficantly outperforms other methods.

\section{Other Related Work}

\textbf{Constrained MDP (CMDP).} The CMDP is a well-studied problem \cite{altman1999constrained}. Among the most important techniques for solving CMDP are the Lagrangian approach and solving the dual LP program via occupation measure. However, these approaches only work when the MDP is completely specified, and the state dimension is small such that solving via an LP is tractable. More recently, the constrained policy optimization approach (CPO) by \cite{achiam2017constrained} learns a policy when the model is not initially known. The focus of CPO is on online safe exploration, and thus is not directly comparable to our setting.

\textbf{Multi-objective Reinforcement Learning.} Another related area is multi-objective reinforcement learning (MORL)\cite{van2014multi,roijers2013survey}. Generally, research in MORL has largely focused on approximating the Pareto frontier that trades-off competing objectives \cite{van2014multi,roijers2013survey}. The underlying approach to MORL frequently relies on linear or non-linear scalarization of rewards to heuristically turns the problem into a standard RL problem. Our proposed approach represents another systematic paradigm to solve MORL, whether in batch or online settings.  



\section{Discussion}
We have presented a systematic approach for batch policy learning under multiple constraints. Our problem formulation can accommodate general definition of constraints, as partly illustrated by our experiments. We provide guarantees for our algorithm for both the main objective and constraint satisfaction. Our strong empirical results show a promise of making constrained batch policy learning applicable for real-world domains, where behavior data is abundant. 

Our implementation complies with the steps laid out in Algorithm \ref{algo:main_algo}. In very large scale or high-dimensional problems, one could consider a noisy update version for both policy learning and evaluation. We leave the theorerical and practical exploration of this extension to future work. In addition, our proposed FQE algorithm for OPE problem achieves strong results, especially in a difficult domain with long horizons. Comparing the bias-variance characteristics of FQE with contemporary OPE methods is another interesting direction for research. 


\clearpage
\begin{small}
\nocite{bertsekas2011approximate,liu2018breaking,swaminathan2015batch,wang2017optimal}
\bibliography{icml_paper}

\begin{thebibliography}{68}
\providecommand{\natexlab}[1]{#1}
\providecommand{\url}[1]{\texttt{#1}}
\expandafter\ifx\csname urlstyle\endcsname\relax
  \providecommand{\doi}[1]{doi: #1}\else
  \providecommand{\doi}{doi: \begingroup \urlstyle{rm}\Url}\fi

\bibitem[Achiam et~al.(2017)Achiam, Held, Tamar, and
  Abbeel]{achiam2017constrained}
Achiam, J., Held, D., Tamar, A., and Abbeel, P.
\newblock Constrained policy optimization.
\newblock In \emph{International Conference on Machine Learning}, pp.\  22--31,
  2017.

\bibitem[Agarwal et~al.(2018)Agarwal, Beygelzimer, Dud{\'\i}k, Langford, and
  Wallach]{agarwal2018reductions}
Agarwal, A., Beygelzimer, A., Dud{\'\i}k, M., Langford, J., and Wallach, H.
\newblock A reductions approach to fair classification.
\newblock In \emph{International Conference on Machine Learning}, 2018.

\bibitem[Altman(1999)]{altman1999constrained}
Altman, E.
\newblock \emph{Constrained Markov decision processes}, volume~7.
\newblock CRC Press, 1999.

\bibitem[Antos et~al.(2008{\natexlab{a}})Antos, Szepesv{\'a}ri, and
  Munos]{antos2008fitted}
Antos, A., Szepesv{\'a}ri, C., and Munos, R.
\newblock Fitted q-iteration in continuous action-space mdps.
\newblock In \emph{Advances in neural information processing systems}, pp.\
  9--16, 2008{\natexlab{a}}.

\bibitem[Antos et~al.(2008{\natexlab{b}})Antos, Szepesv{\'a}ri, and
  Munos]{antos2008learning}
Antos, A., Szepesv{\'a}ri, C., and Munos, R.
\newblock Learning near-optimal policies with bellman-residual minimization
  based fitted policy iteration and a single sample path.
\newblock \emph{Machine Learning}, 71\penalty0 (1):\penalty0 89--129,
  2008{\natexlab{b}}.

\bibitem[Bartlett et~al.(2017)Bartlett, Harvey, Liaw, and
  Mehrabian]{bartlett2017nearly}
Bartlett, P.~L., Harvey, N., Liaw, C., and Mehrabian, A.
\newblock Nearly-tight vc-dimension bounds for piecewise linear neural
  networks.
\newblock In \emph{Proceedings of the 22nd Annual Conference on Learning Theory
  (COLT 2017)}, 2017.

\bibitem[Bertsekas(2011)]{bertsekas2011approximate}
Bertsekas, D.~P.
\newblock Approximate policy iteration: A survey and some new methods.
\newblock \emph{Journal of Control Theory and Applications}, 9\penalty0
  (3):\penalty0 310--335, 2011.

\bibitem[Blackmore et~al.(2011)Blackmore, Ono, and
  Williams]{blackmore2011chance}
Blackmore, L., Ono, M., and Williams, B.~C.
\newblock Chance-constrained optimal path planning with obstacles.
\newblock \emph{IEEE Transactions on Robotics}, 27\penalty0 (6):\penalty0
  1080--1094, 2011.

\bibitem[Bougerol \& Picard(1992)Bougerol and Picard]{bougerol1992strict}
Bougerol, P. and Picard, N.
\newblock Strict stationarity of generalized autoregressive processes.
\newblock \emph{The Annals of Probability}, pp.\  1714--1730, 1992.

\bibitem[Boyd \& Vandenberghe(2004)Boyd and Vandenberghe]{boyd2004convex}
Boyd, S. and Vandenberghe, L.
\newblock \emph{Convex optimization}.
\newblock Cambridge university press, 2004.

\bibitem[Cheng et~al.(2019)Cheng, Yan, Theodorou, and
  Boots]{cheng2019accelerating}
Cheng, C.-A., Yan, X., Theodorou, E., and Boots, B.
\newblock Accelerating imitation learning with predictive models.
\newblock In \emph{Conference on Artificial Intelligence and Statistics
  {(AISTATS)}}, 2019.

\bibitem[Dud{\'\i}k et~al.(2011)Dud{\'\i}k, Langford, and Li]{dudik2011doubly}
Dud{\'\i}k, M., Langford, J., and Li, L.
\newblock Doubly robust policy evaluation and learning.
\newblock In \emph{Proceedings of the 28th International Conference on
  International Conference on Machine Learning}, pp.\  1097--1104. Omnipress,
  2011.

\bibitem[Ernst et~al.(2005)Ernst, Geurts, and Wehenkel]{ernst2005tree}
Ernst, D., Geurts, P., and Wehenkel, L.
\newblock Tree-based batch mode reinforcement learning.
\newblock \emph{Journal of Machine Learning Research}, 6\penalty0
  (Apr):\penalty0 503--556, 2005.

\bibitem[Farahmand et~al.(2009)Farahmand, Ghavamzadeh, Mannor, and
  Szepesv{\'a}ri]{farahmand2009regularized}
Farahmand, A.~M., Ghavamzadeh, M., Mannor, S., and Szepesv{\'a}ri, C.
\newblock Regularized policy iteration.
\newblock In \emph{Advances in Neural Information Processing Systems}, pp.\
  441--448, 2009.

\bibitem[Farajtabar et~al.(2018)Farajtabar, Chow, and
  Ghavamzadeh]{farajtabar2018more}
Farajtabar, M., Chow, Y., and Ghavamzadeh, M.
\newblock More robust doubly robust off-policy evaluation.
\newblock \emph{arXiv preprint arXiv:1802.03493}, 2018.

\bibitem[Freund \& Schapire(1999)Freund and Schapire]{freund1999adaptive}
Freund, Y. and Schapire, R.~E.
\newblock Adaptive game playing using multiplicative weights.
\newblock \emph{Games and Economic Behavior}, 29:\penalty0 79--103, 1999.

\bibitem[Friedman et~al.(2001)Friedman, Hastie, and
  Tibshirani]{friedman2001elements}
Friedman, J., Hastie, T., and Tibshirani, R.
\newblock \emph{The elements of statistical learning}.
\newblock Springer, 2001.

\bibitem[Garc{\i}a \& Fern{\'a}ndez(2015)Garc{\i}a and
  Fern{\'a}ndez]{garcia2015comprehensive}
Garc{\i}a, J. and Fern{\'a}ndez, F.
\newblock A comprehensive survey on safe reinforcement learning.
\newblock \emph{Journal of Machine Learning Research}, 16\penalty0
  (1):\penalty0 1437--1480, 2015.

\bibitem[Guo et~al.(2017)Guo, Thomas, and Brunskill]{guo2017using}
Guo, Z., Thomas, P.~S., and Brunskill, E.
\newblock Using options and covariance testing for long horizon off-policy
  policy evaluation.
\newblock In \emph{Advances in Neural Information Processing Systems}, pp.\
  2492--2501, 2017.

\bibitem[Gy{\"o}rfi et~al.(2006)Gy{\"o}rfi, Kohler, Krzyzak, and
  Walk]{gyorfi2006distribution}
Gy{\"o}rfi, L., Kohler, M., Krzyzak, A., and Walk, H.
\newblock \emph{A distribution-free theory of nonparametric regression}.
\newblock Springer Science \& Business Media, 2006.

\bibitem[Ha \& Schmidhuber(2018)Ha and Schmidhuber]{ha2018world}
Ha, D. and Schmidhuber, J.
\newblock World models.
\newblock \emph{arXiv preprint arXiv:1803.10122}, 2018.

\bibitem[Haarnoja et~al.(2017)Haarnoja, Tang, Abbeel, and
  Levine]{haarnoja2017reinforcement}
Haarnoja, T., Tang, H., Abbeel, P., and Levine, S.
\newblock Reinforcement learning with deep energy-based policies.
\newblock In \emph{International Conference on Machine Learning}, pp.\
  1352--1361, 2017.

\bibitem[Haussler(1995)]{haussler1995sphere}
Haussler, D.
\newblock Sphere packing numbers for subsets of the boolean n-cube with bounded
  vapnik-chervonenkis dimension.
\newblock \emph{Journal of Combinatorial Theory, Series A}, 69\penalty0
  (2):\penalty0 217--232, 1995.

\bibitem[Henaff et~al.(2019)Henaff, Canziani, and LeCun]{henaff2019model}
Henaff, M., Canziani, A., and LeCun, Y.
\newblock Model-predictive policy learning with uncertainty regularization for
  driving in dense traffic.
\newblock \emph{arXiv preprint arXiv:1901.02705}, 2019.

\bibitem[Hester et~al.(2018)Hester, Vecerik, Pietquin, Lanctot, Schaul, Piot,
  Horgan, Quan, Sendonaris, Osband, et~al.]{hester2018deep}
Hester, T., Vecerik, M., Pietquin, O., Lanctot, M., Schaul, T., Piot, B.,
  Horgan, D., Quan, J., Sendonaris, A., Osband, I., et~al.
\newblock Deep q-learning from demonstrations.
\newblock In \emph{Thirty-Second AAAI Conference on Artificial Intelligence},
  2018.

\bibitem[Jiang \& Li(2016)Jiang and Li]{jiang2016doubly}
Jiang, N. and Li, L.
\newblock Doubly robust off-policy value evaluation for reinforcement learning.
\newblock In \emph{International Conference on Machine Learning}, pp.\
  652--661, 2016.

\bibitem[Kakade \& Langford(2002)Kakade and Langford]{kakade2002approximately}
Kakade, S. and Langford, J.
\newblock Approximately optimal approximate reinforcement learning.
\newblock In \emph{ICML}, volume~2, pp.\  267--274, 2002.

\bibitem[Kivinen \& Warmuth(1997)Kivinen and Warmuth]{kivinen1997exponentiated}
Kivinen, J. and Warmuth, M.~K.
\newblock Exponentiated gradient versus gradient descent for linear predictors.
\newblock \emph{Information and Computation}, 132\penalty0 (1):\penalty0 1--63,
  1997.

\bibitem[Koenig \& Simmons(1996)Koenig and Simmons]{koenig1996effect}
Koenig, S. and Simmons, R.~G.
\newblock The effect of representation and knowledge on goal-directed
  exploration with reinforcement-learning algorithms.
\newblock \emph{Machine Learning}, 22\penalty0 (1-3):\penalty0 227--250, 1996.

\bibitem[Lagoudakis \& Parr(2003)Lagoudakis and Parr]{lagoudakis2003least}
Lagoudakis, M.~G. and Parr, R.
\newblock Least-squares policy iteration.
\newblock \emph{Journal of machine learning research}, 4\penalty0
  (Dec):\penalty0 1107--1149, 2003.

\bibitem[Lange et~al.(2012)Lange, Gabel, and Riedmiller]{lange2012batch}
Lange, S., Gabel, T., and Riedmiller, M.
\newblock Batch reinforcement learning.
\newblock In \emph{Reinforcement learning}, pp.\  45--73. Springer, 2012.

\bibitem[Lazaric \& Restelli(2011)Lazaric and Restelli]{lazaric2011transfer}
Lazaric, A. and Restelli, M.
\newblock Transfer from multiple mdps.
\newblock In \emph{Advances in Neural Information Processing Systems}, pp.\
  1746--1754, 2011.

\bibitem[Lazaric et~al.(2010)Lazaric, Ghavamzadeh, and
  Munos]{lazaric2010finite}
Lazaric, A., Ghavamzadeh, M., and Munos, R.
\newblock Finite-sample analysis of lstd.
\newblock In \emph{ICML-27th International Conference on Machine Learning},
  pp.\  615--622, 2010.

\bibitem[Lazaric et~al.(2012)Lazaric, Ghavamzadeh, and
  Munos]{lazaric2012finite}
Lazaric, A., Ghavamzadeh, M., and Munos, R.
\newblock Finite-sample analysis of least-squares policy iteration.
\newblock \emph{Journal of Machine Learning Research}, 13\penalty0
  (Oct):\penalty0 3041--3074, 2012.

\bibitem[Le et~al.(2016)Le, Kang, Yue, and Carr]{le2016smooth}
Le, H.~M., Kang, A., Yue, Y., and Carr, P.
\newblock Smooth imitation learning for online sequence prediction.
\newblock In \emph{Proceedings of the 33rd International Conference on
  International Conference on Machine Learning-Volume 48}, pp.\  680--688.
  JMLR. org, 2016.

\bibitem[Lee et~al.(1996)Lee, Bartlett, and Williamson]{lee1996efficient}
Lee, W.~S., Bartlett, P.~L., and Williamson, R.~C.
\newblock Efficient agnostic learning of neural networks with bounded fan-in.
\newblock \emph{IEEE Transactions on Information Theory}, 42\penalty0
  (6):\penalty0 2118--2132, 1996.

\bibitem[Levine \& Abbeel(2014)Levine and Abbeel]{levine2014learning}
Levine, S. and Abbeel, P.
\newblock Learning neural network policies with guided policy search under
  unknown dynamics.
\newblock In \emph{Advances in Neural Information Processing Systems}, pp.\
  1071--1079, 2014.

\bibitem[Lillicrap et~al.(2016)Lillicrap, Hunt, Pritzel, Heess, Erez, Tassa,
  Silver, and Wierstra]{lillicrap2016continuous}
Lillicrap, T.~P., Hunt, J.~J., Pritzel, A., Heess, N., Erez, T., Tassa, Y.,
  Silver, D., and Wierstra, D.
\newblock Continuous control with deep reinforcement learning.
\newblock In \emph{International Conference on Learning Representations
  {(ICLR)}}, 2016.

\bibitem[Liu et~al.(2018)Liu, Li, Tang, and Zhou]{liu2018breaking}
Liu, Q., Li, L., Tang, Z., and Zhou, D.
\newblock Breaking the curse of horizon: Infinite-horizon off-policy
  estimation.
\newblock In \emph{Advances in Neural Information Processing Systems}, pp.\
  5361--5371, 2018.

\bibitem[Maillard et~al.(2010)Maillard, Munos, Lazaric, and
  Ghavamzadeh]{maillard2010finite}
Maillard, O.-A., Munos, R., Lazaric, A., and Ghavamzadeh, M.
\newblock Finite-sample analysis of bellman residual minimization.
\newblock In \emph{Proceedings of 2nd Asian Conference on Machine Learning},
  pp.\  299--314, 2010.

\bibitem[Mohri et~al.(2012)Mohri, Rostamizadeh, and
  Talwalkar]{mohri2012foundations}
Mohri, M., Rostamizadeh, A., and Talwalkar, A.
\newblock \emph{Foundations of machine learning}.
\newblock MIT press, 2012.

\bibitem[Montgomery \& Levine(2016)Montgomery and Levine]{montgomery2016guided}
Montgomery, W.~H. and Levine, S.
\newblock Guided policy search via approximate mirror descent.
\newblock In \emph{Advances in Neural Information Processing Systems}, pp.\
  4008--4016, 2016.

\bibitem[Munos(2003)]{munos2003error}
Munos, R.
\newblock Error bounds for approximate policy iteration.
\newblock In \emph{ICML}, volume~3, pp.\  560--567, 2003.

\bibitem[Munos(2007)]{munos2007performance}
Munos, R.
\newblock Performance bounds in l\_p-norm for approximate value iteration.
\newblock \emph{SIAM journal on control and optimization}, 46\penalty0
  (2):\penalty0 541--561, 2007.

\bibitem[Munos \& Szepesv{\'a}ri(2008)Munos and
  Szepesv{\'a}ri]{munos2008finite}
Munos, R. and Szepesv{\'a}ri, C.
\newblock Finite-time bounds for fitted value iteration.
\newblock \emph{Journal of Machine Learning Research}, 9\penalty0
  (May):\penalty0 815--857, 2008.

\bibitem[Nemirovsky \& Yudin(1983)Nemirovsky and Yudin]{nemirovsky1983problem}
Nemirovsky, A.~S. and Yudin, D.~B.
\newblock \emph{Problem complexity and method efficiency in optimization.}
\newblock Wiley, 1983.

\bibitem[Oh et~al.(2018)Oh, Guo, Singh, and Lee]{oh2018self}
Oh, J., Guo, Y., Singh, S., and Lee, H.
\newblock Self-imitation learning.
\newblock In \emph{International Conference on Machine Learning}, 2018.

\bibitem[Ono et~al.(2015)Ono, Pavone, Kuwata, and Balaram]{ono2015chance}
Ono, M., Pavone, M., Kuwata, Y., and Balaram, J.
\newblock Chance-constrained dynamic programming with application to risk-aware
  robotic space exploration.
\newblock \emph{Autonomous Robots}, 39\penalty0 (4):\penalty0 555--571, 2015.

\bibitem[Ormoneit \& Sen(2002)Ormoneit and Sen]{ormoneit2002kernel}
Ormoneit, D. and Sen, {\'S}.
\newblock Kernel-based reinforcement learning.
\newblock \emph{Machine learning}, 49\penalty0 (2-3):\penalty0 161--178, 2002.

\bibitem[Pietquin et~al.(2011)Pietquin, Geist, Chandramohan, and
  Frezza-Buet]{pietquin2011sample}
Pietquin, O., Geist, M., Chandramohan, S., and Frezza-Buet, H.
\newblock Sample-efficient batch reinforcement learning for dialogue management
  optimization.
\newblock \emph{ACM Transactions on Speech and Language Processing (TSLP)},
  7\penalty0 (3):\penalty0 7, 2011.

\bibitem[Precup et~al.(2000)Precup, Sutton, and Singh]{precup2000eligibility}
Precup, D., Sutton, R.~S., and Singh, S.~P.
\newblock Eligibility traces for off-policy policy evaluation.
\newblock In \emph{Proceedings of the Seventeenth International Conference on
  Machine Learning}, pp.\  759--766. Morgan Kaufmann Publishers Inc., 2000.

\bibitem[Precup et~al.(2001)Precup, Sutton, and Dasgupta]{precup2001off}
Precup, D., Sutton, R.~S., and Dasgupta, S.
\newblock Off-policy temporal difference learning with function approximation.
\newblock In \emph{Proceedings of the Eighteenth International Conference on
  Machine Learning}, pp.\  417--424. Morgan Kaufmann Publishers Inc., 2001.

\bibitem[Riedmiller(2005)]{riedmiller2005neural}
Riedmiller, M.
\newblock Neural fitted q iteration--first experiences with a data efficient
  neural reinforcement learning method.
\newblock In \emph{European Conference on Machine Learning}, pp.\  317--328.
  Springer, 2005.

\bibitem[Riedmiller et~al.(2009)Riedmiller, Gabel, Hafner, and
  Lange]{riedmiller2009reinforcement}
Riedmiller, M., Gabel, T., Hafner, R., and Lange, S.
\newblock Reinforcement learning for robot soccer.
\newblock \emph{Autonomous Robots}, 27\penalty0 (1):\penalty0 55--73, 2009.

\bibitem[Roijers et~al.(2013)Roijers, Vamplew, Whiteson, and
  Dazeley]{roijers2013survey}
Roijers, D.~M., Vamplew, P., Whiteson, S., and Dazeley, R.
\newblock A survey of multi-objective sequential decision-making.
\newblock \emph{Journal of Artificial Intelligence Research}, 48:\penalty0
  67--113, 2013.

\bibitem[Ross \& Bagnell(2014)Ross and Bagnell]{ross2014reinforcement}
Ross, S. and Bagnell, J.~A.
\newblock Reinforcement and imitation learning via interactive no-regret
  learning.
\newblock \emph{arXiv preprint arXiv:1406.5979}, 2014.

\bibitem[Schulman et~al.(2015)Schulman, Levine, Abbeel, Jordan, and
  Moritz]{schulman2015trust}
Schulman, J., Levine, S., Abbeel, P., Jordan, M., and Moritz, P.
\newblock Trust region policy optimization.
\newblock In \emph{International Conference on Machine Learning}, pp.\
  1889--1897, 2015.

\bibitem[Shalev-Shwartz et~al.(2012)]{shalev2012online}
Shalev-Shwartz, S. et~al.
\newblock Online learning and online convex optimization.
\newblock \emph{Foundations and Trends{\textregistered} in Machine Learning},
  4\penalty0 (2):\penalty0 107--194, 2012.

\bibitem[Sutton \& Barto(2018)Sutton and Barto]{sutton2018reinforcement}
Sutton, R.~S. and Barto, A.~G.
\newblock \emph{Reinforcement learning: An introduction}.
\newblock MIT press, 2018.

\bibitem[Swaminathan \& Joachims(2015)Swaminathan and
  Joachims]{swaminathan2015batch}
Swaminathan, A. and Joachims, T.
\newblock Batch learning from logged bandit feedback through counterfactual
  risk minimization.
\newblock \emph{Journal of Machine Learning Research}, 16\penalty0
  (1):\penalty0 1731--1755, 2015.

\bibitem[Thomas \& Brunskill(2016)Thomas and Brunskill]{thomas2016data}
Thomas, P. and Brunskill, E.
\newblock Data-efficient off-policy policy evaluation for reinforcement
  learning.
\newblock In \emph{International Conference on Machine Learning}, pp.\
  2139--2148, 2016.

\bibitem[Van~Hasselt et~al.(2016)Van~Hasselt, Guez, and Silver]{van2016deep}
Van~Hasselt, H., Guez, A., and Silver, D.
\newblock Deep reinforcement learning with double q-learning.
\newblock In \emph{AAAI}, volume~2, pp.\ ~5. Phoenix, AZ, 2016.

\bibitem[Van~Moffaert \& Now{\'e}(2014)Van~Moffaert and Now{\'e}]{van2014multi}
Van~Moffaert, K. and Now{\'e}, A.
\newblock Multi-objective reinforcement learning using sets of pareto
  dominating policies.
\newblock \emph{The Journal of Machine Learning Research}, 15\penalty0
  (1):\penalty0 3483--3512, 2014.

\bibitem[Von~Neumann \& Morgenstern(2007)Von~Neumann and
  Morgenstern]{von2007theory}
Von~Neumann, J. and Morgenstern, O.
\newblock \emph{Theory of games and economic behavior (commemorative edition)}.
\newblock Princeton university press, 2007.

\bibitem[Wang et~al.(2017)Wang, Agarwal, and Dud\'{\i}k]{wang2017optimal}
Wang, Y.-X., Agarwal, A., and Dud\'{\i}k, M.
\newblock Optimal and adaptive off-policy evaluation in contextual bandits.
\newblock In \emph{International Conference on Machine Learning}, pp.\
  3589--3597, 2017.

\bibitem[Ziebart(2010)]{ziebart2010modeling}
Ziebart, B.~D.
\newblock \emph{Modeling purposeful adaptive behavior with the principle of
  maximum causal entropy}.
\newblock PhD thesis, CMU, 2010.

\bibitem[Ziebart et~al.(2008)Ziebart, Maas, Bagnell, and
  Dey]{ziebart2008maximum}
Ziebart, B.~D., Maas, A.~L., Bagnell, J.~A., and Dey, A.~K.
\newblock Maximum entropy inverse reinforcement learning.
\newblock In \emph{AAAI}, volume~8, pp.\  1433--1438. Chicago, IL, USA, 2008.

\bibitem[Zinkevich(2003)]{zinkevich2003online}
Zinkevich, M.
\newblock Online convex programming and generalized infinitesimal gradient
  ascent.
\newblock In \emph{Proceedings of the 20th International Conference on Machine
  Learning (ICML-03)}, pp.\  928--936, 2003.

\end{thebibliography}
\bibliographystyle{icml2019}
\end{small}

\clearpage
\appendix
\onecolumn
\section{Equivalence between Regularization and Constraint Satisfaction}
\label{sec:app-regularization}
\subsection{Formulating Different Regularized Policy Learning Problems as Constrained Policy Learning}
In this section, we provide connections between regularized policy learning and our constrained formulation \eqref{eqn:main_problem}. Although the main paper focuses on batch policy learning, here we are agnostic between online and batch learning settings. 

\textbf{Entropy regularized RL.} The standard reinforcement learning objective, either in online or batch setting, is to find a policy $\pi^*_{\text{std}}$ that minimizes the long-term cost (equivalent to maximizing the accumuted rewards): $\pi^*_{\text{std}} = \argmin_{\pi}\sum_{t} \E_{(x_t,a_t)\sim\pi}[c(x_t,a_t)] = \argmin_{\pi}\E_{(x,a)\sim\mu_\pi}[c(x,a)]$. Maximum entropy reinforcement learning \cite{haarnoja2017reinforcement} augments the cost with an entropy term, such that the optimal policy maximizes its entropy at each visited state: $\pi^*_{\text{MaxEnt}} = \argmin_{\pi} \E_{(x,a)\sim\mu_\pi}[c(x,a)]-\lambda\mathbb{H}(\pi(\cdot|x))$. As discuseed by \cite{haarnoja2017reinforcement}, the goal is for the agent to maximize the entropy of the entire trajectory, and not greedily maximizing entropy at the current time step (i.e., Boltzmann exploration). Maximum entropy policy learning was first proposed by \cite{ziebart2008maximum,ziebart2010modeling} in the context of learning from expert demonstrations. Entropy regulazed RL/IL is equivalent to our problem \eqref{eqn:main_problem} by simply set $C(\pi) = \E_{(x_t,a_t)\sim\pi}[c(x_t,a_t)]$ (standard RL objective), and $g(x,a) = \pi(a|x)\log\pi(a|x)$, thus $G(\pi) = -\mathbb{H}(\pi) \leq \tau$

\textbf{Smooth imitation learning (\& Regularized imitation learning).} This is a constrained imitation learning problem studied by \cite{le2016smooth}: learning to mimic smooth behavior in continuous space from human desmonstrations. The data collected from human demonstrations is considered to be fixed and given a priori, thus the imitation learning task is also a batch policy learning problem. The proposed approach from \cite{le2016smooth} is to view policy learning as a function regularization problem: policy $\pi = (f,g)$ is a combination of functions $f$ and $h$, where $f$ belongs to some expressive function class $\mathrm{F}$ (e.g., decision trees, neural networks) and $h\in\mathrm{H}$ with certifiable smoothness property (e.g., linear models). Policy learning is the solution to the functional regularization problem $\pi = \argmin_{f,g}\E_{x\sim\mu_\pi} \norm{f(x)-\pi_E(x)} +\lambda \norm{h(x) - \pi_E(x)}$, where $\pi_E$ is the expert policy. This constrained imitation learning setting is equivalent to our problem \eqref{eqn:main_problem} as follows: $C(\pi) = C((f,h)) = \E_{x\sim\mu_\pi}\norm{f(x) - \pi_E(x)}$ and $G(\pi) = G((f,h))=\min_{h^\prime\in\mathrm{H}}\norm{h^\prime(x)-\pi_E(x)} \leq \tau$

\textbf{Regularizing RL with expert demonstrations / Learning from imperfect demonstrations.} Efficient exploration in RL is a well-known challenge. Expert demonstrations provide a way around online exploration to reduce the sample complexity for learning. However, the label budget for expert demonstrations may be limited, resulting in a sparse coverage of the state space compared to what the online RL agent can explore \cite{hester2018deep}. Additionally, expert demonstrations may be imperfect \cite{oh2018self}. Some recent work proposed to regularize standard RL objective with some deviation measure between the learning policy and (sparse) expert data \cite{hester2018deep,oh2018self,henaff2019model}. 

For clarity we focus on the regularized RL objective for Q-learning in \cite{hester2018deep}, which is defined as $J(\pi) = J_{DQ}(Q)+\lambda_1 J_n(Q)+\lambda_2 J_E(Q) + \lambda_3 J_{L2}(Q)$, where $J_{DQ}(Q)$ is the standard deep Q-learning loss, $J_n(Q)$ is the n-step return loss, $J_E(Q)$ is the imitation learning loss defined as $J_E(Q) = \max_{a\in\mathrm{A}}\left[ Q(x,a)+\ell(a_E,a) - Q(x,a_E)\right]$, and $J_{L2}(Q)$ is an L2 regularization loss applied to the Q-network to prevent overfitting to a small expert dataset. The regularization parameters $\lambda$'s are obtained by hyperparameter tuning. This approach provides a bridge between RL and IL, whose objective functions are fundamentally different (see AggreVate by \cite{ross2014reinforcement} for an alternative approach). 

We can cast this problem into \eqref{eqn:main_problem} as: $C(\pi) = C_{DQ}(Q) + \lambda_3 C_{L2}(Q)$ (standard RL objective), and two constraints: $g_1(\pi) = \E_{x\sim\mu_\pi}[\max_{a\in\mathrm{A}} Q(x,a) +\ell(a_E,a) - Q(x,a_E)]$, and $g_2(x,a) = \E_{x\sim\mu_\pi}[c_t+\gamma c_{t+1}+\ldots+\gamma^{n-1}c_{t+n-1}+\min_a^\prime \gamma^n Q(x_{t+n},a^\prime) - Q(x_t,a)]$. Here $g_1$ captures the loss w.r.t. expert demonstrations and $g_2$ reflects the n-step return constraint.

More generally, one can define the imitation learning constraint as $G(\pi) = \E_{x\sim\mu_\pi}\ell(\pi(x),\pi_E(x))$ for an appropriate divergence definition between $\pi(x)$ and $\pi_E(x)$ (defined at states where expert demonstrations are available).

\textbf{Conservative policy improvement.} Many policy search algorithms perform small policy update steps, requiring the new policy $\pi$ to stay within a neighborhood of the most recent policy iterate $\pi_k$ to ensure learning stability \cite{levine2014learning,schulman2015trust,montgomery2016guided,achiam2017constrained}. This simply corresponds to the definition of $G(\pi) = \texttt{distance}(\pi,\pi_k)\leq \tau$, where \texttt{distance} is typically $KL$-divergence or total variation distance between the distribution induced by $\pi$ and $\pi_k$. For $KL$-divergence, the single timestep cost $g(x,a) = -\pi(a|x)\log(\frac{\pi_k(a|x)}{\pi(a|x)})$
\subsection{Equivalence of Regularization and Constraint Viewpoint - Proof of Proposition \ref{prop:equivalence}}
 \texttt{Regularization}$\implies$\texttt{Constraint}: Let $\lambda>0$ and $\pi^*$ be optimal policy in $\texttt{Regularization}$. Set $\tau = G(\pi^*)$. Suppose that $\pi^*$ is not optimal in \texttt{Constraint}. Then $\exists \pi\in\Pi$ such that $G(\pi)\leq\tau$ and $C(\pi)<C(\pi^*)$. We then have
 $$C(\pi)+\lambda^\top G(\pi) < C(\pi^*)+\lambda^\top \tau = C(\pi^*)+\lambda^\top G(\pi^*)$$
 which contradicts the optimality of $\pi^*$ for \texttt{Regularization} problem. Thus $\pi^*$ is also the optimal solution of the \texttt{Constraint} problem.
 
 \texttt{Constraint}$\implies$\texttt{Regularization}: Given $\tau$ and let $\pi^*$ be the corresponding optimal solution of the \texttt{Constraint} problem. The Lagrangian of \texttt{Constraint} is given by $L(\pi,\lambda) = C(\pi)+\lambda^\top G(\pi), \lambda \geq 0$. We then have $\pi^* = \argmin\limits_{\pi\in\Pi} \max\limits_{\lambda\geq 0} L(\pi,\lambda)$. Let $$\lambda^* = \argmax\limits_{\lambda\geq 0}\min\limits_{\pi\in\Pi} L(\pi,\lambda)$$ Slater's condition implies strong duality. By strong duality and the strong max-min property \cite{boyd2004convex}, we can exchange the order of maximization and minimization. Thus $\pi^*$ is the optimal solution of $$\min_{\pi\in\Pi} \quad C(\pi) + (\lambda^*)^\top(G(\pi) - \tau)$$Removing the constaint $(\lambda^*)^\top \tau$, we have that $\pi^*$ is the optimal solution of the \texttt{Regularization} problem with $\lambda = \lambda^*$. And since $\pi^* \neq \argmin\limits_{\pi\in\Pi} C(\pi)$, we must have $\lambda^*\geq 0$.
\clearpage
\section{Convergence Proofs}
\label{sec:app-convergence}
\subsection{Convergence of Meta-algorithm - Proof of Proposition \ref{prop:convergence}}

Let us evaluate the empirical primal-dual gap of the Lagrangian after $T$ iterations:
\begin{align}
    \max_{\lambda} L(\pih_T, \lambda) &= \max_\lambda \frac{1}{T}\sum_t L(\pi_t,\lambda) \label{eqn:pi_hat} \\
    &\leq \frac{1}{T} \sum_t L(\pi_t,\lambda_t) + \frac{o(T)}{T} \label{eqn:no_regret_prop} \\
    &\leq \frac{1}{T} \sum_t L(\pi, \lambda_t) + \frac{o(T)}{T} \quad \forall \pi\in\Pi \label{eqn:best_response_prop} \\
    &= L(\pi, \widehat{\lambda}_T) + \frac{o(T)}{T} \quad \forall \pi \label{eqn:lambda_hat}
\end{align}
Equations (\ref{eqn:pi_hat}) and (\ref{eqn:lambda_hat}) are due to the definition of $\pih_T$ and $\widehat{\lambda}_T$ and linearity of $L(\pi,\lambda)$ wrt $\lambda$ and the distribution over policies in $\Pi$. Equation (\ref{eqn:no_regret_prop}) is due to the no-regret property of \texttt{Online-algorithm}. Equation (\ref{eqn:best_response_prop}) is true since $\pi_t$ is best response wrt $\lambda_t$. Since equation (\ref{eqn:lambda_hat}) holds for all $\pi$, we can conclude that for $T$ sufficiently large such that $\frac{o(T)}{T}\leq \omega$, we have $\max_{\lambda} L(\pih_T, \lambda) \leq \min_{\pi} L(\pi, \widehat{\lambda}_T) + \omega$
, which will terminate the algorithm. 

Note that we always have $\max_{\lambda} L(\pih_T, \lambda) \geq L(\pih_T,\widehat{\lambda}_T) \geq \min_{\pi} L(\pi, \widehat{\lambda}_T)$. Algorithm \ref{algo:meta}'s convergence rate depends on the regret bound of the \texttt{Online-algorithm} procedure. Multiple algorithms exist with regret scaling as $\Omega(\sqrt{T})$ (e.g., online gradient descent with regularizer, variants of online mirror descent). In that case, the algorithm will terminate after $O(\frac{1}{\omega^2})$ iterations.

\subsection{Empirical Convergence Analysis of Main Algorithm - Proof of Theorem \ref{thm:convergence_main}}
By choosing normalized exponentiated gradient as the online learning subroutine, we have the following regret bound after $T$ iterations of the main algorithm \ref{algo:main_algo} (chapter 2 of \cite{shalev2012online}) for any $\lambda\in\R_{+}^{m+1},\norm{\lambda}_1=B$:
\begin{equation}
\label{eqn:converge_no_regret}
    \frac{1}{T}\sum_{t=1}^T \Lh(\pi_t,\lambda) \leq \frac{1}{T}\sum_{t=1}^T \Lh(\pi_t,\lambda_t) +\frac{\frac{B\log(m+1)}{\eta}+\eta \widebar{G}^2BT}{T}
\end{equation}
Denote $\omega_T = \frac{\frac{B\log(m+1)}{\eta}+\eta \widebar{G}^2BT}{T}$ to simplify notations. By the linearity of $\Lh(\pi,\lambda)$ in both $\pi$ and $\lambda$, we have for any $\lambda$ that
\begin{equation*}
    \Lh(\pih_T,\lambda) \stackrel{\text{linearity}}{=} \frac{1}{T}\sum_{t=1}^T \Lh(\pi_t,\lambda) \stackrel{\text{eqn }(\ref{eqn:converge_no_regret})}{\leq} \frac{1}{T}\sum_{t=1}^T \Lh(\pi_t,\lambda_t) + \omega_T \stackrel{\text{best response }\pi_t}{\leq} \frac{1}{T}\sum_{t=1}^T \Lh(\pih_T,\lambda_t) + \omega_T \stackrel{\text{linearity}}{=}\Lh(\pih_T,\lah_T) + \omega_T
\end{equation*}
Since this is true for any $\lambda$, $\max_\lambda\Lh(\pih_T,\lambda)\leq \Lh(\pih_T,\lah_T) + \omega_T$. 

On the other hand, for any policy $\pi$, we also have
\begin{equation*}
    \Lh(\pi,\lah_T) \stackrel{\text{linearity}}{=} \frac{1}{T}\sum_{t=1}^T \Lh(\pi,\lambda_t)\stackrel{\text{best response } \pi_t}{\geq} \frac{1}{T}\sum_{t=1}^T \Lh(\pi_t,\lambda_t)
    \stackrel{\text{eqn } (\ref{eqn:converge_no_regret})}{\geq} \frac{1}{T}\sum_{t=1}^T \Lh(\pi_t,\lah_T)-\omega_T
    \stackrel{\text{linearity}}{=}\Lh(\pih_T,\lah_T)-\omega_T
\end{equation*}
Thus $\min_\pi\Lh(\pi,\lah_T) \geq \Lh(\pih_T,\lah_T)-\omega_T$, leading to
$$\max_\lambda\Lh(\pih_T,\lambda) - \min_\pi\Lh(\pi,\lah_T) \leq \Lh(\pih_T,\lah_T) + \omega_T - (\Lh(\pih_T,\lah_T)-\omega_T) = 2\omega_T$$
After $T$ iterations of the main algorithm \ref{algo:main_algo}, therefore, the empirical primal-dual gap is bounded by 
$$\max_\lambda\Lh(\pih_T,\lambda) - \min_\pi\Lh(\pi,\lah_T) \leq \frac{2\frac{B\log(m+1)}{\eta}+2\eta \widebar{G}^2BT}{T}$$
In particular, if we want the gap to fall below a desired threshold $\omega$, setting the online learning rate $\eta = \frac{\omega}{4\widebar{G}^2B}$ will ensure that the algorithm converges after $\frac{16B^2 \widebar{G}^2\log(m+1)}{\omega^2}$ iterations.

\clearpage
\section{End-to-end Generalization Analysis of Main Algorithm}
\label{sec:app-generalization}
In this section, we prove the following full statement of theorem \ref{thm:end_to_end_main} of the main paper. Note that to lessen notation, we define $\widebar{V} = \widebar{C}+B\widebar{G}$ to be the bound of value functions under considerations in algorithm \ref{algo:main_algo}.
\begin{thm}[Generalization bound of algorithm \ref{algo:main_algo}]\label{thm:end_to_end_appendix}
Let $\pi^*$ be the optimal policy to problem \ref{eqn:main_problem}. Let $K$ be the number of iterations of FQE and FQI. Let $\pih$ be the policy returned by our main algorithm \ref{algo:main_algo}, with termination threshold $\omega$. For any $\epsilon>0, \delta\in(0,1)$, when $n\geq\frac{24\cdot214\cdot \widebar{V}^4}{\epsilon^2}\big( \log\frac{K(m+1)}{\delta}+\pdim_{\F}\log\frac{320 \widebar{V}^2}{\epsilon^2}+\log(14e(\pdim_{\F}+1))\big)$, we have with probability at least $1-\delta$:
$$C(\pih)\leq C(\pi^*) + \omega + \frac{(4+B)\gamma}{(1-\gamma)^3}\big( \sqrt{C_{\rho}}\epsilon + 2\gamma^{K/2} \widebar{V}\big)$$
and
$$G(\pih)\leq \tau + 2\frac{\widebar{V}+\omega}{B} +\frac{\gamma^{1/2}}{(1-\gamma)^{3/2}} \big( \sqrt{C_{\rho}}\epsilon + \frac{2\gamma^{K/2}\widebar{V}}{(1-\gamma)^{1/2}}\big)$$
\end{thm}

Let $\pih= \frac{1}{T}\sum_t \pi_t$ be the returned policy $T$ iterations, with corresponding dual variable $\lah= \frac{1}{T}\sum_t \lambda_t$. 

By the stopping condition, the empirical duality gap is less than some threshold $\omega$, i.e., $\max\limits_{\lambda\in\R^{m+1}_{+}, \norm{\lambda}_1=B} \widehat{L}(\pih, \lambda) - \min\limits_{\pi\in\Pi} \widehat{L}(\pi,\lah) \leq\omega$
where $\widehat{L}(\pi,\lambda) = \widehat{C}(\pi)+\lambda^\top(\widehat{G}(\pi)-\tau)$. We first show that the returned policy approximately satisfies the constraints.
The proof of theorem \ref{thm:end_to_end_appendix} will make use of the following empirical constraint satisfaction bound:
\begin{lem}[Empirical constraint satisfactions] Assume that the constraints $\Gh(\pi) \leq \tau$ are feasible. Then the returned policy $\pih$ approximately satisfies all constraints
$$\max\limits_{i=1:m+1}\left( \widehat{g}_i(\pih) -\tau_i \right) \leq 2\frac{\widebar{C}+\omega}{B}$$
\label{lem:empirical_constraint_satisfaction} 
\end{lem}
\begin{proof}
We consider $\max\limits_{i=1:m+1}\left( \widehat{g}_i(\pih) -\tau_i \right) >0$ (otherwise the lemma statement is trivially true). The termination condition implies that 
$\widehat{L}(\pih,\lah) - \max\limits_{\lambda\in\R^{m+1}_{+}, \norm{\lambda}_1=B} \widehat{L}(\pih, \lambda) \geq -\omega$
\begin{equation}
\label{eqn:empirical_bound_2}
    \implies\lah^\top(\Gh(\pih)-\widehat{\tau})\geq\max\limits_{\lambda\in\R^{m+1}_{+}, \norm{\lambda}_1=B}\lambda^\top(\Gh(\pih)-\widehat{\tau}) -\omega
\end{equation}

Relaxing the RHS of equation (\ref{eqn:empirical_bound_2}) by setting $\lambda[j] = B$ for $j = \argmax\limits_{i=1: m+1}\left[\widehat{g}_i(\pih)-\tau_i\right]$ and $\lambda[i] = 0\enskip\forall i\neq j$ yields:
\begin{equation}
\label{eqn:empirical_bound_4}
    B\max\limits_{i=1:m+1}\left[ \widehat{g}_i(\pih) - \tau_i \right] -\omega \leq \lah^\top(\Gh(\pih)-\tau)
\end{equation}
Given $\pi$ such that $\widehat{G}(\pi)\leq\tau$, also by the termination condition: 
$$\widehat{L}(\pih,\lah)-\widehat{L}(\pi,\lah) \leq \max\limits_{\lambda\in\R^{m+1}_{+}, \norm{\lambda}_1=B} \widehat{L}(\pih, \lambda) - \min\limits_{\pi\in\Pi} \widehat{L}(\pi,\lah) \leq\omega$$
Thus implies
\begin{equation}
\label{eqn:empirical_bound_1}
    \widehat{L}(\pih,\lah) \leq \widehat{L}(\pi,\lah) +\omega= \Ch(\pi)+\lah^\top(\Gh(\pi)-\tau) \leq \Ch(\pi) + \omega
\end{equation}
combining what we have from equation (\ref{eqn:empirical_bound_1}) and (\ref{eqn:empirical_bound_4}):
\begin{equation*}
    B\max\limits_{i=1:m+1}\left[ \widehat{g}_i(\pih) - \widehat{\tau}_i \right] -\omega \leq \lah^\top(\Gh(\pih)-\widehat{\tau}) = \widehat{L}(\pih,\lah) - \Ch(\pih) \leq \Ch(\pi) + \omega - \Ch(\pih)
\end{equation*}
Rearranging and bounding $\Ch(\pi)\leq\widebar{C}$ and $\Ch(\pih)\leq-\widebar{C}$ finishes the proof of the lemma.
\end{proof}
We now return to the proof of theorem \ref{thm:end_to_end_appendix}, our task is to lift empirical error to generalization bound for main objective and constraints.

Denote by $\epsilon_{FQE}$ the (generalization) error introduced by the Fitted Q Evaluation procedure (algorithm \ref{algo:fqe}) and $\epsilon_{FQI}$ the (generalization) error introduced by the Fitted Q Iteration procedure (algorithm \ref{algo:fqi}). For now we keep $\epsilon_{FQE}$ and $\epsilon_{FQI}$ unspecified (to be specified shortly). That is, for each $t=1,2,\ldots,T$, we have with probability at least $1-\delta$: 
$$C(\pi_t)+\lambda_t^\top(G(\pi_t)-\tau)\leq C(\pi^*)+\lambda_t^\top(G(\pi^*)-\tau) + \epsilon_{FQI}$$
Since $\pi^*$ satisfies the constraints, i.e., $G(\pi^*)-\tau\leq 0$ componentwise, and $\lambda_t\geq 0$, we also have with probability $1-\delta$ 
\begin{equation}
\label{eqn:error_analysis_1}
  L(\pi_t,\lambda_t) = C(\pi_t)+\lambda_t^\top(G(\pi_t)-\tau)\leq C(\pi^*)+ \epsilon_{FQI}
\end{equation}
Similarly, with probability $1-\delta$, all of the following inequalities are true
\begin{align}
    \widehat{C}(\pi_t) + \epsilon_{FQE} &\geq C(\pi_t) \geq \widehat{C}(\pi_t) - \epsilon_{FQE} \label{eqn:error_analysis_3}\\
    \widehat{G}(\pi_t) + \epsilon_{FQE}\mathbf{1} &\geq G(\pi_t) \geq \widehat{G}(\pi_t) - \epsilon_{FQE}\mathbf{1} \text{   (row wise for all } m \text{ constraints}) \label{eqn:error_analysis_5}
\end{align}
Thus with probability at least $1-\delta$
\begin{align}
  L(\pi_t, \lambda_t) =C(\pi_t)+\lambda_t^\top (G(\pi_t)-\tau) &\geq \widehat{C}(\pi_t)+\lambda_t^\top(\widehat{G}(\pi_t) - \tau)-\epsilon_{FQE}(1+\lambda_t^\top\mathbf{1})  \nonumber \\
  &\geq \widehat{C}(\pi_t)+\lambda_t^\top(\widehat{G}(\pi_t) - \tau)-\epsilon_{FQE}(1+B) \nonumber \\
  &=\widehat{L}(\pi_t,\lambda_t) -\epsilon_{FQE}(1+B) \label{eqn:error_analysis_2} 
\end{align}
Recall that the execution of mixture policy $\pih$ is done by uniformly sampling one policy $\pi_t$ from $\{\pi_1,\ldots,\pi_T \}$, and rolling-out with $\pi_t$. Thus from equations (\ref{eqn:error_analysis_1}) and (\ref{eqn:error_analysis_2}), we have
$\E_{t\sim U[1:T]} \widehat{L}(\pi_t,\lambda_t)\leq C(\pi^*)+\epsilon_{FQI}+(1+B)\epsilon_{FQE}$ w.p. $1-\delta$. In other words, with probability $1-\delta$:
$$\frac{1}{T} \sum_{t=1}^T \widehat{L}(\pi_t,\lambda_t) \leq C(\pi^*)+\epsilon_{FQI}+(1+B)\epsilon_{FQE}$$
Due to the no-regret property of our online algorithm (EG in this case):
$$\frac{1}{T}\sum_{t=1}^T \Lh(\pi_t,\lambda_t)\geq \max_{\lambda}\Lh(\pih,\lambda) - \omega = \Ch(\pih)+\max_\lambda \lambda^\top(\Gh(\pih)-\tau)-\omega$$
If $\Gh(\pih)-\tau \leq 0$ componentwise, choose $\lambda[i] = 0, i=1,2,\ldots,m$ and $\lambda[m+1]=B$. Otherwise, we can choose $\lambda[j] = B$ for $j = \argmax\limits_{i=1: m+1}\left[\widehat{g}_i(\pih)-\tau[i]\right]$ and $\lambda[i] = 0\enskip\forall i\neq j$. We can see that $\max\limits_{\lambda\in\R^{m+1}_{+}, \norm{\lambda}_1=B}\lambda^\top(\Gh(\pih)-\tau) \geq 0$.
Therefore:
$$\Ch(\pih)-\omega \leq C(\pi^*)+\epsilon_{FQI} + (1+B)\epsilon_{FQE} \text{ with probability at least } 1-\delta$$
Combined with the first term from equation (\ref{eqn:error_analysis_3}):
$$C(\pih)-\epsilon_{FQE}-\omega \leq C(\pi^*) + \epsilon_{FQI} + (1+B)\epsilon_{FQE}$$ or
\begin{equation}
C(\pih) \leq C(\pi^*) + \omega + \epsilon_{FQI} + (2+B)\epsilon_{FQE} \label{eqn:error_analysis_4}
\end{equation}
We now bring in the generalization error results from our standalone analysis of FQI (appendix \ref{sec:proof-fqi}) and FQE (appendix \ref{sec:proof-fqe}) into equation (\ref{eqn:error_analysis_4}). 

Specifically, when $n\geq\frac{24\cdot214\cdot \widebar{V}^4}{\epsilon^2}\left( \log\frac{K(m+1)}{\delta}+\pdim_{\F}\log\frac{320 \widebar{V}^2}{\epsilon^2}+\log(14e(\pdim_{\F}+1))\right)$, when FQI and FQE are run with $K$ iterations, we have the guarantee that for any $\epsilon>0$, with probability at least $1-\delta$
\begin{align}
    C(\pih)&\leq C(\pi^*) + \omega + \underbrace{\frac{2\gamma}{(1-\gamma)^3}\left( \sqrt{C_{\mu}}\epsilon + 2\gamma^{K/2} \widebar{V}\right)}_{\text{FQI generalization error}} + \underbrace{\frac{\gamma^{1/2}(2+B)}{(1-\gamma)^{3/2}} \left( \sqrt{C_{\mu}}\epsilon + \frac{\gamma^{K/2}}{(1-\gamma)^{1/2}}2 \widebar{V}\right)}_{(2+B)\times\text{ FQE generalization error}} \nonumber \\
    &\leq C(\pi^*) + \omega + \frac{(4+B)\gamma}{(1-\gamma)^3}\left( \sqrt{C_{\mu}}\epsilon + 2\gamma^{K/2} \widebar{V}\right) \label{eqn:error_analysis_6}
\end{align}
From lemma \ref{lem:empirical_constraint_satisfaction}, $\Gh(\pih)\leq \tau+2\frac{\widebar{C}+\omega}{B}\leq \tau+2\frac{\widebar{V}+\omega}{B}$. From equation (\ref{eqn:error_analysis_5}), for each t=1,2,\ldots,T, we have $\Gh(\pi_t)\geq G(\pi_t)-\epsilon_{FQE}\mathbf{1}$ with probability $1-\delta$. Thus
$$\mathbf{P}\left( \Gh(\pih)\geq G(\pih)-\epsilon_{FQE}\mathbf{1} \right) = \sum_{t=1}^T \mathbf{P}(\Gh(\pi_t)\geq G(\pi_t)-\epsilon_{FQE}\mathbf{1}|\pih = \pi_t)\mathbf{P}(\pih=\pi_t)\geq T(1-\delta)\frac{1}{T} = 1-\delta$$
Therefore, we have the following generalization guarantee for the approximate satisfaction of all constraints:
\begin{equation}
    G(\pih)\leq \tau + 2\frac{\widebar{V}+\omega}{B} + \frac{\gamma^{1/2}}{(1-\gamma)^{3/2}} \left( \sqrt{C_{\mu}}\epsilon + \frac{\gamma^{K/2}}{(1-\gamma)^{1/2}}2 \widebar{V}\right) \label{eqn:error_analysis_7}
\end{equation}
Inequalities (\ref{eqn:error_analysis_6}) and (\ref{eqn:error_analysis_7}) complete the proof of theorem \ref{thm:end_to_end_appendix} (and theorem \ref{thm:end_to_end_main} of the main paper)
\clearpage
\section{Preliminaries to Analysis of Fitted Q Evaluation (FQE) and Fitted Q Iteration (FQI)}
\label{sec:appendix_preliminaries}
In this section, we set-up necessary notations and definitions for the theoretical analysis of FQE and FQI. To simplify the presentation, we will focus exclusively on weighted $\ell_2$ norm for error analysis. 

With the definitions and assumptions presented in this section, we will present the sample complexity guarantee of Fitted-Q-Evaluation (FQE) in appendix \ref{sec:proof-fqe}. The proof for FQI will follow similarly in appendix \ref{sec:proof-fqi}. 

While it is possible to adapt proofs from related algorithms \cite{munos2008finite,antos2008learning} to analyze FQE and FQI, in the next two sections we show improved convergence rate from $O(n^{-4})$ to $O(n^{-2})$, where $n$ is the number of samples in data set $\D$. 

To be consistent with the notations in the main paper, we use the convention $C(\pi)$ as the value function that denotes long-term accumulated cost, instead of using $V(\pi)$ denoting long-term rewards in the traditional RL literature. Our notation for $Q$ function is similar to the RL literature - the only difference is that the optimal policy minimizes $Q(x,a)$ instead of maximizing. We denote the bound on the value function as $\widebar{C}$ (alternatively if the single timestep cost is bounded by $\widebar{c}$, then $\widebar{C} = \frac{\widebar{c}}{1-\gamma}$). For simplicity, the standalone analysis of FQE and FQI concerns only with the cost objective $c$. Dealing with cost $c+\lambda^\top g$ offers no extra difficulty - in that case we simply augment the bound of the value function to $\widebar{V} = \widebar{C}+B\widebar{G}$.
\subsection{Bellman operators}
The \emph{Bellman optimality operator} $\T:\mathcal{B}(\X\times\A;\widebar{C})\mapsto\mathcal{B}(\X\times\A;\widebar{C})$ as 
\begin{equation}
\label{eqn:appendix_Bellman_optimality_operator}
    (\T Q)(x,a) = c(x,a) + \gamma\int_{\X}\min_{a^\prime\in\A} Q(x^\prime,a^\prime)p(dx^\prime|x,a)
\end{equation}
The optimal value functions are defined as usual by $C^*(x) = \sup\limits_{\pi}C^\pi(x)$ and $Q^*(x,a) = \sup\limits_{\pi} Q^\pi(x,a)\enskip\forall x\in\X, a\in\A$. 

For a given policy $\pi$, the \emph{Bellman evaluation operator} $\Tpi:\mathcal{B}(\X\times\A;\widebar{C})\mapsto\mathcal{B}(\X\times\A;\widebar{C})$ as 
\begin{equation}
\label{eqn:appendix_Bellman_operator}
    (\Tpi Q)(x,a) = c(x,a) + \gamma\int_{\X} Q(x^\prime,\pi(x^\prime))p(dx^\prime|x,a)
\end{equation}
It is well known that $\Tpi Q^\pi = Q^\pi,$ a fixed point of the $\Tpi$ operator. 
\subsection{Data distribution and weighted $\ell_2$ norm}
Denote the state-action data generating distribution as $\mu$, induced by some data-generating (behavior) policy $\pi_\D$, that is, $(x_i,a_i)\sim\mu$ for $(x_i,a_i,x_i^\prime, c_i)\in\D$. 

Note that data set $\D$ is formed by multiple trajectories generated by $\pi_\D$. For each $(x_i,a_i)$, we have $x_i^\prime\sim p(\cdot|x_i,a_i)$ and $c_i=c(x_i,a_i)$. 
For any (measurable) function $f:\X\times\A\mapsto\R$, define the $\mu$-weighted $\ell_2$ norm of $f$ as $\norm{f}_\mu^2 = \int_{\X\times\A}f(x,a)^2\mu(dx,da) = \int_{\X\times\A} f(x,a)^2 \mu_x(dx)\pi_\D(a|dx)$. Similarly for any other state-action distribution $\rho$, $\norm{f}_\rho^2 = \int_{\X\times\A}f(x,a)^2\rho(dx,da)$
\subsection{Inherent Bellman error}
FQE and FQI depend on a chosen function class $\F$ to approximate $Q(x,a)$. To express how well the Bellman operator $\T g$ can be approximated by a function in the policy class $\F$, when $\T g\notin\F$, a notion of distance, known as inherent Bellman error was first proposed by \cite{munos2003error} and used in the analysis of related ADP algorithms \cite{munos2008finite,munos2007performance,antos2008fitted,antos2008learning,lazaric2010finite,lazaric2012finite,lazaric2011transfer,maillard2010finite}.
\begin{defn}[Inherent Bellman Error] Given a function class $\F$ and a chosen distribution $\rho$, the \textit{inherent Bellman error} of $\F$ is defined as
\begin{equation*}
    d_\F = d(\F,\T\F) = \sup_{h\in\F}\inf_{f\in\F}\norm{f-\T h}_\rho
\end{equation*}
where $\norm{\cdot}_\rho$ is the $\rho-$weighted $\ell_2$ norm and $\T$ is the Bellman optimality operator defined in (\ref{eqn:appendix_Bellman_optimality_operator})
\end{defn}
To analyze FQE, we will form a similar definition for the Bellman evaluation operator
\begin{defn}[Inherent Bellman Evaluation Error] Given a function class $\F$ and a policy $\pi$, the \textit{inherent Bellman evaluation error} of $\F$ is defined as
\begin{equation*}
    d_\F^\pi = d(\F,\T^\pi\F) = \sup_{h\in\F}\inf_{f\in\F}\norm{f-\T^\pi h}_{\rho_\pi}
\end{equation*}
where $\norm{\cdot}_{\rho_\pi}$ is the $\ell_2$ norm weighted by $\rho_\pi$. $\rho_\pi$ is defined as the state-action distribution induced by policy $\pi$, and $\T^\pi$ is the Bellman operator defined in (\ref{eqn:appendix_Bellman_operator})
\end{defn}
\subsection{Concentration coefficients}
Let $P^\pi$ denote the operator acting on $f:\X\times\A\mapsto\R$ such that $(P^\pi f)(x,a) = \int_\X f(x^\prime, \pi(x^\prime)) p(x^\prime|x,a) dx^\prime$. Acting on $f$ (e.g., approximates $Q$), $P^\pi$ captures the transition dynamics of taking action $a$ and following $\pi$ thereafters. 

The following definition and assumption are standard in the analysis of related approximate dynamic programming algorithms \cite{lazaric2012finite, munos2008finite, antos2008fitted}. As approximate value iteration and policy iteration algorithms perform policy update, the new policy at each round will induce a different stationary state-action distribution. One way to quantify the distribution shift is the notion of concentrability coefficient of future state-action distribution, a variant of the notion introduced by \cite{munos2003error}. 

\begin{defn}[Concentration coefficient of state-action distribution]\label{def:concentrability} Given data generating distribution $\mu\sim\pi_\D$, initial state distribution $\chi$. For $m\geq0$, and an arbitrary sequence of stationary policies $\{ \pi_m\}_{m\geq 1}$ let 
\begin{equation*}
\beta_{\mu}(m) = \sup_{\pi_1,\ldots,\pi_m} \left\| \frac{d(\chi P^{\pi_1}P^{\pi_2}\ldots P^{\pi_m})}{d\mu} \right\|_\infty    
\end{equation*}
($\beta_{\mu}(m) = \infty$ if the future state distribution $\chi P^{\pi_1}P^{\pi_2}\ldots P^{\pi_m}$ is not absolutely continuous w.r.t. $\mu$, i.e, $\chi P^{\pi_1}P^{\pi_2}\ldots P^{\pi_m}(x,a)>0$ for some $\mu(x,a) = 0$)
\end{defn}
\begin{assumption}
\label{assume:concentrability}
$\beta_{\mu} = (1-\gamma)^2\sum\limits_{m\geq 1} m\gamma^{m-1} \beta_{\mu}(m) < \infty$
\end{assumption}

\textbf{Combination Lock Example.} An example of an MDP that violates Assumption \ref{assume:concentrability} is the \doubleQuote{combination lock} example proposed by \cite{koenig1996effect}. In this finite MDP, we have $N$ states $\X = \{ 1,2,\ldots,N\}$, and 2 actions: going L or R. The initial state is $x_0=1$. In any state $x$, action $L$ takes agent back to initial state $x_0$, and action $R$ advances the agent to the next state $x+1$ in a chain fashion. Suppose that the reward is 0 everywhere except for the very last state $N$. One can see that for an MDP such that any behavior policy $\pi_D$ that has a bounded from below probability of taking action $L$ from any state $x$, i.e., $\pi_D(L|x)\geq \nu>0$, then it takes an exponential number of trajectories to learn or evaluate a policy that always takes action $R$. In this setting, we can see that the concentration coefficient $\beta_\mu$ can be designed to be arbitrarily large. 
\subsection{Complexity measure of function class $\F$}
\begin{defn}[Random $L_1$ Norm Covers] Let $\epsilon>0$, let $\F$ be a set of functions $\X\mapsto\R$, let $x_1^n = (x_1,\ldots,x_n)$ be $n$ fixed points in $\X$. Then a collection of functions $\F_\epsilon = \{f_1,\ldots, f_N\}$ is an $\epsilon$-cover of $\F$ on $x_1^n$ if 
$$\forall f\in\F, \exists f^\prime\in\F_\epsilon: \lvert\frac{1}{n}\sum_{i=1}^n f(x_i) - \frac{1}{n} \sum_{i=1}^n f^\prime(x_i)\rvert\leq \epsilon$$
The empirical covering number, denote by $\mathcal{N}_1(\epsilon, \F, x_1^n)$, is the size of the smallest $\epsilon$-cover on $x_1^n$. Take $\mathcal{N}_1(\epsilon, \F, x_1^n) = \infty$ if no finite $\epsilon$-cover exists. 
\end{defn}
\begin{defn}[Pseudo-Dimension] A real-valued function class $\F$ has pseudo-dimension $\pdim_{\F}$ defined as the VC dimension of the function class induced by the sub-level set of functions of $\F$. In other words, define function class $\mathrm{H} = \{(x,y)\mapsto \text{sign}(f(x)-y : f\in\F\}$, then
$$\pdim_{\F} = \texttt{VC-dimension}(\mathrm{H})$$

\end{defn}
\clearpage
\section{Generalization Analysis of Fitted Q Evaluation}
\label{sec:appendix_proof_fqe}
In this section we prove the following statement for Fitted Q Evaluation (FQE).
\begin{thm}[Guarantee for FQE - General Case (theorem \ref{thm:fqe_main} in main paper)]\label{thm:fqe_appendix}
Under Assumption \ref{assume:concentrability}, for $\epsilon>0$ \& $\delta\in(0,1)$, after $K$ iterations of Fitted Q Evaluation (Algorithm \ref{algo:fqe}), for $n=O\big(\frac{\widebar{C}^4}{\epsilon^2}( \log\frac{K}{\delta}+\textnormal{\texttt{dim}}_{\F}\log\frac{\widebar{C}^2}{\epsilon^2}+\log \textnormal{\texttt{dim}}_{\F})\big)$, we have with probability $1-\delta$:
$$\big\lvert C(\pi) - \widehat{C}(\pi)\big\rvert \leq \frac{\gamma^{1/2}}{(1-\gamma)^{3/2}} \big( \sqrt{\beta_{\mu}}\left(2d_\F^\pi+\epsilon\right) + \frac{2\gamma^{K/2}\widebar{C}}{(1-\gamma)^{1/2}}\big).$$
\end{thm}
\begin{thm}[Guarantee for FQE - Bellman Realizable Case] \label{thm:fqe_appendix_realizable}
Under Assumptions \ref{assume:concentrability}-\ref{assume:realizability_fqe}, for any $\epsilon>0, \delta\in(0,1)$, after $K$ iterations of Fitted Q Evaluation (Algorithm \ref{algo:fqe}), when $n\geq\frac{24\cdot214\cdot \widebar{C}^4}{\epsilon^2}\big( \log\frac{K}{\delta}+\pdim_{\F}\log\frac{320 \widebar{C}^2}{\epsilon^2}+\log(14e(\pdim_{\F}+1))\big)$, we have with probability $1-\delta$:
$$\big\lvert C(\pi) - \widehat{C}(\pi)\big\rvert \leq \frac{\gamma^{1/2}}{(1-\gamma)^{3/2}} \big( \sqrt{\beta_{\mu}}\epsilon + \frac{2\gamma^{K/2}\widebar{C}}{(1-\gamma)^{1/2}}\big)$$
\end{thm}
We first focus on theorem \ref{thm:fqe_appendix_realizable}, analyzing FQE assuming a sufficiently rich function class $\F$ so that the Bellman evaluation update $\T^\pi$ is closed wrt $\F$ (thus inherent Bellman evaluation error is 0). We call this the \emph{Bellman evaluation realizability assumption}. This assumption simplifies the presentation of our bounds and also simplifies the final error analysis of Algo. \ref{algo:main_algo}. 

After analyzing FQE under this Bellman realizable setting, we will turn to error bound for general, non-realizable setting in theorem \ref{thm:fqe_appendix} (also theorem \ref{thm:fqe_main} in the main paper). The main difference in the non-realizable setting is the appearance of an extra term $d_\F^\pi$ our final bound.

\label{sec:proof-fqe}
\subsection{Error bound for single iteration - Bellman realizable case}
\label{subsec:fqe_single_iteration_realizable}
\begin{assumption}[Bellman evaluation realizability]
\label{assume:realizability_fqe}
We consider function classes $\F$ sufficiently rich so that $\forall f,\T^\pi f\in\F$.
\end{assumption}
We begin with the following result bounding the error for a single iteration of FQE, under \doubleQuote{training} distribution $\mu\sim\pi_D$
\begin{prop}[Error bound for single iteration]\label{lem:fqe_single_iteration_bound} 
Let the functions in $\F$ also be bounded by $\widebar{C}$, and let $\pdim_{\F}$ denote the pseudo-dimension of the function class $\F$. We have with probability at least $1-\delta$:
$$\norm{Q_k-\T^\pi Q_{k-1}}_\mu < \epsilon$$ 
when $n\geq\frac{24\cdot214\cdot \widebar{C}^4}{\epsilon^2}\left( \log\frac{1}{\delta}+\pdim_{\F}\log\frac{320 \widebar{C}^2}{\epsilon^2}+\log(14e(\pdim_{\F}+1))\right)$
\end{prop}
\begin{rem}
Note from proposition \ref{lem:fqe_single_iteration_bound} that the dependence of sample complexity $n$ here on $\epsilon$ is $\widetilde{O}(\frac{1}{\epsilon^2})$, which is better than previously known analysis for Fitted Value Iteration \cite{munos2008finite} and FittedPolicyQ (continuous version of Fitted Q Iteration \cite{antos2008fitted}) dependence of $\widetilde{O}(\frac{1}{\epsilon^4})$. The finite sample analysis of LSTD \cite{lazaric2010finite} showed an $\widetilde{O}(\frac{1}{\epsilon^2})$ dependence using linear function approximation. Here we prove similar convergence rate for general non-linear (bounded) function approximators.
\end{rem}
\emph{Proof of Proposition \ref{lem:fqe_single_iteration_bound}.} Recall the training target in round $k$ is $y_i = c_i+\gamma Q_{k-1}(x_i^\prime, \pi(x_i^\prime))$ for $i=1,2,\ldots,n$, and $Q_k\in\F$ is the solution to the following regression problem:
\begin{equation*}
    Q_k = \argmin_{f\in\F}\frac{1}{n}\sum_{i=1}^n (f(x_i,a_i)-y_i)^2
\end{equation*}
Consider random variables $(x,a)\sim\mu$ and $y=c(x,a)+\gamma Q_{k-1}(x^\prime,\pi(x^\prime))$ where $x^\prime\sim p(\cdot|x,a)$. 
By this definition, $\T^\pi Q_{k-1}$ is the \textit{regression function} that minimizes square loss $\min\limits_{h:\R^{X\times A}\mapsto\R}\E\abs{h(x,a)-y}^2$ out of all functions $h$ (not necessarily in $\F$). This is due to $(\T^\pi Q_{k-1})(\tilde{x},\tilde{a}) = \E\left[y|x=\tilde{x},a=\tilde{a} \right]$ by definition of the Bellman operator. Consider $Q_{k-1}$ fixed and we now want to relate the learned function $Q_k$ over finite set of $n$ samples with the regression function over the whole data distribution via uniform deviation bound. We use the following lemma:
\begin{lem}[\cite{gyorfi2006distribution}, theorem 11.4. Original version \cite{lee1996efficient}, theorem 3]\label{lem:bartlett} Consider random vector $(X,Y)$ and $n$ i.i.d samples $(X_i,Y_i)$. Let $m(x)$ be the (optimal) regression function under square loss $m(x) = \E[Y|X=x]$. Assume $\abs{Y}\leq B$ a.s. and $B\leq 1$. Let $\F$ be a set of function $f:\R^d\mapsto\R$ and let $\abs{f(x)}\leq B$. Then for each $n\geq 1$
\begin{align*}
    \mathbf{P}\bigg\{ \exists f\in\F: \E\abs{f(X)-Y}^2-\E\abs{m(X)-Y}^2 -\frac{1}{n}\sum_{i=1}^n\left( \abs{f(X_i)-Y_i}^2-\abs{m(X_i)-Y_i}^2\right) \geq \\ \epsilon\cdot\left(\alpha+\beta+\E\abs{f(X)-Y}^2-\E\abs{m(X)-Y}^2 \right) \bigg\} \\
    \leq 14\sup_{x_1^n}\mathcal{N}_1\left(\frac{\beta\epsilon}{20B},\F,x_1^n \right)\exp{\left(-\frac{\epsilon^2(1-\epsilon)\alpha n}{214(1+\epsilon)B^4}\right)}   
\end{align*}
where $\alpha,\beta>0$ and $0<\epsilon<1/2$
\end{lem}
To apply this lemma, first note that since $\T^\pi Q_{k-1}$ is the optimal regression function\footnote{It is easy to see that if $m(x) = \E[y|x]$ is the regression function then for any function $f(x)$, we have $\E\left[ (f(x)-m(x))(m(x)-y)=0\right]$}, we have
\begin{align*}
    \E_\mu\left[(Q_k(x,a)-y)^2\right] &= \E_\mu\left[\left(Q_k(x,a)-\T^\pi Q_{k-1}(x,a) + \T^\pi Q_{k-1}(x,a) -y\right)^2\right] \nonumber \\
    &=\E_\mu\left[\left(Q_k(x,a)-\T^\pi Q_{k-1}(x,a)\right)^2] + \E_\mu[\left(\T^\pi Q_{k-1}(x,a)-y\right)^2\right]
\end{align*}

thus $$\norm{Q_k-\T^\pi Q_{k-1}}_\mu^2 = \E\left[ (Q_k(x,a)-\T^\pi Q_{k-1}(x,a))^2\right] = \E\left[ (Q_k(x,a) - y)^2\right] - \E\left[ (\T^\pi Q_{k-1}(x,a)-y)^2\right]$$ where by definition
\begin{align*}
\E\left[ (Q_k(x,a)-\T^\pi Q_{k-1}(x,a))^2\right] &= \int\left(Q_k(x,a) - \T^\pi Q_{k-1}(x,a)\right)^2\mu(dx,da)    \\
&=\int (Q_k(x,a) - \T^\pi(x,a))^2 \mu_x(dx)\pi_\D(a|dx)
\end{align*}
Next, given a fixed data set $\widetilde{D}_k\sim\mu$
\begin{align}
    &\mathbf{P}\big\{ \norm{Q_k-\T^\pi Q_{k-1}}_\mu^2 > \epsilon \big\} = \mathbf{P}\bigg\{ \E\left[ (Q_k(x,a) - y)^2\right] - \E\left[ (\T^\pi Q_{k-1}(x,a)-y)^2\right] > \epsilon \bigg\} \nonumber\\
    & \leq \mathbf{P}\bigg\{ \E\left[ (Q_k(x,a) - y)^2\right] - \E\left[ (\T^\pi Q_{k-1}(x,a)-y)^2\right] \nonumber\\
    &\qquad\qquad- 2\cdot\left(\frac{1}{n}\sum_{i=1}^n (Q_k(x_i,a_i)-y_i)^2 - \frac{1}{n}\sum_{i=1}^n(\T^\pi Q_{k-1}(x_i,a_i)-y_i)^2 \right) >\epsilon \bigg\} \label{eqn:fqe_relax_best_in_class}\\
    &=\mathbf{P}\bigg\{\E\left[ (Q_k(x,a)-y)^2\right] -\E\left[ (\T^\pi Q_{k-1}(x,a)-y)^2\right] \nonumber\\
    &\qquad\qquad-\frac{1}{n}\sum_{i=1}^n\left[ (Q_k(x_i,a_i)-y_i)^2-(\T^\pi Q_{k-1}(x_i,a_i)-y_i)^2\right]\nonumber\\
    &\qquad\qquad\qquad>\frac{1}{2}(\epsilon+\E\left[(Q_k(x,a)-y)^2 \right] -\E\left[(\T^\pi Q_{k-1}(x,a)-y)^2 \right])\bigg\} \label{eqn:fqe_rearrange_11_4}\\
    &\leq\mathbf{P} \bigg\{ \exists f\in\F: \E\left[(f(x,a)-y)^2\right]-\E\left[ (\T^\pi Q_{k-1}(x,a)-y)^2\right] \nonumber\\
    &\qquad\qquad-\frac{1}{n}\sum_{i=1}^n\left[ (f(x_i,a_i)-y_i)^2-(\T^\pi Q_{k-1}(x_i,a_i)-y_i)^2\right]\nonumber \\
    &\qquad\qquad\qquad \geq \frac{1}{2}(\frac{\epsilon}{2}+\frac{\epsilon}{2}+\E\left[(f(x,a)-y)^2\right]-\E\left[(\T^\pi Q_{k-1}(x,a)-y)^2\right] )\bigg\} \nonumber\\
    &\leq 14\sup_{x_1^n}\mathcal{N}_1\left(\frac{\epsilon}{80\widebar{C}},\F,x_1^n\right)\cdot\exp{\left(-\frac{n\epsilon}{24\cdot214\widebar{C}^4}\right)}\label{eqn:fqe_cover_bound}
\end{align}
Equation (\ref{eqn:fqe_relax_best_in_class}) uses the definition of $Q_k = \argmin\limits_{f\in\F}\frac{1}{n}\sum_{i=1}^n (f(x_i,a_i)-y_i)^2$ and the fact that $\T^\pi Q_{k-1}\in\F$, thus making the extra term a positive addition. Equation (\ref{eqn:fqe_rearrange_11_4}) is due to rearranging the terms. Equation (\ref{eqn:fqe_cover_bound}) is an application of lemma \ref{lem:bartlett}. We can further bound the empirical covering number by invoking the following lemma due to Haussler \cite{haussler1995sphere}:
\begin{lem}[\cite{haussler1995sphere}, Corollary 3]\label{lem:haussler} For any set $X$, any points $x^{1:n}\in\X^n$, any class $\F$ of functions on $X$ taking values in $[0,\widebar{C}]$ with pseudo-dimension $\pdim_{\F}<\infty$, and any $\epsilon>0$
$$\mathcal{N}_1(\epsilon, \F, x_1^n)\leq e(\pdim_{\F}+1)\left(\frac{2e\widebar{C}}{\epsilon}\right)^{\pdim_{\F}}$$

\end{lem}
Applying lemma \ref{lem:haussler} to equation (\ref{eqn:fqe_cover_bound}), we have the inequality
\begin{equation}
\label{eqn:fqe_single_bound_realize}
    \mathbf{P}\big\{ \norm{Q_k-\T^\pi Q_{k-1}}_\mu^2 > \epsilon \big\} \leq 14\cdot e\cdot (\pdim_{\F}+1)\left( \frac{320 \widebar{C}^2}{\epsilon}\right)^{\pdim_{\F}}\cdot\exp{\left(-\frac{n\epsilon}{24\cdot214\widebar{C}^4}\right)}
\end{equation}
We thus have that when $n\geq\frac{24\cdot214\cdot \widebar{C}^4}{\epsilon^2}\left( \log\frac{1}{\delta}+\pdim_{\F}\log\frac{320 \widebar{C}^2}{\epsilon^2}+\log(14e(\pdim_{\F}+1))\right)$:
$$\norm{Q_k-\T^\pi Q_{k-1}}_\rho < \epsilon$$ with probability at least $1-\delta$.
Notice that the dependence of sample complexity $n$ here on $\epsilon$ is $\widetilde{O}(\frac{1}{\epsilon^2})$, which is better than previously known analyses for other approximate dynamic programming algorithms such as Fitted Value Iteration \cite{munos2008finite}, FittedPolicyQ \cite{antos2008learning,antos2008fitted} with dependence of $O(\frac{1}{\epsilon^4})$.
\subsection{Error bound for single iteration - Bellman non-realizable case}
\label{subsec:fqe_single_iteration_nonrealizable}
We now give similar error bound for the general case, where Assumption \ref{assume:realizability_fqe} does not hold. Consider the decomposition
\begin{align}
    \norm{Q_k-\T^\pi Q_{k-1}}_\mu^2 &= \E\left[ (Q_k(x,a) - y)^2\right] - \E\left[ (\T^\pi Q_{k-1}(x,a)-y)^2\right] \nonumber\\
    &= \bigg\{ \E\left[ (Q_k(x,a) - y)^2\right] - \E\left[ (\T^\pi Q_{k-1}(x,a)-y)^2\right] \nonumber \\
    &\qquad\qquad- 2\cdot\left(\frac{1}{n}\sum_{i=1}^n (Q_k(x_i,a_i)-y_i)^2 - \frac{1}{n}\sum_{i=1}^n(\T^\pi Q_{k-1}(x_i,a_i)-y_i)^2 \right)\bigg\} \nonumber \\
    &\qquad\qquad+ \bigg\{ 2\cdot\left(\frac{1}{n}\sum_{i=1}^n (Q_k(x_i,a_i)-y_i)^2 - \frac{1}{n}\sum_{i=1}^n(\T^\pi Q_{k-1}(x_i,a_i)-y_i)^2 \right) \bigg\} \nonumber \\
    &=\texttt{component\_1} + \texttt{component\_2} \nonumber
\end{align}
Splitting the probability of error into two separate bounds.
We saw from the previous section (equation (\ref{eqn:fqe_single_bound_realize})) that 
\begin{equation}
    \mathbf{P}(\texttt{component\_1} >\epsilon/2) \leq 14\cdot e\cdot (\pdim_{\F}+1)\left( \frac{640 \widebar{C}^2}{\epsilon}\right)^{\pdim_{\F}}\cdot\exp{\left(-\frac{n\epsilon}{48\cdot214\widebar{C}^4}\right)} \label{eqn:fqe_non_realize_single_iteration_3}
\end{equation}
We no longer have $\texttt{component\_2} \leq 0$ since $\T^\pi Q_{k-1}\notin \F$. Let $f^* = \arginf\limits_{f\in\F}\norm{f-\T^\pi Q_{k-1}}_\mu^2$. Since $Q_k = \argmin\limits_{f\in\F} \frac{1}{n}\sum_{i=1}^n (f(x_i,a_i)-y_i )^2$, we can upper-bound \texttt{component\_2} by
$$\texttt{component\_2}\leq 2\cdot\left(\frac{1}{n}\sum_{i=1}^n (f^*(x_i,a_i)-y_i)^2 - \frac{1}{n}\sum_{i=1}^n(\T^\pi Q_{k-1}(x_i,a_i)-y_i)^2 \right)$$
We can treat $f^*$ as a fixed function, unlike random function $Q_k$, and use standard concentration inequalities to bound the empirical average from the expectation. Let random variable $z = ((x,a),y)$, $z_i = ((x_i,a_i),y_i), i=1,\ldots,n$ and let 
$$h(z) = (f^*(x,a)-y)^2 - (\T^\pi Q_{k-1}(x,a)-y)^2$$
We have $\abs{h(z)}\leq 4\widebar{C}^2$. We will derive a bound for $$\mathbf{P}\left(\frac{1}{n}\sum_{i=1}^n h(z_i)-\E h(z)>\frac{\epsilon}{4}+\E h(z) \right)$$using Bernstein inequality\cite{mohri2012foundations}. First, using the relationship $h(z) = (f^*(x,a)+\T^\pi Q_{k-1}(x,a) - 2y)(f^*(x,a)-\T^\pi Q_{k-1}(x,a))$, the variance of $h(z)$ can be bounded by a constant factor of $\E h(z)$, since
\begin{align}
    \mathbf{Var}(h(z))&\leq \E h(z)^2 \leq 16\widebar{C}^2\E\left[ (f^*(x,a) - \T^\pi Q_{k-1}(x,a))^2\right] \nonumber \\
    &=16 \widebar{C}^2\left(\E\left[(f^*(x,a)-y)^2 \right]-\E\left[(\T^\pi Q_{k-1}(x,a)-y)^2 \right] \right) \label{eqn:fqe_non_realize_single_iteration_1}\\
    &=16 \widebar{C}^2\E h(z) \label{eqn:fqe_non_realize_single_iteration_2}
\end{align}
Equation (\ref{eqn:fqe_non_realize_single_iteration_1}) stems from $\T^\pi Q_{k-1}$ being the optimal regression function. Now we can apply equation (\ref{eqn:fqe_non_realize_single_iteration_2}) and Bernstein inequality to obtain
\begin{align*}
    &\mathbf{P}\left(\frac{1}{n}\sum_{i=1}^n h(z_i)-\E h(z)>\frac{\epsilon}{4}+\E h(z) \right) \leq \mathbf{P} \left( \frac{1}{n}\sum_{i=1}^n h(z_i)-\E h(z)>\frac{\epsilon}{4}+ \frac{\mathbf{Var}(h(z))}{16\widebar{C}^2}\right) \leq \ldots\\
    &\leq \exp{\left( -\frac{n\left(\frac{\epsilon}{4} +\frac{\mathbf{Var}}{16\widebar{C}^2} \right)^2 }{2\mathbf{Var}+2\frac{4\widebar{C}^2}{3}\left( \frac{\epsilon}{4} + \frac{\mathbf{Var}}{16\widebar{C}^2}\right)  } \right)} \\
    &\leq \exp{\left( -\frac{n\left(\frac{\epsilon}{4} +\frac{\mathbf{Var}}{16\widebar{C}^2} \right)^2 }{ \left(32\widebar{C}^2+\frac{8\widebar{C}^2}{3}\right) \left( \frac{\epsilon}{4} + \frac{\mathbf{Var}}{16\widebar{C}^2}\right)  } \right)} = \exp{\left( -\frac{n\left(\frac{\epsilon}{4} +\frac{\mathbf{Var}}{16\widebar{C}^2} \right) }{ 32\widebar{C}^2+\frac{8\widebar{C}^2}{3}  } \right)} \leq \exp{\left(-\frac{1}{128+\frac{32}{3}}\cdot\frac{n\epsilon}{\widebar{C}^2} \right)}
\end{align*}

Thus
\begin{equation}
\mathbf{P}\left( 2\cdot\left[\frac{1}{n}\sum_{i=1}^n h(z_i) - 2\E h(z)\right] >\frac{\epsilon}{2}\right) \leq \exp{\left(-\frac{3}{416}\cdot\frac{n\epsilon}{\widebar{C}^2}\right)}
    \label{eqn:fqe_non_realize_single_iteration_4}
\end{equation}
Now we have $$\texttt{component\_2}\leq 2\cdot\frac{1}{n}\sum_{i=1}^n h(z_i) = 2\cdot\left[\frac{1}{n}\sum_{i=1}^n h(z_i) - 2\E h(z)\right] + 4\E h(z)$$
Using again the fact that $\T^\pi Q_{k-1}$ is the optimal regression function
\begin{align}
    \E h(z) &= \E_D\left[ (f^*(x,a)-y)^2\right]-\E_D\left[ (\T^\pi Q_{k-1}(x,a)-y)^2\right] = \E_D\left[ (f^*(x,a) - \T^\pi Q_{k-1}(x,a))^2\right] \nonumber \\
    &= \inf_{f\in\F}\norm{f-\T^\pi Q_{k-1}}_\mu^2 \label{eqn:fqe_non_realize_single_iteration_5}
\end{align}
Combining equations (\ref{eqn:fqe_non_realize_single_iteration_3}), (\ref{eqn:fqe_non_realize_single_iteration_4}) and (\ref{eqn:fqe_non_realize_single_iteration_5}), we can conclude that 
\begin{align}
    \mathbf{P}\big\{ \norm{Q_k-\T^\pi Q_{k-1}}_\mu^2-4\inf_{f\in\F}\norm{f-\T^\pi Q_{k-1}}_\mu^2 > \epsilon \big\} &\leq  14\cdot e\cdot (\pdim_\F+1)\left( \frac{640 \widebar{C}^2}{\epsilon}\right)^{\pdim_\F}\cdot\exp{\left(-\frac{n\epsilon}{48\cdot214\widebar{C}^4}\right)} \nonumber \\ 
    &\qquad +\exp{\left(-\frac{3}{416}\cdot\frac{n\epsilon}{\widebar{C}^2}\right)} \nonumber
\end{align}
thus implying
\begin{align}
    \mathbf{P}\big\{ \norm{Q_k-\T^\pi Q_{k-1}}_\mu-2\inf_{f\in\F}\norm{f-\T^\pi Q_{k-1}}_\mu > \epsilon \big\} &\leq  14\cdot e\cdot (\pdim_\F+1)\left( \frac{640 \widebar{C}^2}{\epsilon^2}\right)^{\pdim_\F}\cdot\exp{\left(-\frac{n\epsilon^2}{48\cdot214\widebar{C}^4}\right)} \nonumber \\ 
    &\qquad +\exp{\left(-\frac{3}{416}\cdot\frac{n\epsilon^2}{\widebar{C}^2}\right)} \label{eqn:fqe_non_realize_single_iteration_6}
\end{align}
We now can further upper-bound the term $2\inf_{f\in\F}\norm{f-\T^\pi Q_{k-1}}_\mu \leq 2\sup_{f^\prime\in\F}\inf_{f\in\F}\norm{f-\T^\pi f^\prime}_\mu = 2d_\F^\pi$ (the worst-case \textit{inherent Bellman evaluation error}), leading to the final bound for the Bellman non-realizable case. 

One may wish to further remove the inherent Bellman evaluation error from our error bound. However, counter-examples exist where the inherent Bellman error cannot generally be estimated using function approximation (see section 11.6 of \cite{sutton2018reinforcement}). Fortunately, inherent Bellman error can be driven to be small by choosing rich function class $\F$ (low bias), at the expense of more samples requirement (higher variance, through higher pseudo-dimension $\pdim_\F$).

While the bound in (\ref{eqn:fqe_non_realize_single_iteration_6}) looks more complicated than the Bellman realizable case in equation \ref{eqn:fqe_single_bound_realize}, note that the convergence rate will still be $O(\frac{1}{n^2})$.
\subsection{Bounding the error across iterations}
\label{subsec:fqe_error_propagation}
Previous sub-sections \ref{subsec:fqe_single_iteration_nonrealizable} and \ref{subsec:fqe_single_iteration_nonrealizable} have analyzed the error of FQE for a single iteration in Bellman realizable and non-realizable case. We now analyze how errors from different iterations flow through the FQE algorithm. The proof borrows the idea from lemma 3 and 4 of \cite{munos2008finite} for fitted value iteration (for value function $V$ instead of $Q$), with appropriate modifications for our off-policy evaluation context.

Recall that $C^\pi,Q^\pi$ denote the true value function and action-value function, respectively, under the evaluation policy $\pi$. And $C_K = \E[Q_K(x,\pi(x))]$ denote the value function associated with the returned function $Q_K$ from algorithm \ref{algo:fqe}. Our goal is to bound the difference $C^\pi-C_K$ between the true value function and the estimated value of the returned function $Q_K$. 

Let the unknown state-action distribution induced by the evaluation policy $\pi$ be $\rho$. We first bound the loss $\norm{Q^\pi-Q_K}_{\rho}$ under the \doubleQuote{test-time }distribution $\rho$ of $(x,a)$, which differs from the state-action $\mu$ induced by data-generating policy $\pi_\D$. We will then lift the loss bound from $Q_K$ to $C_K$.

\textbf{Step 1: Upper-bound the value estimation error}

Let $\epsilon_{k-1} = Q_k - \T^\pi Q_{k-1}\in\mathcal{\X\times\A,\widebar{C}}$. We have for every $k$ that
\begin{align}
    Q^\pi - Q_k &= \T^\pi Q^\pi - \T^\pi Q_{k-1}+\epsilon_{k-1} \quad(Q^\pi \text{ is fixed point of } T^\pi) \nonumber \\
    &= \gamma P^\pi(Q^\pi - Q_{k-1}) +\epsilon_{k-1} \nonumber
\end{align}
Thus by simple recursion
\begin{align}
    Q^\pi - Q_K &= \sum_{k=0}^{K-1} \gamma^{K-k-1} (P^\pi)^{K-k-1}\epsilon_k + \gamma^K(P^\pi)^K(Q^\pi - Q_0) \nonumber \\
    &=\frac{1-\gamma^{K+1}}{1-\gamma}\left[ \sum_{k=0}^{K-1}\frac{(1-\gamma)\gamma^{K-k-1}}{1-\gamma^{K+1}} (P^\pi)^{K-k-1}\epsilon_k +\frac{(1-\gamma)\gamma^K}{1-\gamma^{K+1}}(P^\pi)^K(Q^\pi - Q_0) \right] \nonumber \\
    &=\frac{1-\gamma^{K+1}}{1-\gamma}\left[\sum_{k=0}^{K-1}\alpha_k A_k \epsilon_k+\alpha_K A_K (Q^\pi-Q_0) \right] \label{eqn:fqe_propagation_1}
\end{align}
where for simplicity of notations, we denote
\begin{align*}
    \alpha_k &= \frac{(1-\gamma)\gamma^{K-k-1}}{1-\gamma^{K+1}} \text{ for }k<K, \alpha_K = \frac{(1-\gamma)\gamma^K}{1-\gamma^{K+1}} \\
    A_k &= (P^\pi)^{K-k-1}, A_K = (P^\pi)^K
\end{align*}
Note that $A_k$'s are probability kernels and $\alpha_k$'s are deliberately chosen such that $\sum_k \alpha_k = 1$.

We can apply point-wise absolute value on both sides of (\ref{eqn:fqe_propagation_1}) with $\abs{f}$ being the short-hand notation for $\abs{f(x,a)}$ and inequality holds point-wise. By triangle inequalities:
\begin{equation}
\label{eqn:fqe_propagation_2}
    \abs{Q^\pi-Q_K}\leq \frac{1-\gamma^{K+1}}{1-\gamma}\left[\sum_{k=0}^{K-1}\alpha_k A_k \abs{\epsilon_k}+\alpha_K A_K \abs{Q^\pi-Q_0} \right]
\end{equation}

\textbf{Step 2: Bounding $\norm{Q^\pi-Q_K}_\rho$ for any unknown distribution $\rho$}. To handle distribution shift from $\mu$ to $\rho$, we decompose the loss as follows:
\begin{align}
    \norm{Q^\pi-Q_K}_\rho^2&=\int\rho(dx,da)\left(Q^\pi(x,a)-Q_K(x,a) \right)^2 \nonumber \\
    &\leq \left[ \frac{1-\gamma^{K+1}}{1-\gamma}\right]^2\int\rho(dx,da)\left[\left(\sum_{k=0}^{K-1}\alpha_k A_k\abs{\epsilon_k}+\alpha_K A_K\abs{Q^\pi-Q_0}\right)(x,a)\right]^2 \text{ (from}(\ref{eqn:fqe_propagation_2}))\nonumber\\
    &\leq \left[ \frac{1-\gamma^{K+1}}{1-\gamma}\right]^2\int\rho(dx,da)\left[\sum_{k=0}^{K-1}\alpha_k (A_k \epsilon_k)^2+\alpha_K (A_K(Q^*-Q_0))^2 \right](x,a) \text{ (Jensen)}\nonumber\\
    &\leq \left[ \frac{1-\gamma^{K+1}}{1-\gamma}\right]^2\int\rho(dx,da)\left[\sum_{k=0}^{K-1}\alpha_k A_k \epsilon_k^2+\alpha_K A_K(Q^*-Q_0)^2 \right](x,a) \text{ (Jensen)}\nonumber
\end{align}
Using assumption \ref{assume:concentrability} (assumption \ref{assume:concentrability_main} of the main paper), we can bound each term $\rho A_k$ as 
\begin{equation*}
    \rho A_k = \rho(P^\pi)^{K-k-1} \leq \mu \beta_{\mu}(K-k-1) \text { (definition \ref{def:concentrability})}
\end{equation*}
Thus
$$\norm{Q^\pi-Q_K}_\rho^2 \leq \left[ \frac{1-\gamma^{K+1}}{1-\gamma}\right]^2 \left[\frac{1}{1-\gamma^{K+1}}\sum_{k=0}^{K-1} (1-\gamma)\gamma^{K-k-1} \beta_{\mu}(K-k-1)\norm{\epsilon_k}_\mu^2 +\alpha_K(2\widebar{C})^2 \right]$$
Assumption \ref{assume:concentrability} (stronger than necessary for proof of FQE) can be used to upper-bound the first order concentration coefficient:
$$(1-\gamma)\sum_{m\geq0}\gamma^m \beta_\mu(m) \leq \frac{\gamma}{1-\gamma}\left[ (1-\gamma)^2\sum_{m\geq 1} m\gamma^{m-1}\beta_\mu(m)\right] = \frac{\gamma}{1-\gamma} \beta_{\mu}$$
This gives the upper-bound for $\norm{Q^\pi - Q_K}_\rho^2$ as
\begin{align}
    \norm{Q^\pi - Q_K}_\rho^2 &\leq \left[ \frac{1-\gamma^{K+1}}{1-\gamma}\right]^2 \left[\frac{\gamma}{(1-\gamma)(1-\gamma^{K+1})} \beta_{\mu}\max_{k}\norm{\epsilon_k}_\mu^2+\frac{(1-\gamma)\gamma^K}{1-\gamma^{K+1}}(2\widebar{C})^2\right] \nonumber \\
    &\leq\frac{1-\gamma^{K+1}}{(1-\gamma)^2} \left[ \frac{\gamma}{1-\gamma}\beta_{\mu}\max_k\norm{\epsilon_k}_\mu^2 + (1-\gamma)\gamma^K(2\widebar{C})^2\right] \nonumber \\
    &\leq \frac{\gamma}{(1-\gamma)^3} \beta_{\mu}\max_k\norm{\epsilon_k}_\mu^2 + \frac{\gamma^K}{1-\gamma}(2\widebar{C})^2 \nonumber
\end{align}
Using $a^2+b^2\leq(a+b)^2$ for nonnegative $a,b$, we conclude that
\begin{equation}
\label{eqn:fqe_error_propagation_bound}
    \norm{Q^\pi - Q_K}_\rho \leq \frac{\gamma^{1/2}}{(1-\gamma)^{3/2}} \left( \sqrt{\beta_{\mu}}\max_k\norm{\epsilon_k}_\mu + \frac{\gamma^{K/2}}{(1-\gamma)^{1/2}}2 \widebar{C}\right)
\end{equation}
\textbf{Step 3: Turning error bound from $Q$ to $\abs{C^\pi-C_K}$} Now we can choose $\rho$ to be the state-action distribution by the evaluation policy $\pi$. The error bound on the value function $C$ follows simply by integrating inequality (\ref{eqn:fqe_error_propagation_bound}) over state-action pairs induced by $\pi$. The final error across iterations can be related to individual iteration error by 
\begin{equation}
    \abs{C^\pi - C_K} \leq \frac{\gamma^{1/2}}{(1-\gamma)^{3/2}} \left( \sqrt{\beta_{\mu}}\max_k\norm{\epsilon_k}_\mu + \frac{\gamma^{K/2}}{(1-\gamma)^{1/2}}2 \widebar{C}\right) \label{eqn:fqe_v_bound_propagation}
\end{equation}

\subsection{Finite-sample guarantees for Fitted Q Evaluation}
Combining results from (\ref{eqn:fqe_single_bound_realize}), (\ref{eqn:fqe_non_realize_single_iteration_6}) and (\ref{eqn:fqe_v_bound_propagation}), we have the final guarantees for FQE under both realizable and general cases.

\textbf{Realizable Case - Proof of theorem \ref{thm:fqe_appendix_realizable}.} From (\ref{eqn:fqe_single_bound_realize}),  when $n\geq\frac{24\cdot214\cdot \widebar{C}^4}{\epsilon^2}\left( \log\frac{K}{\delta}+\pdim_\F\log\frac{320 \widebar{C}^2}{\epsilon^2}+\log(14e(\pdim_\F+1))\right)$, we have
$\norm{\epsilon_k}_\mu < \epsilon$ with probability at least $1-\delta/K$ for any $0\leq k<K$. Thus we conclude that for any $\epsilon>0, 0<\delta<1$, after $K$ iterations of Fitted Q Evaluation, the value estimate returned by $Q_K$ satisfies:
$$\abs{C^\pi - C_K} \leq \frac{\gamma^{1/2}}{(1-\gamma)^{3/2}} \left( \sqrt{\beta_{\mu}}\epsilon + \frac{\gamma^{K/2}}{(1-\gamma)^{1/2}}2 \widebar{C}\right)$$
holds with probability $1-\delta$ when $n\geq\frac{24\cdot214\cdot \widebar{C}^4}{\epsilon^2}\left( \log\frac{K}{\delta}+\pdim_\F\log\frac{320 \widebar{C}^2}{\epsilon^2}+\log(14e(\pdim_\F+1))\right)$. This concludes the proof of theorem \ref{thm:fqe_appendix_realizable}.

\textbf{Non-realizable Case - Proof of theorem \ref{thm:fqe_appendix} and theorem \ref{thm:fqe_main} of main paper.} Similarly, from (\ref{eqn:fqe_non_realize_single_iteration_6}) we have
\begin{align}
    \mathbf{P}\big\{ \norm{Q_k-\T^\pi Q_{k-1}}_\mu-2\inf_{f\in\F}\norm{f-\T^\pi Q_{k-1}}_\mu > \epsilon \big\} &\leq  14\cdot e\cdot (\pdim_{\F}+1)\left( \frac{640 \widebar{C}^2}{\epsilon^2}\right)^{\pdim_{\F}}\cdot\exp{\left(-\frac{n\epsilon^2}{48\cdot214\widebar{C}^4}\right)} \nonumber \\ 
    &\qquad +\exp{\left(-\frac{3}{416}\cdot\frac{n\epsilon^2}{\widebar{C}^2}\right)} \nonumber
\end{align}
Since $\inf_{f\in\F}\norm{f-\T^\pi Q_{k-1}}_\mu \leq\sup_{h\in\F}\inf_{f\in\F}\norm{f-\T^\pi h}_\mu = d_\F^\pi$ (the \textit{inherent Bellman evaluation error}), similar arguments to the realizable case lead to the conclusion that for any $\epsilon>0, 0<\delta<1$, after $K$ iterations of FQE:
$$\abs{C^\pi - C_K} \leq \frac{\gamma^{1/2}}{(1-\gamma)^{3/2}} \left( \sqrt{\beta_{\mu}}(2d_\F^\pi+\epsilon) + \frac{\gamma^{K/2}}{(1-\gamma)^{1/2}}2 \widebar{C}\right)$$
holds with probability $1-\delta$ when $n=O\big(\frac{\widebar{C}^4}{\epsilon^2}( \log\frac{K}{\delta}+\textnormal{\texttt{dim}}_{\F}\log\frac{\widebar{C}^2}{\epsilon^2}+\log \textnormal{\texttt{dim}}_{\F})\big)$, thus finishes the proof of theorem \ref{thm:fqe_appendix}. 
 
 Note that in both cases, the $\widetilde{O}(\frac{1}{\epsilon^2})$ dependency of $n$ is significant improvement over previous finite-sample analysis of related approximate dynamic programming algorithms \cite{munos2008finite, antos2008learning, antos2008fitted}. This dependency matches that of previous analysis using linear function approximators from \cite{lazaric2012finite,lazaric2010finite} for LSTD and LSPI algorithms. Here our analysis, using similar assumptions, is applicable for general non-linear, bounded function classes. , which is an improvement over convergence rate of $O(\frac{1}{n^4})$ in related approximate dynamic programming algorithms \cite{antos2008fitted, antos2008learning, munos2008finite}. 
\clearpage
\section{Finite-Sample Analysis of Fitted Q Iteration (FQI)}
\label{sec:proof-fqi}
\subsection{Algorithm and Discussion}
\begin{algorithm}[H]
    \begin{small}
    \caption{ Fitted Q Iteration with Function Approximation: $\texttt{FQI}(c)$ \cite{ernst2005tree}}
    \label{algo:fqi}
    \begin{algorithmic}[1]
        \REQUIRE Collected data set $\D = \{ x_i,a_i,x_i^\prime,c_i\}_{i=1}^n$. Function class $\F$
        \STATE Initialize $Q_0 \in \F$ randomly
        \FOR{$k = 1,2,\ldots,K$}
        \STATE Compute target $y_i = c_i+\gamma\min_a Q_{k-1}(x_i^\prime,a) \enskip \forall i$ \\
        \STATE Build training set $\widetilde{\D}_k = \{(x_i,a_i),y_i \}_{i=1}^n$
        \STATE Solve a supervised learning problem: \\
        \quad $Q_k = \argmin\limits_{f\in\F} \frac{1}{n}\sum_{i=1}^n (f(x_i,a_i)-y_i )^2$
        \ENDFOR
        \ENSURE $\pi_K(\cdot) = \argmin\limits_{a} Q_K(\cdot,a)$ (greedy policy with respect to the returned function $Q_K$)
    \end{algorithmic}
    \end{small}
\end{algorithm}
The analysis of FQI (algorithm \ref{algo:fqi}) follows analogously from the analysis of FQE from the previous section (Appendix \ref{sec:proof-fqe}). For brevity, we skip certain detailed derivations, especially those that are largely identical to FQE's analysis. 

To the best of our knowledge, a finite-sample analysis of FQI with general non-linear function approximation has not been published (Continuous FQI from \cite{antos2008fitted} is in fact a Fitted Policy Iteration algorithm and is different from algo \ref{algo:fqi}). In principle, one can adapt existing analysis of fitted value iteration \cite{munos2008finite} and FittedPolicyQ \cite{antos2008learning,antos2008fitted} to show that under similar assumptions, among policies greedy w.r.t. functions in $\F$, FQI will find $\epsilon-$ optimal policy using $n = \widetilde{O}(\frac{1}{\epsilon^4})$ samples. We derive an improved analysis of FQI with general non-linear function approximations, with better sample complexity of $n = \widetilde{O}(\frac{1}{\epsilon^2})$. We note that the appendix of \cite{lazaric2011transfer} contains an analysis of LinearFQI showing similar rate to ours, albeit with linear function approximators. 


In this section, we prove the following statement:
\begin{thm}[Guarantee for FQI - General Case (theorem \ref{thm:fqi_main} in main paper)]
Under Assumption \ref{assume:concentrability}, for any $\epsilon>0, \delta\in(0,1)$, after $K$ iterations of Fitted Q Iteration (algorithm \ref{algo:fqi}), for $n=O\big(\frac{\widebar{C}^4}{\epsilon^2}( \log\frac{K}{\delta}+\pdim_{\F}\log\frac{\widebar{C}^2}{\epsilon^2}+\log \pdim_{\F})\big)$, we have with probability $1-\delta$:
$$C^*-C(\pi_K) \leq \frac{2\gamma}{(1-\gamma)^3}\big( \sqrt{\beta_{\mu}}\left(2d_\F+\epsilon\right) + 2\gamma^{K/2} \widebar{C}\big)$$
where $\pi_K$ is the policy greedy with respect to the returned function $Q_K$, and $C^*$ is the value of the optimal policy.
\end{thm}
The key steps to the proof follow similar scheme to the proof of FQE. We first bound the error for each iteration, and then analyze how the errors flow through the algorithm. 
\subsection{Single iteration error bound $\norm{Q_k-\T Q_{k-1}}_\mu$}
Here $\mu$ is the state-action distribution induced by the data-generating policy $\pi_\D$. 

We begin with the decomposition:
\begin{align}
    &\norm{Q_k-\T Q_{k-1}}_\mu^2 = \E\left[ (Q_k(x,a) - y)^2\right] - \E\left[ (\T Q_{k-1}(x,a)-y)^2\right] \nonumber\\
    &= \bigg\{ \E\left[ (Q_k(x,a) - y)^2\right] - \E\left[ (\T Q_{k-1}(x,a)-y)^2\right] - 2\cdot\left(\frac{1}{n}\sum_{i=1}^n (Q_k(x_i,a_i)-y_i)^2 - \frac{1}{n}\sum_{i=1}^n(\T Q_{k-1}(x_i,a_i)-y_i)^2 \right)\bigg\} \nonumber \\
    &\qquad\qquad+ \bigg\{ 2\cdot\left(\frac{1}{n}\sum_{i=1}^n (Q_k(x_i,a_i)-y_i)^2 - \frac{1}{n}\sum_{i=1}^n(\T Q_{k-1}(x_i,a_i)-y_i)^2 \right) \bigg\} \nonumber \\
    &=\texttt{component\_1} + \texttt{component\_2} \nonumber
\end{align}
For $\T$ the Bellman (optimality) operator (equation \ref{eqn:appendix_Bellman_optimality_operator}),  $\T Q_{k-1}$ is the \textit{regression function} that minimizes square loss $\min\limits_{h:\R^{X\times A}\mapsto\R}\E\abs{h(x,a)-y}^2$, with the random variables $(x,a)\sim\mu$ and $y=c(x,a)+\gamma\min_{a^\prime}Q_{k-1}(x^\prime,a^\prime)$ where $x^\prime\sim p(x^\prime|x,a)$. Invoking lemma \ref{lem:bartlett} and following the steps similar to equations (\ref{eqn:fqe_relax_best_in_class}),(\ref{eqn:fqe_rearrange_11_4}),(\ref{eqn:fqe_cover_bound}) and (\ref{eqn:fqe_single_bound_realize}) from appendix \ref{sec:proof-fqe}, we can bound the first component as
\begin{equation}
    \mathbf{P}(\texttt{component\_1} >\epsilon/2) \leq 14\cdot e\cdot (\pdim_\F+1)\left( \frac{640 \widebar{C}^2}{\epsilon}\right)^{\pdim_\F}\cdot\exp{\left(-\frac{n\epsilon}{48\cdot214\widebar{C}^4}\right)} \label{eqn:fqi_non_realize_single_iteration_3}
\end{equation}
Let $f^* = \arginf\limits_{f\in\F}\norm{f-\T Q_{k-1}}_\mu^2$. Since $Q_k = \argmin\limits_{f\in\F} \frac{1}{n}\sum_{i=1}^n (f(x_i,a_i)-y_i )^2$, we can upper-bound \texttt{component\_2} by
$$\texttt{component\_2}\leq 2\cdot\left(\frac{1}{n}\sum_{i=1}^n (f^*(x_i,a_i)-y_i)^2 - \frac{1}{n}\sum_{i=1}^n(\T Q_{k-1}(x_i,a_i)-y_i)^2 \right)$$

Let random variable $z = ((x,a),y)$, $z_i = ((x_i,a_i),y_i), i=1,\ldots,n$ and let 
$$h(z) = (f^*(x,a)-y)^2 - (\T Q_{k-1}(x,a)-y)^2$$
We have $\abs{h(z)}\leq 4\widebar{C}^2$. We can derive a bound for $\mathbf{P}\left(\frac{1}{n}\sum_{i=1}^n h(z_i)-\E h(z)>\frac{\epsilon}{4}+\E h(z) \right)$ using Bernstein inequality, similar to equations (\ref{eqn:fqe_non_realize_single_iteration_1}) and (\ref{eqn:fqe_non_realize_single_iteration_2}) from appendix \ref{sec:proof-fqe} to obtain:
\begin{equation}
\mathbf{P}\left( 2\cdot\left[\frac{1}{n}\sum_{i=1}^n h(z_i) - 2\E h(z)\right] >\frac{\epsilon}{2}\right) \leq \exp{\left(-\frac{3}{416}\cdot\frac{n\epsilon}{\widebar{C}^2}\right)}
    \label{eqn:fqi_non_realize_single_iteration_4}
\end{equation}
Now we have $$\texttt{component\_2}\leq 2\cdot\frac{1}{n}\sum_{i=1}^n h(z_i) = 2\cdot\left[\frac{1}{n}\sum_{i=1}^n h(z_i) - 2\E h(z)\right] + 4\E h(z)$$
Since
\begin{align}
    \E h(z) &= \E_{\widetilde{D}_k}\left[ (f^*(x,a)-y)^2\right]-\E_{\widetilde{D}_k}\left[ (\T Q_{k-1}(x,a)-y)^2\right] = \E_{\widetilde{D}_k}\left[ (f^*(x,a) - \T Q_{k-1}(x,a))^2\right] \nonumber \\
    &= \inf_{f\in\F}\norm{f-\T Q_{k-1}}_\mu^2 \label{eqn:fqi_non_realize_single_iteration_5}
\end{align}

Combining equations (\ref{eqn:fqi_non_realize_single_iteration_3}), (\ref{eqn:fqi_non_realize_single_iteration_4}) and (\ref{eqn:fqi_non_realize_single_iteration_5}), we obtain that
\begin{align}
    \mathbf{P}\big\{ \norm{Q_k-\T Q_{k-1}}_\mu^2-4\inf_{f\in\F}\norm{f-\T Q_{k-1}}_\mu^2 > \epsilon \big\} &\leq  14\cdot e\cdot (\pdim_\F+1)\left( \frac{640 \widebar{C}^2}{\epsilon}\right)^{\pdim_\F}\cdot\exp{\left(-\frac{n\epsilon}{48\cdot214\widebar{C}^4}\right)} \nonumber \\ 
    &\qquad +\exp{\left(-\frac{3}{416}\cdot\frac{n\epsilon}{\widebar{C}^2}\right)} \label{eqn:fqi_non_realize_single_iteration_6}
\end{align}

\subsection{Propagation of error bound for $\norm{Q^*-Q^{\pi_K}}_{\rho}$}
The analysis of error propagation for FQI is more involved than that of FQE, but the proof largely follows the error propagation analysis in lemma 3 and 4 of \cite{munos2008finite} in the fitted value iteration context (for $V$ function). We include the $Q$ function's (slighly more complicated) derivation here for completeness. 

Recall that $\pi_K$ is greedy wrt the learned function $Q_K$ returned by FQI. We aim to bound the difference $C^*-C^{\pi_K}$ between the optimal value function and that $\pi_K$. For a (to-be-specified) distribution $\rho$ of state-action pairs (different from the data distribution $\mu$), we bound the generalization loss $\norm{Q^*-Q^{\pi_K}}_{\rho}$

\textbf{Step 1: Upper-bound the propagation error (value)}. Let $\epsilon_{k-1} = Q_k - \T Q_{k-1}$. We have that 
\begin{align}
    Q^* - Q_k = \T^{\pi^*}Q^*-\T^{\pi^*}Q_{k-1}+\T^{\pi^*}Q_{k-1}-\T Q_{k-1}+\epsilon_{k-1} &\leq \T^{\pi^*}Q^*-\T^{\pi^*}Q_{k-1} + \epsilon_{k-1} \textit{ (b/c } \T Q_{k-1}\geq \T^{\pi^*} Q_{k-1}\text{)} \nonumber\\
    &=\gamma P^{\pi^*}(Q^*-Q_{k-1})+\epsilon_{k-1}\nonumber
\end{align}
Thus by recursion $Q^*-Q_K\leq \sum_{k=0}^{K-1}\gamma^{K-k-1}(P^{\pi^*})^{K-k-1}\epsilon_k + \gamma^K (P^{\pi^*})^K(Q^*-Q_0)  $

\textbf{Step 2: Lower-bound the propagation error (value)}. Similarly
\begin{align}
    Q^*-Q_k = \T Q^*-\T^{\pi_{k-1}}Q^*+\T^{\pi_{k-1}}Q^*-\T Q_{k-1}+\epsilon_{k-1}
    &\geq \T^{\pi_{k-1}}Q^*-\T Q_{k-1} +\epsilon_{k-1} \text{ (as } \T Q^*\geq \T^{\pi_{k-1}} Q^*\text{)} \nonumber\\
    &\geq \T^{\pi_{k-1}}Q^*-\T^{\pi_{k-1}}Q_{k-1}+\epsilon_{k-1} \textit{ (b/c } \pi_{k-1}\textit{ greedy wrt }Q_{k-1}) \nonumber\\
    &=\gamma P^{\pi_{k-1}}(Q^*-Q_{k-1}) +\epsilon_{k-1}\nonumber
\end{align}
And by recursion $Q^*-Q_K\geq \sum_{k=0}^{K-1} \gamma^{K-k-1}(P^{\pi_{K-1}}P^{\pi_{K-2}}\ldots P^{\pi_{k+1}})\epsilon_k + \gamma^K(P^{\pi_{K-1}}P^{\pi_{K-2}}\ldots P^{\pi_0})(Q^*-Q_0)$

\textbf{Step 3: Upper-bound the propagation error (policy)}. Beginning with a decomposition of value wrt to policy $\pi_K$
\begin{align}
    Q^*-Q^{\pi_K} &= \T^{\pi^*}Q^*-\T^{\pi^*}Q_K+\T^{\pi^*}Q_K-\T^{\pi_K}Q_K+\T^{\pi_K}Q_K-\T^{\pi_K}Q^{\pi_K} \nonumber\\
    &\leq (\T^{\pi^*}Q^*-\T^{\pi^*}Q_K)+(\T^{\pi_K}Q_K-\T^{\pi_K}Q^{\pi_K}) \qquad (\text{ since }\T^{\pi^*}Q_K\leq \T Q_K = \T^{\pi_K}Q_K) \nonumber\\
    &=\gamma P^{\pi^*}(Q^*-Q_K)+\gamma P^{\pi_K}(Q_K-Q^{\pi_K}) \nonumber\\
    &=\gamma P^{\pi^*}(Q^*-Q_K)+\gamma P^{\pi_K}(Q_K-Q^*+Q^*-Q^{\pi_K})\nonumber
\end{align}
Thus leading to $(I-\gamma P^{\pi_K})(Q^*-Q^{\pi_K})\leq \gamma(P^{\pi^*}-P^{\pi_K})(Q^*-Q_K)$
The operator $(I-\gamma P^{\pi_K})$ is invertible and $(I-\gamma P^{\pi_K})^{-1} = \sum_{m\geq0}\gamma^m(P^{\pi_K})^m$ is monotonic. Thus
\begin{align}
    Q^*-Q^{\pi_K}&\leq \gamma(I-\gamma P^{\pi_K})^{-1}(P^{\pi^*}-P^{\pi_K})(Q^*-Q_K) \nonumber \\
    &=\gamma(I-\gamma P^{\pi_K})^{-1}P^{\pi^*}(Q^*-Q_K)- \gamma(I-\gamma P^{\pi_K})^{-1}P^{\pi_K}(Q^*-Q_K)\label{eqn:propagation_1}
\end{align}
Applying inequalities from Step 1 and Step 2 to the RHS of (\ref{eqn:propagation_1}), we have
\begin{align}
    Q^*-Q^{\pi_K}\leq (I-\gamma P^{\pi_K})^{-1}\bigg[&\sum_{k=0}^{K-1}\gamma^{K-k}\left( (P^{\pi^*})^{K-k}-P^{\pi_K}P^{\pi_{K-1}}\ldots P^{\pi_{k+1}} \right)\epsilon_k \nonumber \\
    &+\gamma^{K+1}\left( (P^{\pi^*})^{K+1} - (P^{\pi_K}P^{\pi_{K-1}}\ldots P^{\pi_0}) \right) (Q^*-Q_0)\bigg] \label{eqn:propagation_2}
\end{align}
Next we apply point-wise absolute value on RHS of (\ref{eqn:propagation_2}), with $\abs{\epsilon_k}$ being the short-hand notation for $\abs{\epsilon_k(x,a)}$ point-wise. Using triangle inequalities and rewriting (\ref{eqn:propagation_2}) in a more compact form (\cite{munos2008finite}):
\begin{equation*}
    Q^*-Q^{\pi_K}\leq \frac{2\gamma(1-\gamma^{K+1})}{(1-\gamma)^2}\left[\sum_{k=0}^{K-1}\alpha_k A_k \abs{\epsilon_k}+\alpha_K A_K \abs{Q^*-Q_0} \right]
\end{equation*}
where $\alpha_k = \frac{(1-\gamma)\gamma^{K-k-1}}{1-\gamma^{K+1}} \text{ for }k<K, \alpha_K = \frac{(1-\gamma)\gamma^K}{1-\gamma^{K+1}}$ and 
\begin{align*}
    A_k &= \frac{1-\gamma}{2}(I-\gamma P^{\pi_K})^{-1}\left[(P^{\pi^*})^{K-k}+ P^{\pi_K}P^{\pi_{K-1}}\ldots P^{\pi_{k+1}}\right] \text{ for } k<K \nonumber \\
    A_K &= \frac{1-\gamma}{2}(I-\gamma P^{\pi_K})^{-1}\left[(P^{\pi^*})^{K+1}+ P^{\pi_K}P^{\pi_{K-1}}\ldots P^{\pi_{0}}\right]\nonumber
\end{align*}
Note that $A_k$'s are probability kernels that combine the $P^{\pi_i}$ terms and $\alpha_k$'s are chosen such that $\sum_k \alpha_k = 1$.

\textbf{Step 4: Bounding $\norm{Q^*-Q^{\pi_K}}_\rho^2$ for any test distribution $\rho$}. 

This step handles distribution shift from $\mu$ to $\rho$ (similar to Step 2 from sub-section \ref{subsec:fqe_error_propagation} of appendix \ref{sec:proof-fqe})
\begin{equation*}
    \norm{Q^*-Q^{\pi_K}}_\rho^2 \leq \left[ \frac{2\gamma(1-\gamma^{K+1})}{(1-\gamma)^2}\right]^2\int\rho(dx,da)\left[\sum_{k=0}^{K-1}\alpha_k A_k \epsilon_k^2+\alpha_K A_K(Q^*-Q_0)^2 \right](x,a) \text{ (twice Jensen)}\nonumber
\end{equation*}
Using assumption \ref{assume:concentrability} (assumption \ref{assume:concentrability_main} in the main paper), each term $\rho A_k$ is bounded as 
\begin{align}
    \rho A_k &= \frac{1-\gamma}{2}\rho (I-\gamma P^{\pi_K})^{-1}\left[(P^{\pi^*})^{K-k}+P^{\pi_K}P^{\pi_{K-1}}\ldots P^{\pi_{k+1}} \right] \nonumber\\
    &=\frac{1-\gamma}{2}\sum_{m\geq0}\gamma^m \rho(P^{\pi_K})^m \left[(P^{\pi^*})^{K-k}+P^{\pi_K}P^{\pi_{K-1}}\ldots P^{\pi_{k+1}} \right] \leq (1-\gamma)\sum_{m\geq0} \gamma^m \beta_\mu(m+K-k)\mu \qquad \text { (def \ref{def:concentrability})} \nonumber
\end{align}
Thus
\begin{align}
    \norm{Q^*-Q^{\pi_K}}_\rho^2 &\leq \left[ \frac{2\gamma(1-\gamma^{K+1})}{(1-\gamma)^2}\right]^2 \left[\frac{1}{1-\gamma^{K+1}}\sum_{k=0}^{K-1} (1-\gamma)^2\sum_{m\geq0} \gamma^{m+K-k-1} \beta_\mu(m+K-k)\norm{\epsilon_k}_\mu^2 +\alpha_K(2\widebar{C})^2 \right] \nonumber \\
    &\leq \left[ \frac{2\gamma(1-\gamma^{K+1})}{(1-\gamma)^2}\right]^2 \left[\frac{1}{1-\gamma^{K+1}}\beta_\mu\max_k\norm{\epsilon_k}_\mu^2 + \frac{(1-\gamma)\gamma^K}{1-\gamma^{K+1}}(2\widebar{C})^2 \right] \qquad \text{(assumption \ref{assume:concentrability})}\nonumber \\
    &\leq \left[ \frac{2\gamma(1-\gamma^{K+1})}{(1-\gamma)^2}\right]^2 \left[\frac{1}{1-\gamma^{K+1}}\beta_\mu\max_k\norm{\epsilon_k}_\mu^2 + \frac{\gamma^K}{1-\gamma^{K+1}}(2\widebar{C})^2 \right] \nonumber \\
    &\leq\left[\frac{2\gamma}{(1-\gamma)^2} \right]^2 \left[\beta_\mu\max_{k}\norm{\epsilon_k}_\mu^2 + \gamma^K(2\widebar{C})^2 \right] \nonumber
\end{align}
Using $a^2+b^2\leq(a+b)^2$ for nonnegative $a,b$, we thus conclude that
\begin{equation}
\label{eqn:fqi_error_propagation_bound}
    \norm{Q^* - Q^{\pi_K}}_\rho \leq \frac{2\gamma}{(1-\gamma)^2} \left( \sqrt{\beta_\mu}\max_k\norm{\epsilon_k}_\mu + 2\gamma^{K/2} \widebar{C}\right)
\end{equation}
\textbf{Step 5: Bounding $C^*-C^{\pi_K}$}
Using the performance difference lemma (lemma 6.1 of \cite{kakade2002approximately}, which states that $C^*-C^{\pi_K} = -\frac{1}{1-\gamma}\E_{\substack{x\sim d_{\pi_K} \\ a\sim\pi_K}} A^*\left[x,a \right]$. We can upper-bound the performance difference of value function as
\begin{align}
    C^*-C^{\pi_K}  &= \frac{1}{1-\gamma}\E_{\substack{x\sim d_{\pi_K}\\a\sim \pi_K}} \left[ C^*(x) - Q^*(x,a)\right] = \frac{1}{1-\gamma}\E_{x~\sim d_{\pi_K}} \left[ C^*(x) - Q^*(x,\pi_K(x))\right] \nonumber \\
    &\leq \frac{1}{1-\gamma}\E_{x~\sim d_{\pi_K}} \left[ Q^*(x,\pi^*(x)) -Q_K(x,\pi^*(x)) + Q_K(x,\pi_K(x)- Q^*(x,\pi_K(x))\right] \text{(greedy)} \nonumber \\
    &\leq \frac{1}{1-\gamma}\E_{x~\sim d_{\pi_K}} \lvert Q^*(x,\pi^*(x)) -Q_K(x,\pi^*(x))\rvert + \lvert Q_K(x,\pi_K(x)- Q^*(x,\pi_K(x))\rvert \nonumber \\
    &\leq \frac{1}{1-\gamma} \left( \norm{Q^*-Q^{\pi_K}}_{d_{\pi_K}\times\pi^*}+\norm{Q^*-Q^{\pi_K}}_{d_{\pi_K}\times\pi_K} \right) \text{(upper-bound 1-norm by 2-norm)}\nonumber \\
    &\leq \frac{2\gamma}{(1-\gamma)^3}\left( \sqrt{\beta_\mu}\max_k\norm{\epsilon_k}_\mu + 2\gamma^{K/2} \widebar{C}\right) \label{eqn:fqi_v_bound_propagation}
\end{align}
Note that inequality (\ref{eqn:fqi_v_bound_propagation}) follows from (\ref{eqn:fqi_error_propagation_bound}) by specifying $\rho=\chi P^{\pi_K}P^{\pi^*}$ and $\rho=\chi P^{\pi_K}P^{\pi_K}$, respectively ($\chi$ is the initial state distribution). 
\subsection{Finite-sample guarantees for Fitted Q Iteration}
From (\ref{eqn:fqi_non_realize_single_iteration_6}) we have:
\begin{align}
    \mathbf{P}\big\{ \norm{Q_k-\T Q_{k-1}}_\mu-2\inf_{f\in\F}\norm{f-\T Q_{k-1}}_\mu > \epsilon \big\} &\leq  14\cdot e\cdot (\pdim_\F+1)\left( \frac{640 \widebar{C}^2}{\epsilon^2}\right)^{\pdim_\F}\cdot\exp{\left(-\frac{n\epsilon^2}{48\cdot214\widebar{C}^4}\right)} \nonumber \\ 
    &\qquad +\exp{\left(-\frac{3}{416}\cdot\frac{n\epsilon^2}{\widebar{C}^2}\right)} \nonumber
\end{align}
Note that $\inf_{f\in\F}\norm{f-\T Q_{k-1}}_\mu \leq\sup_{h\in\F}\inf_{f\in\F}\norm{f-\T h}_\mu = d_\F$ (the \textit{inherent Bellman error} from equation \ref{eqn:appendix_Bellman_optimality_operator}). Combining with equation (\ref{eqn:fqi_v_bound_propagation}), we have the conclusion that for any $\epsilon>0, 0<\delta<1$, after $K$ iterations of Fitted Q Iteration, and for $\pi_K$ the greedy policy wrt $Q_K$:
$$C^*-C^\pi_K \leq \frac{2\gamma}{(1-\gamma)^3}\left( \sqrt{\beta_\mu}(2d_\F+\epsilon) + 2\gamma^{K/2} \widebar{C}\right)$$
holds with probability $1-\delta$ when $n=O\big(\frac{\widebar{C}^4}{\epsilon^2}( \log\frac{K}{\delta}+\pdim_{\F}\log\frac{\widebar{C}^2}{\epsilon^2}+\log \pdim_{\F})\big)$.

Note that compared to the Fitted Value Iteration analysis of \cite{munos2008finite}, our error includes an extra factor $2$ for $d_\F$.
\subsection{Statement for the Bellman-realizable Case}
To facilitate the end-to-end generalization analysis of theorem \ref{thm:end_to_end_main} in the main paper, we include a version of FQI analysis under Bellman-realizable assumption in this section. The theorem is a consequence of previous analysis in this section.
\begin{assumption}[Bellman evaluation realizability]
\label{assume:realizability_fqi}
We consider function classes $\F$ sufficiently rich so that $\forall f,\T f\in\F$.
\end{assumption}
\begin{thm}[Guarantee for FQI - Bellman-realizable Case]
Under Assumption \ref{assume:concentrability} and \ref{assume:realizability_fqi}, for any $\epsilon>0, \delta\in(0,1)$, after $K$ iterations of Fitted Q Iteration, for $n\geq\frac{24\cdot214\cdot \widebar{C}^4}{\epsilon^2}\big( \log\frac{K}{\delta}+\pdim_{\F}\log\frac{320 \widebar{C}^2}{\epsilon^2}+\log(14e(\pdim_{\F}+1))\big)$, we have with probability $1-\delta$:
$$C^*-C(\pi_K) \leq \frac{2\gamma}{(1-\gamma)^3}\big( \sqrt{\beta_{\mu}}\epsilon + 2\gamma^{K/2} \widebar{C}\big)$$
where $\pi_K$ is the policy greedy with respect to the returned function $Q_K$, and $C^*$ is the value of the optimal policy.
\end{thm}

\clearpage
\section{Additional Instantiation of Meta-Algorithm (algorithm \ref{algo:meta})}
\label{sec:app-algorithm}
We provide an additional instantiation of the meta-algorithm described in the main paper, with Online Gradient Descent (OGD) \cite{zinkevich2003online} and Least-Squares Policy Iteration (LSPI) \cite{lagoudakis2003least} as subroutines. Using LSPI requires a feature map $\phi$ such that any state-action pair can be represented by $k$ features. The value function is linear in parameters represented by $\phi$. Policy representation is simplified to a weight vector $w\in\R^k$.

Similar to our main algorithm \ref{algo:main_algo}, OGD updates require bounded parameters $\lambda$. We thus introduce hyper-parameter $B$ as the bound of $\lambda$ in $\ell_2$ norm. The gradient update is projected to the $\ell_2$ ball when the norm of $\lambda$ exceeds $B$ (line 15 of algorithm \ref{algo:algo_lspi_ogd}).
\begin{algorithm}[h]
    \caption{ Batch Learning under Constraints using Online Gradient Descent and Least-Squares Policy Iteration} 
    \label{algo:algo_lspi_ogd}
    \begin{algorithmic}[1]
    \REQUIRE Dataset $\D = \{ x_i,a_i,x_i^\prime,c_i,g_i\}_{i=1}^n \sim\pi_{\D}$. Online algorithm parameters: $\ell_2$ norm bound $B$, learning rate $\eta$
    \REQUIRE Number of basis function $k$. Basis function $\phi$ (feature map for state-action pairs)
    \STATE Initialize $\lambda_1 = (0,\ldots,0)\in\R^{m}$
        \FOR{each round $t$}
        \STATE Learn $w_t\leftarrow \texttt{LSPI}(c+\lambda_t^\top g)$ \hfill // \textit{LSPI with cost $c+\lambda_t^\top g$}\\
        \STATE Evaluate $\widehat{C}(w_t)\leftarrow\texttt{LSTDQ}(w_t,c)$ \hfill // \textit{Algo \ref{algo:lstdq} with $\pi_t$, cost $c$}\\
        \STATE Evaluate $\widehat{G}(w_t)\leftarrow\texttt{LSTDQ}(w_t,g)$ \hfill // \textit{Algo \ref{algo:lstdq} with $\pi_t$}, cost $g$\\
        \STATE $\widehat{w}_t \leftarrow \frac{1}{t} \sum_{t^\prime = 1}^{t} w_{t^\prime}$ \\
        \STATE $\widehat{C}(\widehat{w}_t) \leftarrow \frac{1}{t} \sum_{t^\prime = 1}^{t} \widehat{C}(w_{t^\prime})$, $\widehat{G}(\widehat{w}_t) \leftarrow \frac{1}{t} \sum_{t^\prime = 1}^{t} \widehat{G}(w_{t^\prime})$  \\
        \STATE $\widehat{\lambda}_t \leftarrow \frac{1}{t}\sum_{t^\prime = 1}^t \lambda_{t^\prime}$ \\
        \STATE Learn $\widetilde{w}\leftarrow\texttt{LSPI}(c+\widehat{\lambda}_t^\top g)$ \hfill // \textit{LSPI with cost $c+\widehat{\lambda}_t^\top g$} \\
        \STATE Evaluate $\widehat{C}(\widetilde{w})\leftarrow\texttt{LSTDQ}(\widetilde{w},c), \widehat{G}(\widetilde{w})\leftarrow\texttt{LSTDQ}(\widetilde{w},g)$ \\
        \STATE $\Lhmax= \max\limits_{\lambda, \norm{\lambda}_2\leq B} \left( \widehat{C}(\widehat{w}_t) + \lambda^\top (\widehat{G}(\widehat{w}_t)-\tau)\right)$
        \STATE $\Lhmin = \widehat{C}(\widetilde{w}) + \widehat{\lambda}_t^\top(\widehat{G}(\widetilde{w})-\tau)$
        \IF{$\Lhmax - \Lhmin \leq \omega$}
        \STATE Return $\widehat{\pi}_t$ greedy w.r.t $\widehat{w}_t$ \big(i.e., $\widehat{\pi}_t(x) = \argmin_{a\in\A}\widehat{w}_t^\top\phi(x,a)\enskip\forall x$ \big)\\
        \ENDIF
        \STATE $\lambda_{t+1} = \mathcal{P}(\lambda_t - \eta (\widehat{G}(\pi_t) - \tau))$ where projection $\mathcal{P}(\lambda) = B\frac{\lambda}{\max\{ B, \norm{\lambda}_2\}}$
        \ENDFOR
    \end{algorithmic}
\end{algorithm}

\begin{algorithm}[H]
    \caption{ Least-Squares Policy Iteration: $\texttt{LSPI}(c)$ \cite{lagoudakis2003least} }
    \label{algo:lspi}
    \begin{algorithmic}[1]
        \REQUIRE Stopping criterion $\epsilon$
        \STATE Initialize $w^\prime \leftarrow w_0$
        \REPEAT
        \STATE $w\leftarrow w^\prime$ \\
        \STATE $w^\prime\leftarrow \texttt{LSTDQ}(w,c)$
        \UNTIL $\norm{w-w^\prime}\leq\epsilon$
        \ENSURE Policy weight $w$ \big(i.e., $\pi(x) = \argmin_{a\in\A}w^\top\phi(x,a)\enskip\forall x$\big)
    \end{algorithmic}
\end{algorithm}


\begin{algorithm}[H]
    \caption{$\texttt{LSTDQ}(w,c)$ \cite{lagoudakis2003least} }
    \label{algo:lstdq}
    \begin{algorithmic}[1]
        \STATE Initialize $\mathbf{\widetilde{A}}\leftarrow \mathbf{0}$\hfill // $k\times k$ matrix\\
        \STATE Initialize $\widetilde{b}\leftarrow \mathbf{0}$\hfill \small{// $k\times 1$ vector}\\
        \FOR{each $(x,a,x^\prime,c)\in\D$}
        \STATE $a^\prime = \argmin_{\widetilde{a}\in\A} w^\top\phi(x^\prime, \widetilde{a})$
        \STATE $\mathbf{\widetilde{A}}\leftarrow \mathbf{\widetilde{A}}+\phi(x,a)\big(\phi(x,a)-\gamma\phi(x^\prime,a^\prime) \big)^\top$ \\
        \STATE $\widetilde{b}\leftarrow \widetilde{b}+\phi(x,a)c$
        \ENDFOR
        \STATE $\widetilde{w}\leftarrow \mathbf{\widetilde{A}}^{-1}\widetilde{b}$
        \ENSURE $\widetilde{w}$
    \end{algorithmic}
\end{algorithm}

\clearpage
\section{Additional Experimental Details}
\label{sec:app-experiment}
\subsection{Environment Descriptions and Procedures}
\begin{figure}[h]
\centering
    \customlabel{fig:lake_ending_policy}{fig:screenshot}{left}
       \includegraphics[width=1.2in]{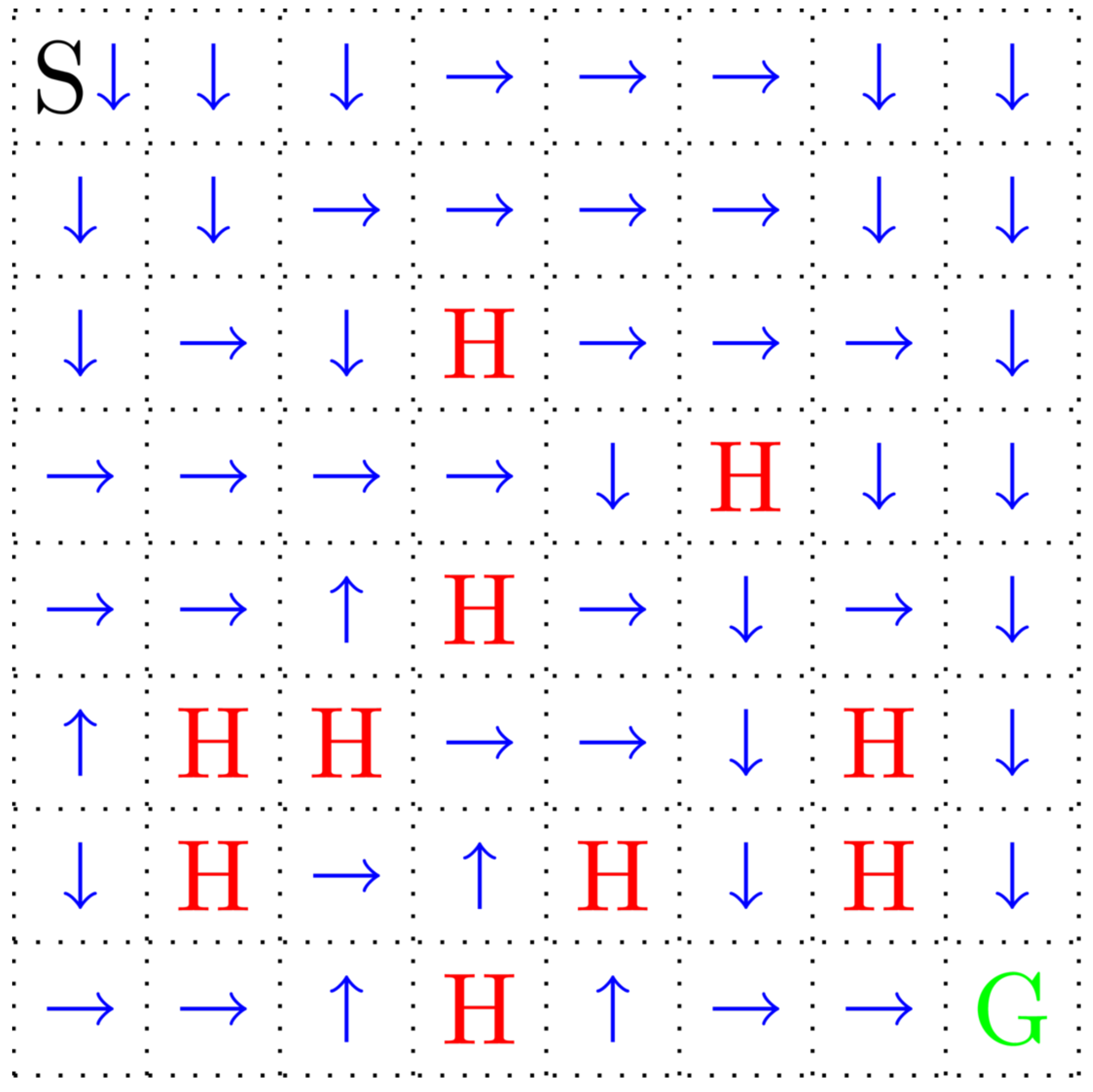}
       \ \ \ 
   \customlabel{fig:car_screenshot}{fig:screenshot}{right}%
         \includegraphics[width=1.8in]{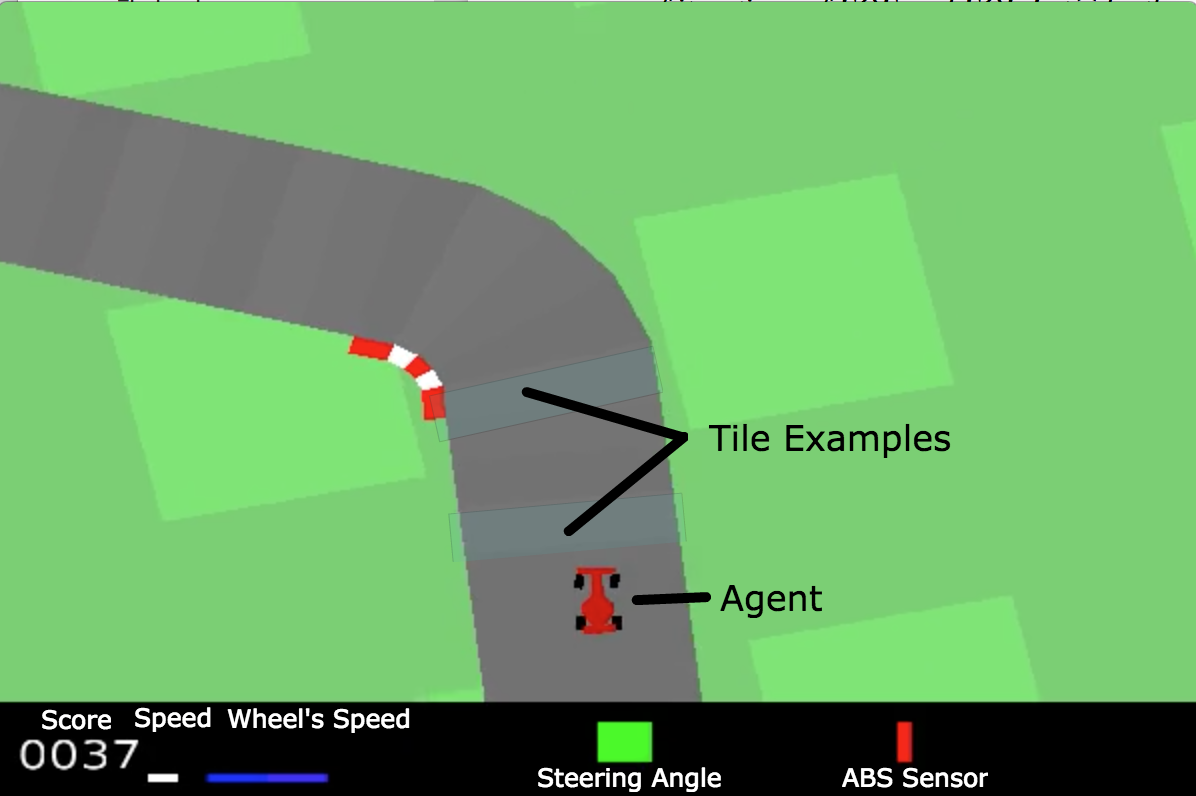}
                 \vspace{-0.1in}
\caption{Depicting the \emph{FrozenLake} and \emph{CarRacing} environments.\vspace{0.05in}}
\label{fig:screenshot}
\end{figure}

\textbf{Frozen Lake.} The environment is a 8x8 grid as seen in Figure~\ref{fig:lake_ending_policy}, based on OpenAi's FrozenLake-v0. In each episode, the agent starts from $S$ and traverse to goal $G$. While traversing the grid, the agent must avoid the pre-determined holes denoted by $H$. If the agent steps off of the grid, the agent returns to the same grid location. The episode terminates when the agent reaches the goal or falls into a hole. The arrows in Figure~\ref{fig:lake_ending_policy} is an example policy returned by our algorithm, showing an optimal route. 

 Denote $\X_{holes}$ as the set of all holes in the grid and $\X_{goal} = \{x_{goal}\}$ is a singleton set representing the goal in the grid. The contrained batch policy learning problem is:
 \begin{align}
 \begin{split}
     &\min_{\pi\in\Pi} \quad C(\pi) = \E[\mathbb{I}(x'\not\in\X_{goals})] = \mathrm{P}(x'\not\in\{x_{goal}\}) \\
     &\text{s.t. } \quad G(\pi) = \E[\mathbb{I}(x'\in\X_{holes})] = \mathrm{P}(x'\in\X_{holes}) \leq \tau 
\end{split}
 \end{align}


We collect $5000$ trajectories by selecting an action randomly with probability $.95$ and an action from a DDQN-trained model with probability $.05$.Furthermore we set $B = 30$ and $\eta = 50$, the hyperparameters of our Exponentiated Gradient subroutine. We set the threshold for the constraint $\tau = .1$.


\textbf{Car Racing.} The environment is a racetrack as seen in Figure~\ref{fig:car_screenshot}, modified from OpenAi's CarRacing-v0. In each state, given by the raw pixels, the agent has 12 actions: $a \in A= \{(i,j,k) | i\in\{-1,0,1\}, j\in\{0,1\}, k\in\{0,.2\} \}$. The action tuple $(i,j,k)$ cooresponds to steering angle, amount of gas applied and amount of brake applied, respectively. In each episode, the agent starts at the same point on the track and must traverse over $95\%$ of the track, given by a discretization of $281$ tiles. The agent recieves a reward of $+\frac{1000}{281}$ for each unique tile over which the agent drives. The agent receives a penalty of $-.1$ per-time step. Our collected dataset takes the form: $\D = \{(x_{t-6}, x_{t-3}, x_t), a_t,(x_{t-3}, x_{t}, x_{t+3}), c_t, g_{0,t}, g_{1,t}\}$ where $x_i$ denotes the image at timestep $i$ and $a_t$ is applied 3 times between $x_t$ and $x_{t+3}$. This frame-stacking option is common practice in online RL for Atari and video games.In our collected dataset $\D$, the maximum horizon is 469 time steps. 

The first constraint concerns accumulated number of brakes, a proxy for smooth driving or acceleration. The second constraint concerns how far the agent travels away from the center of the track, given by the Euclidean distance between the agent and the closest point on the center of the track. Let $N_t$ be the number of tiles that is collected by the agent in time $t$. The constrained batch policy learning problem is:
 \begin{align}
 \begin{split}
     &\min_{\pi\in\Pi} \quad \E[\sum_{t=0}^\infty \gamma^t (-\frac{1000}{281}N_t + .1)]  \\
     &\text{s.t. } \quad G_0(\pi) = \E[\sum_{t=0}^\infty \gamma^t \mathbb{I}(a_t\in\A_{braking})]  \leq \tau_0 \\
     &\quad\quad G_1(\pi) = \E[\sum_{t=0}^\infty \gamma^t d(u_t,v_t)]  \leq \tau_1
\end{split}
 \end{align}

We instatiate our subroutines, FQE and FQI, with multi-layered CNNs. Furthermore we set $B = 10$ and $\eta = .01$, the hyperparameters of our Exponentiated Gradient subroutine. We set the threshold for the constraint to be about 75\% of the value exhibited by online RL agent trained by DDQN \cite{van2016deep}.


\begin{figure}[h]
\centering
    \customlabel{fig:lspi_grid}{fig:grid_search}{left}
       \includegraphics[width=1.6in]{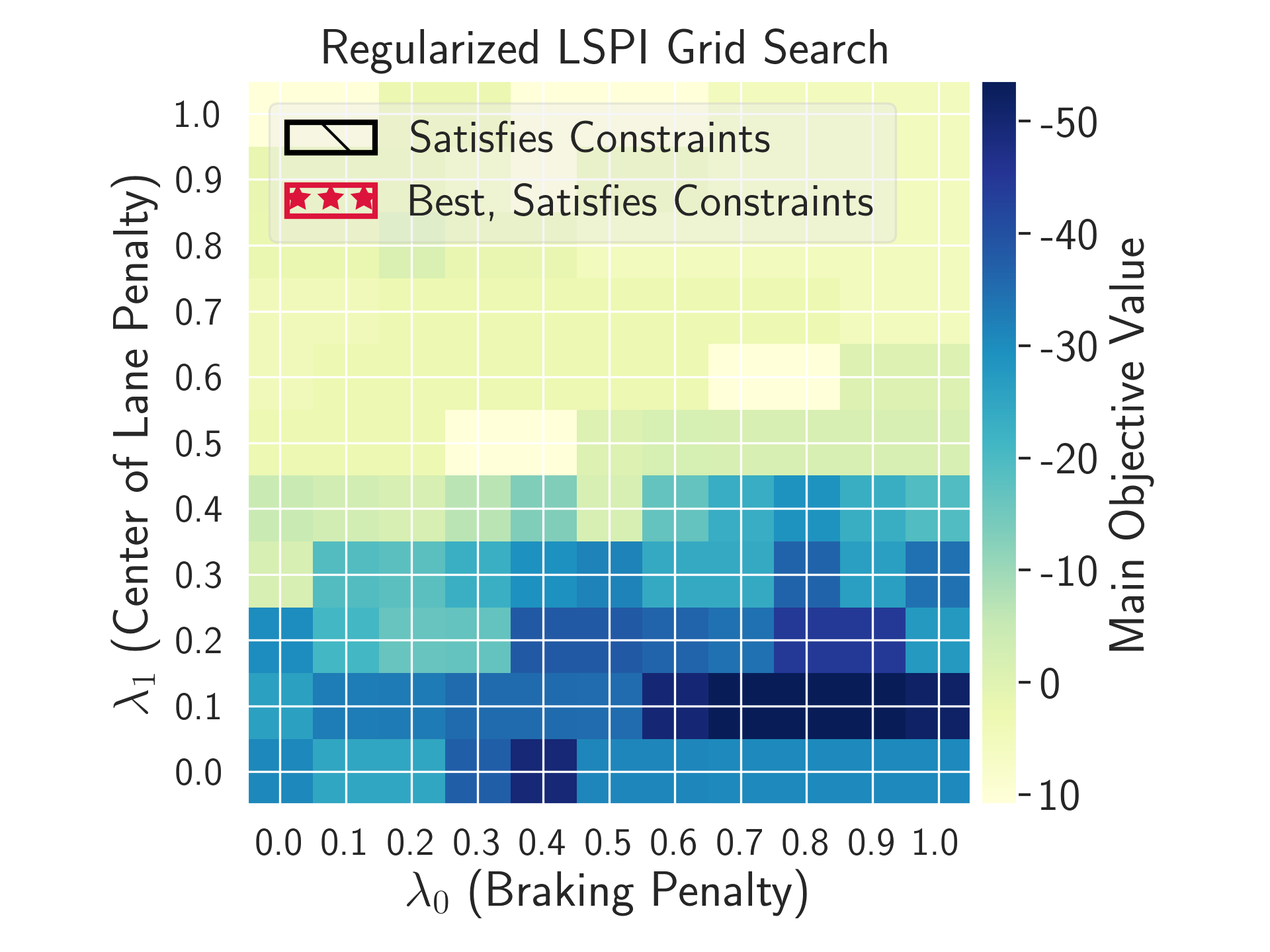}
       \ \ \ 
   \customlabel{fig:fqi_grid}{fig:grid_search}{right}%
         \includegraphics[width=1.6in]{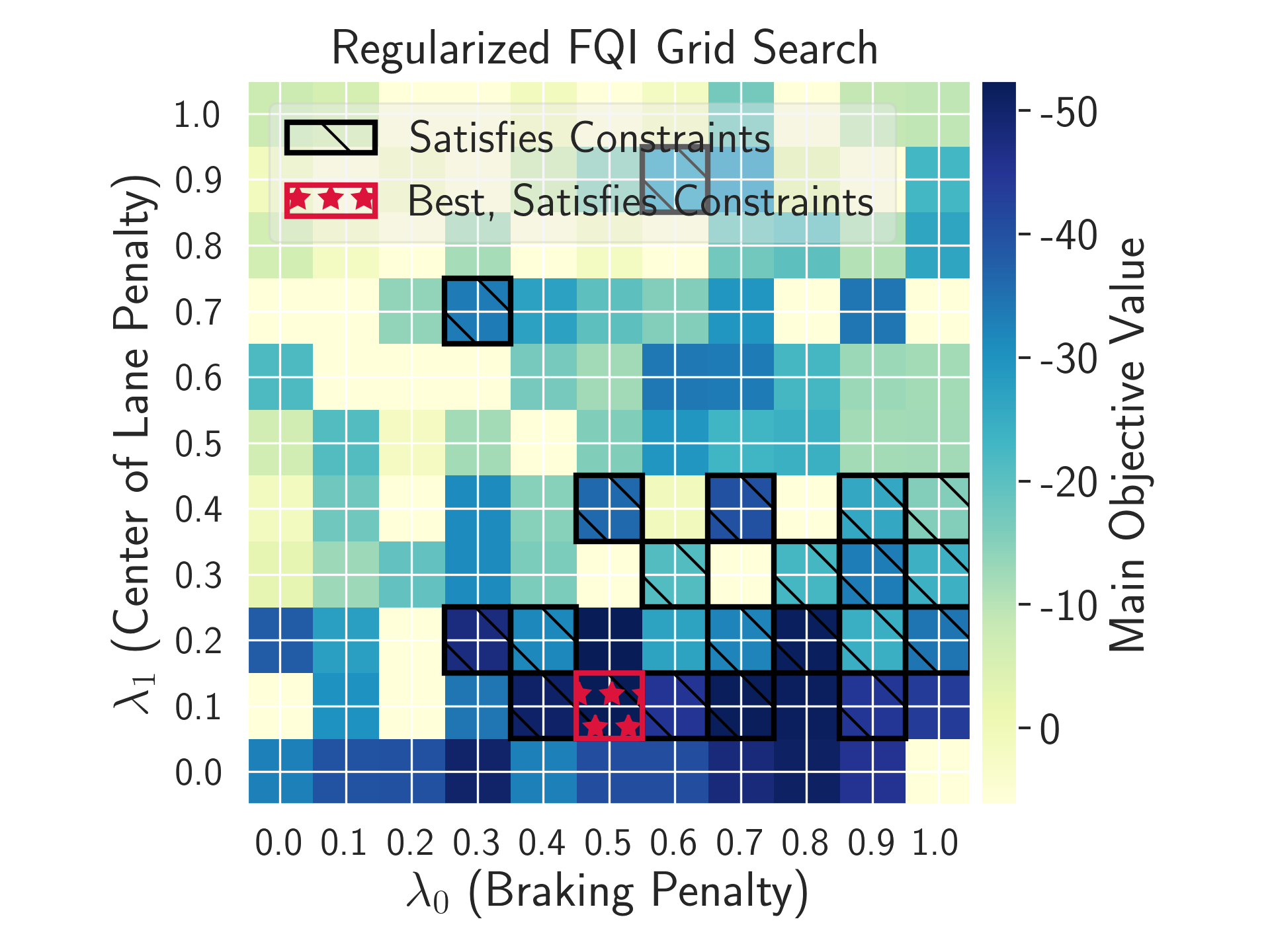}
    \customlabel{fig:band_main}{fig:band_value}{left}
       \includegraphics[width=1.6in]{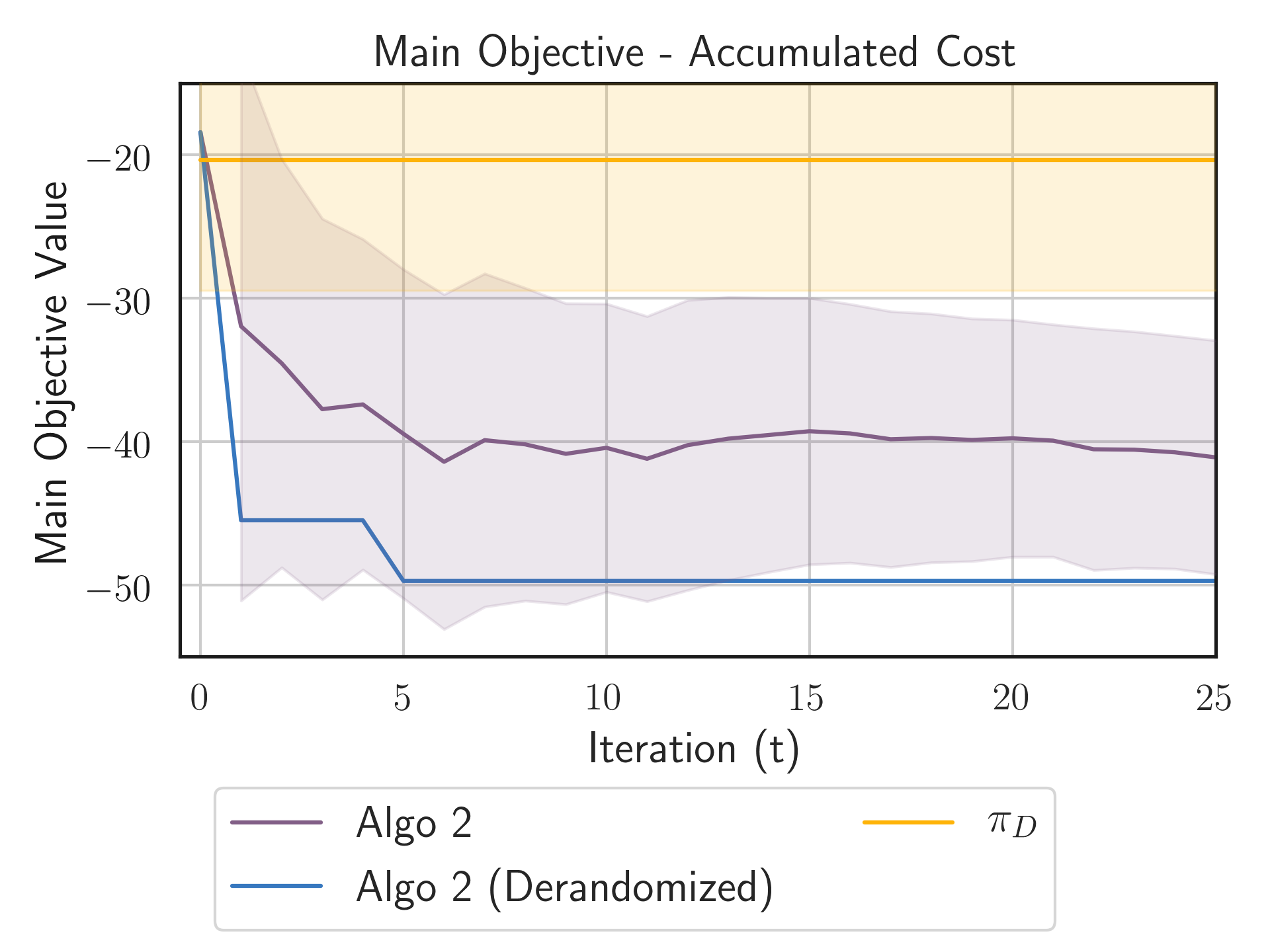}
       \ \ \ 
   \customlabel{fig:band_constraint}{fig:band_value}{right}%
         \includegraphics[width=1.6in]{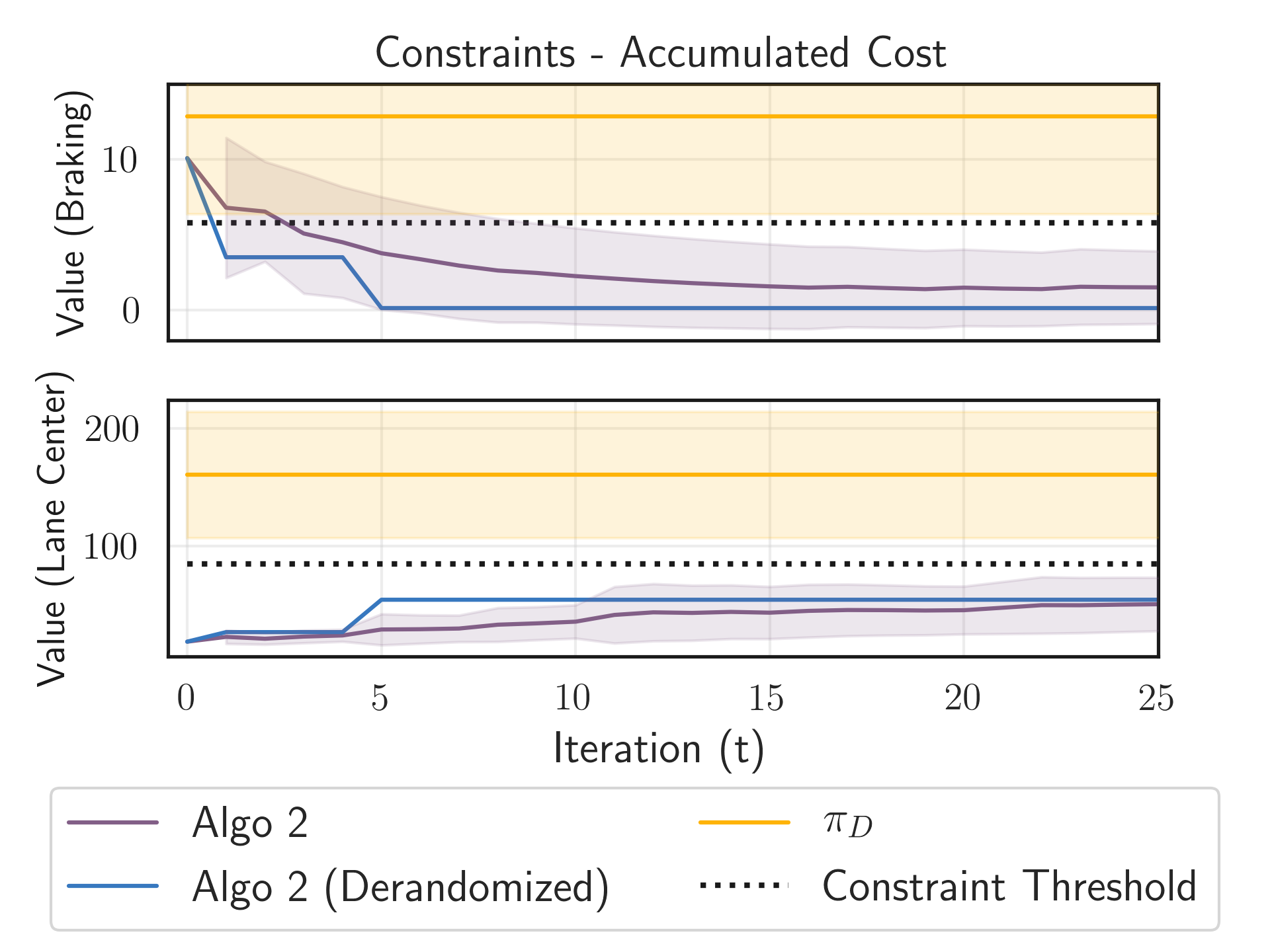}
                 \vspace{-0.1in}
\caption{(First and Second figures) Result of 2-D grid-search for one-shot, regularized policy learning for \emph{LSPI} (left) and \emph{FQI} (right). \\
(Third and Fourth figures) value range of individual policies in our mixtured policy and data generating policy $\pi_D$ for main objective (left) and cost constraint (right) \vspace{0.05in}}
\label{fig:appendix_car}
\end{figure}

\subsection{Additional Discussion for the Car Racing Experiment}
\textbf{Regularized policy learning and grid-search.} We perform grid search over a range of regularization parameters $\lambda$ for both Least-Squares Policy Iteration - LSPI (\cite{lagoudakis2003least}) and Fitted Q Iteration - FQI (\cite{ernst2005tree}). The results, seen from the the first and second plot of Figure~\ref{fig:appendix_car}, show that one-shot regularized learning has difficulty learning a policy that satisfies both constraints. We augment LSPI with non-linear feature mapping from one of our best performing FQI model (using CNNs representation). While both regularized LSPI and regularized FQI can achieve low main objective cost, the constraint cost values tend to be sensitive with the $\lambda$ step. Overall for the whole grid search, about $10\%$ of regularized policies satisfy both constraints, while none of the regularized LSPI policy satisfies both constraints.

\textbf{Mixture policy and de-randomization.} As our algorithm returned a mixture policy, it is natural to analyze the performance of individual policies in the mixture. The third and fourth plot from Figure~\ref{fig:appendix_car} show the range of performance of individual policy in our mixture (purple band). We compare individual policy return with the stochastic behavior of the data generation policy. Note that our policies satisfy constraints almost always, while the individual policy returned in the mixture also tends to outperform $\pi_D$ with respect to the main objective cost. 

\textbf{Off-policy evaluation standalone comparison.} Typically, inverse propensity scoring based methods call for stochastic behavior and evaluation policies \cite{precup2000eligibility,swaminathan2015batch}. However in this domain, the evaluation policy and environment are both deterministic, with long horizon (the max horizon is $\D$ is 469). Consequently Per-Decision Importance Sampling typically evaluates the policy as 0. In general, off-policy policy evaluation in long-horizon domains is known to be challenging \cite{liu2018breaking,guo2017using}. We augment PDIS by approximating the evaluation policy with a stochastic policy, using a softmin temperature parameter. However, PDIS still largely shows significant errors. For Doubly Robust and Weighted Doubly Robust methods, we train a model of the environment as follows: 
\begin{itemize}
  \item a 32-dimensional representation of state input is learned using variational autoencoder. Dimensionality reduction is necessary to aid accuracy, as original state dimension is $96\times96\times3$
  \item an LSTM is used to learn the transition dynamics $P(z(x^\prime) | z(x),a)$, where $z(x)$ is the low-dimensional representation learned from previous step. Technically, using a recurrent neural networks is an augmentation to the dynamical modeling, as true MDPs typically do not require long-term memory
  \item the model is trained separately on a different dataset, collected indendently from the dataset $\D$ used for evaluation
  \end{itemize}
  The architecture of our dynamics model is inspired by recent work in model-based online policy learning \cite{ha2018world}. However, despite our best effort, learning the dynamics model accurately proves highly challenging, as the horizon and dimensionality of this domain are much larger than popular benchmarks in the OPE literature \cite{jiang2016doubly,thomas2016data,farajtabar2018more}. The dynamics model has difficulty predicting the future state several time steps away. Thus we find that the long-horizon, model-based estimation component of DR and WDR in this high-dimensional setting is not sufficiently accurate. For future work, a thorough benchmarking of off-policy evaluation methods in high-dimensional domains would be a valuable contribution. 



\end{document}